\documentclass{article}
\usepackage[T1]{fontenc}
\usepackage[utf8]{inputenc}
\usepackage{color}
\usepackage{mathrsfs}
\usepackage{enumitem}
\usepackage{amsmath}
\usepackage{amsthm}
\usepackage{amssymb}
\usepackage{stmaryrd}
\usepackage[unicode=true,
 bookmarks=false,
 breaklinks=false,pdfborder={0 0 0},pdfborderstyle={},backref=section,colorlinks=true]
 {hyperref}
\hypersetup{
 pdfborderstyle={},linkcolor=blue, citecolor=blue}

\makeatletter
\theoremstyle{plain}
\newtheorem{assumption}{\protect\assumptionname}
\theoremstyle{plain}
\newtheorem{thm}{\protect\theoremname}
\theoremstyle{definition}
\newtheorem{defn}[thm]{\protect\definitionname}
\theoremstyle{plain}
\newtheorem{cor}[thm]{\protect\corollaryname}
\theoremstyle{plain}
\newtheorem{prop}[thm]{\protect\propositionname}
\theoremstyle{remark}
\newtheorem{rem}[thm]{\protect\remarkname}
\theoremstyle{plain}
\newtheorem{lem}[thm]{\protect\lemmaname}

\usepackage{iclr2021_conference,times}

\allowdisplaybreaks[4]

\iclrfinalcopy

\makeatother

\providecommand{\assumptionname}{Assumption}
\providecommand{\corollaryname}{Corollary}
\providecommand{\definitionname}{Definition}
\providecommand{\lemmaname}{Lemma}
\providecommand{\propositionname}{Proposition}
\providecommand{\remarkname}{Remark}
\providecommand{\theoremname}{Theorem}

\begin{document}
\title{Global Convergence of Three-layer Neural Networks in the Mean Field
Regime\thanks{This paper is a conference submission. We refer to the work \cite{nguyen2020rigorous}
and its companion note \cite{pham2020note} for generalizations as
well as other conditions for global convergence in the case of multilayer
neural networks.}}
\author{Huy Tuan Pham\thanks{Department of Mathematics, Stanford University. This work was done
in parts while H. T. Pham was at the University of Cambridge.}$\quad$and Phan-Minh Nguyen\thanks{The Voleon Group. This work was done while P.-M. Nguyen was at Stanford
University.}\; \thanks{The author ordering is randomized.}}
\maketitle
\begin{abstract}
In the mean field regime, neural networks are appropriately scaled
so that as the width tends to infinity, the learning dynamics tends
to a nonlinear and nontrivial dynamical limit, known as the mean field
limit. This lends a way to study large-width neural networks via analyzing
the mean field limit. Recent works have successfully applied such
analysis to two-layer networks and provided global convergence guarantees.
The extension to multilayer ones however has been a highly challenging
puzzle, and little is known about the optimization efficiency in the
mean field regime when there are more than two layers.

In this work, we prove a global convergence result for unregularized
feedforward three-layer networks in the mean field regime. We first
develop a rigorous framework to establish the mean field limit of
three-layer networks under stochastic gradient descent training. To
that end, we propose the idea of a \textit{neuronal embedding}, which
comprises of a fixed probability space that encapsulates neural networks
of arbitrary sizes. The identified mean field limit is then used to
prove a global convergence guarantee under suitable regularity and
convergence mode assumptions, which -- unlike previous works on two-layer
networks -- does not rely critically on convexity. Underlying the
result is a universal approximation property, natural of neural networks,
which importantly is shown to hold at \textit{any} finite training
time (not necessarily at convergence) via an algebraic topology argument.
\end{abstract}

\section{Introduction}

Interests in the theoretical understanding of the training of neural
networks have led to the recent discovery of a new operating regime:
the neural network and its learning rates are scaled appropriately,
such that as the width tends to infinity, the network admits a limiting
learning dynamics in which all parameters evolve nonlinearly with
time\footnote{This is to be contrasted with another major operating regime (the
NTK regime) where parameters essentially do not evolve and the model
behaves like a kernel method (\cite{jacot2018neural,chizat2018note,du2018gradient,allen2018convergence,zou2018stochastic,lee2019wide}).}. This is known as the mean field (MF) limit (\cite{mei2018mean,chizat2018,rotskoff2018neural,sirignano2018mean,nguyen2019mean,araujo2019mean,sirignano2019mean}).
The four works \cite{mei2018mean,chizat2018,rotskoff2018neural,sirignano2018mean}
led the first wave of efforts in 2018 and analyzed two-layer neural
networks. They established a connection between the network under
training and its MF limit. They then used the MF limit to prove that
two-layer networks could be trained to find (near) global optima using
variants of gradient descent, despite non-convexity (\cite{mei2018mean,chizat2018}).
The MF limit identified by these works assumes the form of gradient
flows in the measure space, which factors out the invariance from
the action of a symmetry group on the model. Interestingly, by lifting
to the measure space, with a convex loss function (e.g. squared loss),
one obtains a limiting optimization problem that is convex (\cite{bengio2006convex,bach2017breaking}).
The analyses of \cite{mei2018mean,chizat2018} utilize convexity,
although the mechanisms to attain global convergence in these works
are more sophisticated than the usual convex optimization setup in
Euclidean spaces.

The extension to multilayer networks has enjoyed much less progresses.
The works \cite{nguyen2019mean,araujo2019mean,sirignano2019mean}
argued, heuristically or rigorously, for the existence of a MF limiting
behavior under gradient descent training with different assumptions.
In fact, it has been argued that the difficulty is not simply technical,
but rather conceptual (\cite{nguyen2019mean}): for instance, the
presence of intermediate layers exhibits multiple symmetry groups
with intertwined actions on the model. Convergence to the global optimum
of the model under gradient-based optimization has not been established
when there are more than two layers.

In this work, we prove a global convergence guarantee for feedforward
three-layer networks trained with unregularized stochastic gradient
descent (SGD) in the MF regime. After an introduction of the three-layer
setup and its MF limit in Section \ref{sec:setup}, our development
proceeds in two main steps:

\paragraph*{Step 1 (Theorem \ref{thm:gradient descent coupling} in Section \ref{sec:connection}):}

We first develop a rigorous framework that describes the MF limit
and establishes its connection with a large-width SGD-trained three-layer
network. Here we propose the new idea of a \textit{neuronal embedding},
which comprises of an appropriate non-evolving probability space that
encapsulates neural networks of arbitrary sizes. This probability
space is in general abstract and is constructed according to the (not
necessarily i.i.d.) initialization scheme of the neural network. This
idea addresses directly the intertwined action of multiple symmetry
groups, which is the aforementioned conceptual obstacle (\cite{nguyen2019mean}),
thereby covering setups that cannot be handled by formulations in
\cite{araujo2019mean,sirignano2019mean} (see also Section \ref{sec:Conclusion}
for a comparison). Our analysis follows the technique from \cite{sznitman1991topics,mei2018mean}
and gives a quantitative statement: in particular, the MF limit yields
a good approximation of the neural network as long as $n_{\min}^{-1}\log n_{\max}\ll1$
independent of the data dimension, where $n_{\min}$ and $n_{\max}$
are the minimum and maximum of the widths.

\paragraph*{Step 2 (Theorem \ref{thm:global-optimum-3} in Section \ref{sec:global_conv}):}

We prove that the MF limit, given by our framework, converges to the
global optimum under suitable regularity and convergence mode assumptions.
Several elements of our proof are inspired by \cite{chizat2018};
the technique in their work however does not generalize to our three-layer
setup. Unlike previous two-layer analyses, we do not exploit convexity;
instead we make use of a new element: a universal approximation property.
The result turns out to be conceptually new: global convergence can
be achieved even when the loss function is non-convex. An important
crux of the proof is to show that the universal approximation property
holds at \textit{any} finite training time (but not necessarily at
convergence, i.e. at infinite time, since the property may not realistically
hold at convergence). 

Together these two results imply a positive statement on the optimization
efficiency of SGD-trained unregularized feedforward three-layer networks
(Corollary \ref{cor:global-optimum-3-NN}). Our results can be extended
to the general multilayer case -- with new ideas on top and significantly
more technical works -- or used to obtain new global convergence
guarantees in the two-layer case (\cite{nguyen2020rigorous,pham2020note}).
We choose to keep the current paper concise with the three-layer case
being a prototypical setup that conveys several of the basic ideas.
Complete proofs are presented in appendices.

\paragraph*{Notations.}

$K$ denotes a generic constant that may change from line to line.
$\left|\cdot\right|$ denotes the absolute value for a scalar and
the Euclidean norm for a vector. For an integer $n$, we let $\left[n\right]=\left\{ 1,...,n\right\} $.

\section{Three-layer neural networks and the mean field limit\label{sec:setup}}

\subsection{Three-layer neural network\label{subsec:Three-layer-neural-network}}

We consider the following three-layer network at time $k\in\mathbb{N}_{\geq0}$
that takes as input $x\in\mathbb{R}^{d}$:
\begin{align}
\hat{{\bf y}}\left(x;\mathbf{W}\left(k\right)\right) & =\varphi_{3}\left(\mathbf{H}_{3}\left(x;\mathbf{W}\left(k\right)\right)\right),\label{eq:three-layer-nn}\\
\mathbf{H}_{3}\left(x;\mathbf{W}\left(k\right)\right) & =\frac{1}{n_{2}}\sum_{j_{2}=1}^{n_{2}}{\bf w}_{3}\left(k,j_{2}\right)\varphi_{2}\left({\bf H}_{2}\left(x,j_{2};\mathbf{W}\left(k\right)\right)\right),\nonumber \\
{\bf H}_{2}\left(x,j_{2};\mathbf{W}\left(k\right)\right) & =\frac{1}{n_{1}}\sum_{j_{1}=1}^{n_{1}}{\bf w}_{2}\left(k,j_{1},j_{2}\right)\varphi_{1}\left(\left\langle {\bf w}_{1}\left(k,j_{1}\right),x\right\rangle \right),\nonumber 
\end{align}
in which $\mathbf{W}\left(k\right)=\left({\bf w}_{1}\left(k,\cdot\right),{\bf w}_{2}\left(k,\cdot,\cdot\right),{\bf w}_{3}\left(k,\cdot\right)\right)$
consists of the weights\footnote{To absorb first layer's bias term to $\mathbf{w}_{1}$, we assume
the input $x$ to have $1$ appended to the last entry.} ${\bf w}_{1}\left(k,j_{1}\right)\in\mathbb{R}^{d}$, ${\bf w}_{2}\left(k,j_{1},j_{2}\right)\in\mathbb{R}$
and ${\bf w}_{3}\left(k,j_{2}\right)\in\mathbb{R}$. Here $\varphi_{1}:\;\mathbb{R}\to\mathbb{R}$,
$\varphi_{2}:\;\mathbb{R}\to\mathbb{R}$ and $\varphi_{3}:\;\mathbb{R}\to\mathbb{R}$
are the activation functions, and the network has widths $\left\{ n_{1},n_{2}\right\} $.

We train the network with SGD w.r.t. the loss ${\cal L}:\;\mathbb{R}\times\mathbb{R}\to\mathbb{R}_{\geq0}$.
We assume that at each time $k$, we draw independently a fresh sample
$z\left(k\right)=\left(x\left(k\right),y\left(k\right)\right)\in\mathbb{R}^{d}\times\mathbb{R}$
from a training distribution ${\cal P}$. Given an initialization
$\mathbf{W}\left(0\right)$, we update $\mathbf{W}\left(k\right)$
according to
\begin{align*}
{\bf w}_{3}\left(k+1,j_{2}\right) & ={\bf w}_{3}\left(k,j_{2}\right)-\epsilon\ensuremath{\xi}_{3}\left(k\epsilon\right){\rm Grad}_{3}\left(z\left(k\right),j_{2};\mathbf{W}\left(k\right)\right),\\
{\bf w}_{2}\left(k+1,j_{1},j_{2}\right) & ={\bf w}_{2}\left(k,j_{1},j_{2}\right)-\epsilon\ensuremath{\xi}_{2}\left(k\epsilon\right){\rm Grad}_{2}\left(z\left(k\right),j_{1},j_{2};\mathbf{W}\left(k\right)\right),\\
{\bf w}_{1}\left(k+1,j_{1}\right) & ={\bf w}_{1}\left(k,j_{1}\right)-\epsilon\ensuremath{\xi}_{1}\left(k\epsilon\right){\rm Grad}_{1}\left(z\left(k\right),j_{1};\mathbf{W}\left(k\right)\right),
\end{align*}
in which $j_{1}=1,...,n_{1}$, $j_{2}=1,...,n_{2}$, $\epsilon\in\mathbb{R}_{>0}$
is the learning rate, $\xi_{i}:\;\mathbb{R}_{\geq0}\mapsto\mathbb{R}_{\geq0}$
is the learning rate schedule for $\mathbf{w}_{i}$, and for $z=\left(x,y\right)$,
we define
\begin{align*}
{\rm Grad}_{3}\left(z,j_{2};\mathbf{W}\left(k\right)\right) & =\partial_{2}{\cal L}\left(y,\hat{\mathbf{y}}\left(x;\mathbf{W}\left(k\right)\right)\right)\varphi_{3}'\left(\mathbf{H}_{3}\left(x;\mathbf{W}\left(k\right)\right)\right)\varphi_{2}\left({\bf H}_{2}\left(x,j_{2};\mathbf{W}\left(k\right)\right)\right),\\
{\rm Grad}_{2}\left(z,j_{1},j_{2};\mathbf{W}\left(k\right)\right) & =\Delta_{2}^{\mathbf{H}}\left(z,j_{2};\mathbf{W}\left(k\right)\right)\varphi_{1}\left(\left\langle {\bf w}_{1}\left(k,j_{1}\right),x\right\rangle \right),\\
{\rm Grad}_{1}\left(z,j_{1};\mathbf{W}\left(k\right)\right) & =\bigg(\frac{1}{n_{2}}\sum_{j_{2}=1}^{n_{2}}\Delta_{2}^{\mathbf{H}}\left(z,j_{2};\mathbf{W}\left(k\right)\right){\bf w}_{2}\left(k,j_{1},j_{2}\right)\bigg)\varphi_{1}'\left(\left\langle {\bf w}_{1}\left(k,j_{1}\right),x\right\rangle \right)x,\\
\Delta_{2}^{\mathbf{H}}\left(z,j_{2};\mathbf{W}\left(k\right)\right) & =\partial_{2}{\cal L}\left(y,\hat{\mathbf{y}}\left(x;\mathbf{W}\left(k\right)\right)\right)\varphi_{3}'\left(\mathbf{H}_{3}\left(x;\mathbf{W}\left(k\right)\right)\right){\bf w}_{3}\left(k,j_{2}\right)\varphi_{2}'\left({\bf H}_{2}\left(x,j_{2};\mathbf{W}\left(k\right)\right)\right).
\end{align*}
We note that this setup follows the same scaling w.r.t. $n_{1}$ and
$n_{2}$, which is applied to both the forward pass and the learning
rates in the backward pass, as \cite{nguyen2019mean}.

\subsection{Mean field limit\label{subsec:Mean-field-limit}}

The MF limit is a continuous-time infinite-width analog of the neural
network under training. To describe it, we first introduce the concept
of a \textit{neuronal ensemble}. Given a product probability space
$\left(\Omega,{\cal F},P\right)=\left(\Omega_{1}\times\Omega_{2},{\cal F}_{1}\times{\cal F}_{1},P_{1}\times P_{2}\right)$,
we independently sample $C_{i}\sim P_{i}$, $i=1,2$. In the following,
we use $\mathbb{E}_{C_{i}}$ to denote the expectation w.r.t. the
random variable $C_{i}\sim P_{i}$ and $c_{i}$ to denote an arbitrary
point $c_{i}\in\Omega_{i}$. The space $\left(\Omega,{\cal F},P\right)$
is referred to as a \textit{neuronal ensemble}.

Given a neuronal ensemble $\left(\Omega,{\cal F},P\right)$, the MF
limit is described by a time-evolving system with state/parameter
$W\left(t\right)$, where the time $t\in\mathbb{R}_{\geq0}$ and $W\left(t\right)=\left(w_{1}\left(t,\cdot\right),w_{2}\left(t,\cdot,\cdot\right),w_{3}\left(t,\cdot\right)\right)$
with $w_{1}:\,\mathbb{R}_{\geq0}\times\Omega_{1}\to\mathbb{R}^{d}$,
$w_{2}:\,\mathbb{R}_{\geq0}\times\Omega_{1}\times\Omega_{2}\to\mathbb{R}$
and $w_{3}:\,\mathbb{R}_{\geq0}\times\Omega_{2}\to\mathbb{R}$. It
entails the quantities: 
\begin{align*}
\hat{y}\left(x;W\left(t\right)\right) & =\varphi_{3}\left(H_{3}\left(x;W\left(t\right)\right)\right),\\
H_{3}\left(x;W\left(t\right)\right) & =\mathbb{E}_{C_{2}}\left[w_{3}\left(t,C_{2}\right)\varphi_{2}\left(H_{2}\left(x,C_{2};W\left(t\right)\right)\right)\right],\\
H_{2}\left(x,c_{2};W\left(t\right)\right) & =\mathbb{E}_{C_{1}}\left[w_{2}\left(t,C_{1},c_{2}\right)\varphi_{1}\left(\left\langle w_{1}\left(t,C_{1}\right),x\right\rangle \right)\right].
\end{align*}
Here for each $t\in\mathbb{R}_{\geq0}$, $w_{1}\left(t,\cdot\right)$
is $\left(\Omega_{1},{\cal F}_{1}\right)$-measurable, and similar
for $w_{2}\left(t,\cdot,\cdot\right)$, $w_{3}\left(t,\cdot\right)$.
The MF limit evolves according to a continuous-time dynamics, described
by a system of ODEs, which we refer to as the \textit{MF ODEs}. Specifically,
given an initialization $W\left(0\right)=\left(w_{1}\left(0,\cdot\right),w_{2}\left(0,\cdot,\cdot\right),w_{3}\left(0,\cdot\right)\right)$,
the dynamics solves:
\begin{align*}
\partial_{t}w_{3}\left(t,c_{2}\right) & =-\xi_{3}\left(t\right)\Delta_{3}\left(c_{2};W\left(t\right)\right),\\
\partial_{t}w_{2}\left(t,c_{1},c_{2}\right) & =-\xi_{2}\left(t\right)\Delta_{2}\left(c_{1},c_{2};W\left(t\right)\right),\\
\partial_{t}w_{1}\left(t,c_{1}\right) & =-\xi_{1}\left(t\right)\Delta_{1}\left(c_{1};W\left(t\right)\right).
\end{align*}
Here $c_{1}\in\Omega_{1}$, $c_{2}\in\Omega_{2}$, $\mathbb{E}_{Z}$
denotes the expectation w.r.t. the data $Z=\left(X,Y\right)\sim{\cal P}$,
and for $z=\left(x,y\right)$, we define
\begin{align*}
\Delta_{3}\left(c_{2};W\left(t\right)\right) & =\mathbb{E}_{Z}\left[\partial_{2}{\cal L}\left(Y,\hat{y}\left(X;W\left(t\right)\right)\right)\varphi_{3}'\left(H_{3}\left(X;W\left(t\right)\right)\right)\varphi_{2}\left(H_{2}\left(X,c_{2};W\left(t\right)\right)\right)\right],\\
\Delta_{2}\left(c_{1},c_{2};W\left(t\right)\right) & =\mathbb{E}_{Z}\left[\Delta_{2}^{H}\left(Z,c_{2};W\left(t\right)\right)\varphi_{1}\left(\left\langle w_{1}\left(t,c_{1}\right),X\right\rangle \right)\right],\\
\Delta_{1}\left(c_{1};W\left(t\right)\right) & =\mathbb{E}_{Z}\left[\mathbb{E}_{C_{2}}\left[\Delta_{2}^{H}\left(Z,C_{2};W\left(t\right)\right)w_{2}\left(t,c_{1},C_{2}\right)\right]\varphi_{1}'\left(\left\langle w_{1}\left(t,c_{1}\right),X\right\rangle \right)X\right],\\
\Delta_{2}^{H}\left(z,c_{2};W\left(t\right)\right) & =\partial_{2}{\cal L}\left(y,\hat{y}\left(x;W\left(t\right)\right)\right)\varphi_{3}'\left(H_{3}\left(x;W\left(t\right)\right)\right)w_{3}\left(t,c_{2}\right)\varphi_{2}'\left(H_{2}\left(x,c_{2};W\left(t\right)\right)\right).
\end{align*}

In Appendix \ref{sec:Existence-and-uniqueness-proof}, we show well-posedness
of MF ODEs under the following regularity conditions.
\begin{assumption}[Regularity]
\label{assump:Regularity}We assume that $\varphi_{1}$ and $\varphi_{2}$
are $K$-bounded, $\varphi_{1}'$, $\varphi_{2}'$ and $\varphi_{3}'$
are $K$-bounded and $K$-Lipschitz, $\varphi_{2}'$ and $\varphi_{3}'$
are non-zero everywhere, $\partial_{2}{\cal L}\left(\cdot,\cdot\right)$
is $K$-Lipschitz in the second variable and $K$-bounded, and $\left|X\right|\leq K$
with probability $1$. Furthermore $\xi_{1}$, $\xi_{2}$ and $\xi_{3}$
are $K$-bounded and $K$-Lipschitz.
\end{assumption}

\begin{thm}
\label{thm:existence ODE}Under Assumption \ref{assump:Regularity},
given any neuronal ensemble and an initialization $W\left(0\right)$
such that\footnote{We recall the definition of ${\rm ess\text{-}sup}$ in Appendix \ref{sec:Notational-preliminaries}.}
${\rm ess\text{-}sup}\left|w_{2}\left(0,C_{1},C_{2}\right)\right|,$
${\rm ess\text{-}sup}\left|w_{3}\left(0,C_{2}\right)\right|\leq K$,
there exists a unique solution $W$ to the MF ODEs on $t\in[0,\infty)$.
\end{thm}

An example of a suitable setup is $\varphi_{1}=\varphi_{2}=\tanh$,
$\varphi_{3}$ is the identity, ${\cal L}$ is the Huber loss, although
a non-convex sufficiently smooth loss function suffices. In fact,
all of our developments can be easily modified to treat the squared
loss with an additional assumption $\left|Y\right|\leq K$ with probability
$1$.

So far, given an arbitrary neuronal ensemble $\left(\Omega,{\cal F},P\right)$,
for each initialization $W\left(0\right)$, we have defined a MF limit
$W\left(t\right)$. The connection with the neural network's dynamics
$\mathbf{W}\left(k\right)$ is established in the next section.

\section{Connection between neural network and its mean field limit\label{sec:connection}}

\subsection{Neuronal embedding and the coupling procedure\label{subsec:Neuronal-Embedding}}

To formalize a connection between the neural network and its MF limit,
we consider their initializations. In practical scenarios, to set
the initial parameters $\mathbf{W}\left(0\right)$ of the neural network,
one typically randomizes $\mathbf{W}\left(0\right)$ according to
some distributional law $\rho$. We note that since the neural network
is defined w.r.t. a set of finite integers $\left\{ n_{1},n_{2}\right\} $,
so is $\rho$. We consider a family $\mathsf{Init}$ of initialization
laws, each of which is indexed by the set $\left\{ n_{1},n_{2}\right\} $:
\[
\mathsf{Init}=\{\rho:\;\rho\text{ is the initialization law of a neural network of size \ensuremath{\left\{  n_{1},n_{2}\right\} } },\;n_{1},n_{2}\in\mathbb{N}_{>0}\}.
\]
This is helpful when one is to take a limit that sends $n_{1},n_{2}\to\infty$,
in which case the size of this family $\left|\mathsf{Init}\right|$
is infinite. More generally we allow $\left|\mathsf{Init}\right|<\infty$
(for example, $\mathsf{Init}$ contains a single law $\rho$ of a
network of size $\left\{ n_{1},n_{2}\right\} $ and hence $\left|\mathsf{Init}\right|=1$).
We make the following crucial definition.
\begin{defn}
\label{def:neuronal_embedding}Given a family of initialization laws
$\mathsf{Init}$, we call $(\Omega,{\cal F},P,\left\{ w_{i}^{0}\right\} _{i=1,2,3})$
a \textit{neuronal embedding }of $\mathsf{Init}$ if the following
holds:

\begin{enumerate}
\item $\left(\Omega,{\cal F},P\right)=\left(\Omega_{1}\times\Omega_{2},{\cal F}_{1}\times{\cal F}_{2},P_{1}\times P_{2}\right)$
a product measurable space. As a reminder, we call it a neuronal ensemble.
\item The deterministic functions $w_{1}^{0}:\;\Omega_{1}\to\mathbb{R}^{d}$,
$w_{2}^{0}:\;\Omega_{1}\times\Omega_{2}\to\mathbb{R}$ and $w_{3}^{0}:\;\Omega_{2}\to\mathbb{R}$
are such that, for each index $\left\{ n_{1},n_{2}\right\} $ of $\mathsf{Init}$
and the law $\rho$ of this index, if --- with an abuse of notations
--- we independently sample $\left\{ C_{i}\left(j_{i}\right)\right\} _{j_{i}\in\left[n_{i}\right]}\sim P_{i}$
i.i.d. for each $i=1,2$, then 
\[
{\rm Law}\left(w_{1}^{0}\left(C_{1}\left(j_{1}\right)\right),\;w_{2}^{0}\left(C_{1}\text{\ensuremath{\left(j_{1}\right)}},C_{2}\left(j_{2}\right)\right),\;w_{3}^{0}\left(C_{2}\text{\ensuremath{\left(j_{2}\right)}}\right),\;\;j_{i}\in\left[n_{i}\right],\;i=1,2\right)=\rho.
\]
\end{enumerate}
\end{defn}

To proceed, given $\mathsf{Init}$ and $\left\{ n_{1},n_{2}\right\} $
in its index set, we perform the following \textit{coupling procedure}:
\begin{enumerate}
\item Let $(\Omega,{\cal F},P,\left\{ w_{i}^{0}\right\} _{i=1,2,3})$ be
a neuronal embedding of $\mathsf{Init}$.
\item We form the MF limit $W\left(t\right)$ (for $t\in\mathbb{R}_{\geq0}$)
associated with the neuronal ensemble $\left(\Omega,{\cal F},P\right)$
by setting the initialization $W\left(0\right)$ to $w_{1}\left(0,\cdot\right)=w_{1}^{0}\left(\cdot\right)$,
$w_{2}\left(0,\cdot,\cdot\right)=w_{2}^{0}\left(\cdot,\cdot\right)$
and $w_{3}\left(0,\cdot\right)=w_{3}^{0}\left(\cdot\right)$ and running
the MF ODEs described in Section \ref{subsec:Mean-field-limit}.
\item We independently sample $C_{i}\left(j_{i}\right)\sim P_{i}$ for $i=1,2$
and $j_{i}=1,...,n_{i}$. We then form the neural network initialization
$\mathbf{W}\left(0\right)$ with $\mathbf{w}_{1}\left(0,j_{1}\right)=w_{1}^{0}\left(C_{1}\left(j_{1}\right)\right)$,
$\mathbf{w}_{2}\left(0,j_{1},j_{2}\right)=w_{2}^{0}\left(C_{1}\left(j_{1}\right),C_{2}\left(j_{2}\right)\right)$
and $\mathbf{w}_{3}\left(0,j_{2}\right)=w_{3}^{0}\left(C_{2}\left(j_{2}\right)\right)$
for $j_{1}\in\left[n_{1}\right]$, $j_{2}\in\left[n_{2}\right]$.
We obtain the network's trajectory $\mathbf{W}\left(k\right)$ for
$k\in\mathbb{N}_{\geq0}$ as in Section \ref{subsec:Three-layer-neural-network},
with the data $z\left(k\right)$ generated independently of $\left\{ C_{i}\left(j_{i}\right)\right\} _{i=1,2}$
and hence $\mathbf{W}\left(0\right)$.
\end{enumerate}
We can then define a measure of closeness between $\mathbf{W}\left(\left\lfloor t/\epsilon\right\rfloor \right)$
and $W\left(t\right)$ for $t\in\left[0,T\right]$: 
\begin{align}
\mathscr{D}_{T}\left(W,\mathbf{W}\right)=\sup\big\{ & \left|\mathbf{w}_{1}\left(\left\lfloor t/\epsilon\right\rfloor ,j_{1}\right)-w_{1}\left(t,C_{1}\left(j_{1}\right)\right)\right|,\;\left|\mathbf{w}_{2}\left(\left\lfloor t/\epsilon\right\rfloor ,j_{1},j_{2}\right)-w_{2}\left(t,C_{1}\left(j_{1}\right),C_{2}\left(j_{2}\right)\right)\right|,\nonumber \\
 & \left|\mathbf{w}_{3}\left(\left\lfloor t/\epsilon\right\rfloor ,j_{2}\right)-w_{3}\left(t,C_{2}\left(j_{2}\right)\right)\right|:\;t\leq T,\;j_{1}\leq n_{1},\;j_{2}\leq n_{2}\big\}.\label{eq:dist_W}
\end{align}
Note that $W\left(t\right)$ is a deterministic trajectory independent
of $\left\{ n_{1},n_{2}\right\} $, whereas $\mathbf{W}\left(k\right)$
is random for all $k\in\mathbb{N}_{\geq0}$ due to the randomness
of $\left\{ C_{i}\left(j_{i}\right)\right\} _{i=1,2}$ and the generation
of the training data $z\left(k\right)$. Similarly $\mathscr{D}_{T}\left(W,\mathbf{W}\right)$
is a random quantity.

The idea of the coupling procedure is closely related to the coupling
argument in \cite{sznitman1991topics,mei2018mean}. Here, instead
of playing the role of a proof technique, the coupling serves as a
vehicle to establish the connection between $W$ and $\mathbf{W}$
on the basis of the neuronal embedding. This connection is shown in
Theorem \ref{thm:gradient descent coupling} below, which gives an
upper bound on $\mathscr{D}_{T}\left(W,\mathbf{W}\right)$.

We note that the coupling procedure can be carried out to provide
a connection between $W$ and $\mathbf{W}$ \textsl{as long as there
exists a neuronal embedding for $\mathsf{Init}$}. Later in Section
\ref{subsec:I.i.d.-initialization}, we show that for a common initialization
scheme (in particular, i.i.d. initialization) for $\mathsf{Init}$,
there exists a neuronal embedding. Theorem \ref{thm:gradient descent coupling}
applies to, but is not restricted to, this initialization scheme.

\subsection{Main result: approximation by the MF limit}
\begin{assumption}[Initialization of second and third layers]
\label{assump:Regularity-init}We assume that ${\rm ess\text{-}sup}\left|w_{2}^{0}\left(C_{1},C_{2}\right)\right|$,
${\rm ess\text{-}sup}\left|w_{3}^{0}\left(C_{2}\right)\right|\leq K$,
where $w_{2}^{0}$ and $w_{3}^{0}$ are as described in Definition
\ref{def:neuronal_embedding}.
\end{assumption}

\begin{thm}
\label{thm:gradient descent coupling}Given a family $\mathsf{Init}$
of initialization laws and a tuple $\left\{ n_{1},n_{2}\right\} $
that is in the index set of $\mathsf{Init}$, perform the coupling
procedure as described in Section \ref{subsec:Neuronal-Embedding}.
Fix a terminal time $T\in\epsilon\mathbb{N}_{\geq0}$. Under Assumptions
\ref{assump:Regularity} and \ref{assump:Regularity-init}, for $\epsilon\leq1$,
we have with probability at least $1-2\delta$,
\begin{align*}
\mathscr{D}_{T}\left(W,\mathbf{W}\right) & \leq e^{K_{T}}\left(\frac{1}{\sqrt{n_{\min}}}+\sqrt{\epsilon}\right)\log^{1/2}\left(\frac{3\left(T+1\right)n_{\max}^{2}}{\delta}+e\right)\equiv\mathsf{err}_{\delta,T}\left(\epsilon,n_{1},n_{2}\right),
\end{align*}
in which $n_{\min}=\min\left\{ n_{1},n_{2}\right\} $, $n_{\max}=\max\left\{ n_{1},n_{2}\right\} $,
and $K_{T}=K\left(1+T^{K}\right)$.
\end{thm}

The theorem gives a connection between $\mathbf{W}\left(\left\lfloor t/\epsilon\right\rfloor \right)$,
which is defined upon finite widths $n_{1}$ and $n_{2}$, and the
MF limit $W\left(t\right)$, whose description is independent of $n_{1}$
and $n_{2}$. It lends a way to extract properties of the neural network
in the large-width regime.
\begin{cor}
\label{cor:gradient descent quality}Under the same setting as Theorem
\ref{thm:gradient descent coupling}, consider any test function $\psi:\mathbb{R}\times\mathbb{R}\to\mathbb{R}$
which is $K$-Lipschitz in the second variable uniformly in the first
variable (an example of $\psi$ is the loss ${\cal L}$). For any
$\delta>0$, with probability at least $1-3\delta$,
\[
\sup_{t\leq T}\left|\mathbb{E}_{Z}\left[\psi\left(Y,\hat{\mathbf{y}}\left(X;\mathbf{W}\left(\left\lfloor t/\epsilon\right\rfloor \right)\right)\right)\right]-\mathbb{E}_{Z}\left[\psi\left(Y,\hat{y}\left(X;W\left(t\right)\right)\right)\right]\right|\leq e^{K_{T}}\mathsf{err}_{\delta,T}\left(\epsilon,n_{1},n_{2}\right).
\]
\end{cor}

These bounds hold for any $n_{1}$ and $n_{2}$, similar to \cite{mei2018mean,araujo2019mean},
in contrast with non-quantitative results in \cite{chizat2018,sirignano2019mean}.
These bounds suggest that $n_{1}$ and $n_{2}$ can be chosen independent
of the data dimension $d$. This agrees with the experiments in \cite{nguyen2019mean},
which found $\text{width}\approx1000$ to be typically sufficient
to observe MF behaviors in networks trained with real-life high-dimensional
data.

We observe that the MF trajectory $W\left(t\right)$ is defined as
per the choice of the neuronal embedding $(\Omega,{\cal F},P,\left\{ w_{i}^{0}\right\} _{i=1,2,3})$,
which may not be unique. On the other hand, the neural network's trajectory
$\mathbf{W}\left(k\right)$ depends on the randomization of the initial
parameters $\mathbf{W}\left(0\right)$ according to an initialization
law from the family $\mathsf{Init}$ (as well as the data $z\left(k\right)$)
and hence is independent of this choice. Another corollary of Theorem
\ref{thm:gradient descent coupling} is that given the same family
$\mathsf{Init}$, the law of the MF trajectory is insensitive to the
choice of the neuronal embedding of $\mathsf{Init}$.
\begin{cor}
\label{cor:MF_insensitivity}Consider a family $\mathsf{Init}$ of
initialization laws, indexed by a set of tuples $\left\{ m_{1},m_{2}\right\} $
that contains a sequence of indices $\left\{ m_{1}\left(m\right),m_{2}\left(m\right):\;m\in\mathbb{N}\right\} $
in which as $m\to\infty$, $\min\left\{ m_{1}\left(m\right),m_{2}\left(m\right)\right\} ^{-1}\log\left(\max\left\{ m_{1}\left(m\right),m_{2}\left(m\right)\right\} \right)\to0$.
Let $W\left(t\right)$ and $\hat{W}\left(t\right)$ be two MF trajectories
associated with two choices of neuronal embeddings of $\mathsf{Init}$,
$(\Omega,{\cal F},P,\left\{ w_{i}^{0}\right\} _{i=1,2,3})$ and $(\hat{\Omega},\hat{{\cal F}},\hat{P},\left\{ \hat{w}_{i}^{0}\right\} _{i=1,2,3})$.
The following statement holds for any $T\geq0$ and any two positive
integers $n_{1}$ and $n_{2}$: if we independently sample $C_{i}\left(j_{i}\right)\sim P_{i}$
and $\hat{C}_{i}\left(j_{i}\right)\sim\hat{P}_{i}$ for $j_{i}\in\left[n_{i}\right]$,
$i=1,2$, then ${\rm Law}\left({\cal W}\left(n_{1},n_{2},T\right)\right)={\rm Law}(\hat{{\cal W}}\left(n_{1},n_{2},T\right))$,
where we define ${\cal W}\left(n_{1},n_{2},T\right)$ as the below
collection w.r.t. $W\left(t\right)$, and similarly define $\hat{{\cal W}}\left(n_{1},n_{2},T\right)$
w.r.t. $\hat{W}\left(t\right)$:
\begin{align*}
{\cal W}\left(n_{1},n_{2},T\right) & =\big\{ w_{1}\left(t,C_{1}\left(j_{1}\right)\right),\;w_{2}\left(t,C_{1}\left(j_{1}\right),C_{2}\left(j_{2}\right)\right),\;w_{3}\left(t,C_{2}\left(j_{2}\right)\right):\\
 & \qquad j_{1}\in\left[n_{1}\right],\;j_{2}\in\left[n_{2}\right],\;t\in\left[0,T\right]\big\}.
\end{align*}
\end{cor}

The proofs are deferred to Appendix \ref{sec:Connection-proof}.

\section{Convergence to global optima\label{sec:global_conv}}

In this section, we prove a global convergence guarantee for three-layer
neural networks via the MF limit. We consider a common class of initialization:
i.i.d. initialization.

\subsection{I.i.d. initialization\label{subsec:I.i.d.-initialization}}
\begin{defn}
An initialization law $\rho$ for a neural network of size $\left\{ n_{1},n_{2}\right\} $
is called $\left(\rho^{1},\rho^{2},\rho^{3}\right)$-i.i.d. initialization
(or i.i.d. initialization, for brevity), where $\rho^{1}$, $\rho^{2}$
and $\rho^{3}$ are probability measures over $\mathbb{R}^{d}$, $\mathbb{R}$
and $\mathbb{R}$ respectively, if $\left\{ \mathbf{w}_{1}\left(0,j_{1}\right)\right\} _{j_{1}\in\left[n_{1}\right]}$
are generated i.i.d. according to $\rho^{1}$, $\left\{ \mathbf{w}_{2}\left(0,j_{1},j_{2}\right)\right\} _{j_{1}\in\left[n_{1}\right],\;j_{2}\in\left[n_{2}\right]}$
are generated i.i.d. according to $\rho^{2}$ and $\left\{ \mathbf{w}_{3}\left(0,j_{2}\right)\right\} _{j_{2}\in\left[n_{2}\right]}$
are generated i.i.d. according to $\rho^{3}$, and $\mathbf{w}_{1}$,
$\mathbf{w}_{2}$ and $\mathbf{w}_{3}$ are independent.
\end{defn}

Observe that given $\left(\rho^{1},\rho^{2},\rho^{3}\right)$, one
can build a family $\mathsf{Init}$ of i.i.d. initialization laws
that contains \textsl{any} index set $\left\{ n_{1},n_{2}\right\} $.
Furthermore i.i.d. initializations are supported by our framework,
as stated in the following proposition and proven in Appendix \ref{sec:Global-convergence-proof}.
\begin{prop}
\label{prop:iid_law_det_representable}There exists a neuronal embedding
$\left(\Omega,{\cal F},P,\left\{ w_{i}^{0}\right\} _{i=1,2,3}\right)$
for any family $\mathsf{Init}$ of initialization laws, which are
$\left(\rho^{1},\rho^{2},\rho^{3}\right)$-i.i.d.
\end{prop}

\subsection{Main result: global convergence}

To measure the learning quality, we consider the loss averaged over
the data $Z\sim{\cal P}$:
\[
\mathscr{L}\left(V\right)=\mathbb{E}_{Z}\left[{\cal L}\left(Y,\hat{y}\left(X;V\right)\right)\right],
\]
where $V=\left(v_{1},v_{2},v_{3}\right)$ is a set of three measurable
functions $v_{1}:\;\Omega_{1}\to\mathbb{R}^{d}$, $v_{2}:\;\Omega_{1}\times\Omega_{2}\to\mathbb{R}$,
$v_{3}:\;\Omega_{2}\to\mathbb{R}$.
\begin{assumption}
\label{assump:three-layers}Consider a neuronal embedding $\left(\Omega,{\cal F},P,\left\{ w_{i}^{0}\right\} _{i=1,2,3}\right)$
of the $\left(\rho^{1},\rho^{2},\rho^{3}\right)$-i.i.d. initialization,
and the associated MF limit with initialization $W\left(0\right)$
such that $w_{1}\left(0,\cdot\right)=w_{1}^{0}\left(\cdot\right)$,
$w_{2}\left(0,\cdot,\cdot\right)=w_{2}^{0}\left(\cdot,\cdot\right)$
and $w_{3}\left(0,\cdot\right)=w_{3}^{0}\left(\cdot\right)$. Assume:
\begin{enumerate}
\item Support: The support of $\rho^{1}$ is $\mathbb{R}^{d}$.
\item Convergence mode: There exist limits $\bar{w}_{1}$, $\bar{w}_{2}$
and $\bar{w}_{3}$ such that as $t\to\infty$,
\begin{align}
\mathbb{E}\left[\left(1+\left|\bar{w}_{3}(C_{2})\right|\right)\left|\bar{w}_{3}(C_{2})\right|\left|\bar{w}_{2}(C_{1},C_{2})\right|\left|w_{1}(t,C_{1})-\bar{w}_{1}(C_{1})\right|\right] & \to0,\label{eq:Assump_w1}\\
\mathbb{E}\left[\left(1+\left|\bar{w}_{3}(C_{2})\right|\right)\left|\bar{w}_{3}(C_{2})\right|\left|w_{2}(t,C_{1},C_{2})-\bar{w}_{2}(C_{1},C_{2})\right|\right] & \to0,\label{eq:Assump_w2}\\
\mathbb{E}\left[\left(1+\left|\bar{w}_{3}(C_{2})\right|\right)\left|w_{3}(t,C_{2})-\bar{w}_{3}(C_{2})\right|\right] & \to0,\label{eq:Assump_w3}\\
{\rm ess\text{-}sup}\mathbb{E}_{C_{2}}\left[\left|\partial_{t}w_{2}\left(t,C_{1},C_{2}\right)\right|\right] & \to0.\label{eq:Assump_esssup}
\end{align}
\item Universal approximation: $\left\{ \varphi_{1}\left(\left\langle u,\cdot\right\rangle \right):\;u\in\mathbb{R}^{d}\right\} $
has dense span in $L^{2}\left({\cal P}_{X}\right)$ (the space of
square integrable functions w.r.t. ${\cal P}_{X}$ the distribution
of the input $X$).
\end{enumerate}
\end{assumption}

Assumption \ref{assump:three-layers} is inspired by the work \cite{chizat2018}
on two-layer networks, with certain differences. Assumptions \ref{assump:three-layers}.1
and \ref{assump:three-layers}.3 are natural in neural network learning
(\cite{cybenko1989approximation,chen1995universal}), while we note
\cite{chizat2018} does not utilize universal approximation. Similar
to \cite{chizat2018}, Assumption \ref{assump:three-layers}.2 is
technical and does not seem removable. Note that this assumption specifies
the mode of convergence and is not an assumption on the limits $\bar{w}_{1}$,
$\bar{w}_{2}$ and $\bar{w}_{3}$. Specifically conditions (\ref{eq:Assump_w1})-(\ref{eq:Assump_w3})
are similar to the convergence assumption in \cite{chizat2018}. We
differ from \cite{chizat2018} fundamentally in the essential supremum
condition (\ref{eq:Assump_esssup}). On one hand, this condition helps
avoid the Morse-Sard type condition in \cite{chizat2018}, which is
difficult to verify in general and not simple to generalize to the
three-layer case. On the other hand, it turns out to be a natural
assumption to make, in light of Remark \ref{rem:Converse} below.

We now state the main result of the section. The proof is in Appendix
\ref{sec:Global-convergence-proof}.
\begin{thm}
\label{thm:global-optimum-3}Consider a neuronal embedding $\left(\Omega,{\cal F},P,\left\{ w_{i}^{0}\right\} _{i=1,2,3}\right)$
of $\left(\rho^{1},\rho^{2},\rho^{3}\right)$-i.i.d. initialization.
Consider the MF limit corresponding to the network (\ref{eq:three-layer-nn}),
such that they are coupled together by the coupling procedure in Section
\ref{subsec:Neuronal-Embedding}, under Assumptions \ref{assump:Regularity},
\ref{assump:Regularity-init} and \ref{assump:three-layers}. For
simplicity, assume $\xi_{1}\left(\cdot\right)=\xi_{2}\left(\cdot\right)=1$.
Further assume either:
\begin{itemize}
\item (untrained third layer) $\xi_{3}\left(\cdot\right)=0$ and $w_{3}^{0}\left(C_{2}\right)\neq0$
with a positive probability, or
\item (trained third layer) $\xi_{3}\left(\cdot\right)=1$ and $\mathscr{L}\left(w_{1}^{0},w_{2}^{0},w_{3}^{0}\right)<\mathbb{E}_{Z}\left[{\cal L}\left(Y,\varphi_{3}\left(0\right)\right)\right]$.
\end{itemize}
Then the following hold:
\begin{itemize}
\item Case 1 (convex loss): If ${\cal L}$ is convex in the second variable,
then
\[
\lim_{t\to\infty}\mathscr{L}\left(W\left(t\right)\right)=\inf_{V}\mathscr{L}\left(V\right)=\inf_{\tilde{y}:\;\mathbb{R}^{d}\to\mathbb{R}}\mathbb{E}_{Z}\left[{\cal L}\left(Y,\tilde{y}\left(X\right)\right)\right].
\]
\item Case 2 (generic non-negative loss): Suppose that $\partial_{2}{\cal L}\left(y,\hat{y}\right)=0$
implies ${\cal L}\left(y,\hat{y}\right)=0$. If $y=y(x)$ is a function
of $x$, then $\mathscr{L}\left(W\left(t\right)\right)\to0$ as $t\to\infty$.
\end{itemize}
\end{thm}

Remarkably here the theorem allows for non-convex losses. A further
inspection of the proof shows that no convexity-based property is
used in Case 2 (see, for instance, the high-level proof sketch in
Section \ref{subsec:High-level-idea}); in Case 1, the key steps in
the proof are the same, and the convexity of the loss function serves
as a convenient technical assumption to handle the arbitrary extra
randomness of $Y$ conditional on $X$. We also remark that the same
proof of global convergence should extend beyond the specific fully-connected
architecture considered here. Similar to previous results on SGD-trained
two-layer networks \cite{mei2018mean,chizat2018}, our current result
in the three-layer case is non-quantitative.
\begin{rem}
\label{rem:Converse} Interestingly there is a converse relation between
global convergence and the essential supremum condition (\ref{eq:Assump_esssup}):
under the same setting, global convergence is unattainable if condition
(\ref{eq:Assump_esssup}) does not hold. A similar observation was
made in \cite{wojtowytsch2020convergence} for two-layer ReLU networks.
A precise statement and its proof can be found in Appendix \ref{sec:Converse}.
\end{rem}

The following result is straightforward from Theorem \ref{thm:global-optimum-3}
and Corollary \ref{cor:gradient descent quality}, establishing the
optimization efficiency of the neural network with SGD.
\begin{cor}
\label{cor:global-optimum-3-NN}Consider the neural network (\ref{eq:three-layer-nn}).
Under the same setting as Theorem \ref{thm:global-optimum-3}, in
Case 1,
\[
\lim_{t\to\infty}\lim_{n_{1},n_{2}}\lim_{\epsilon\to0}\mathbb{E}_{Z}\left[{\cal L}\left(Y,\hat{{\bf y}}\left(X;\mathbf{W}\left(\left\lfloor t/\epsilon\right\rfloor \right)\right)\right)\right]=\inf_{f_{1},f_{2},f_{3}}\mathscr{L}\left(f_{1},f_{2},f_{3}\right)=\inf_{\tilde{y}}\mathbb{E}_{Z}\left[{\cal L}\left(Y,\tilde{y}\left(X\right)\right)\right]
\]
in probability, where the limit of the widths is such that $\min\left\{ n_{1},n_{2}\right\} ^{-1}\log\left(\max\left\{ n_{1},n_{2}\right\} \right)\to0$.
In Case 2, the same holds with the right-hand side being $0$.
\end{cor}

\subsection{High-level idea of the proof\label{subsec:High-level-idea}}

We give a high-level discussion of the proof. This is meant to provide
intuitions and explain the technical crux, so our discussion may simplify
and deviate from the actual proof.

Our first insight is to look at the second layer's weight $w_{2}$.
At convergence time $t=\infty$, we expect to have zero movement and
hence, denoting $W\left(\infty\right)=\left(\bar{w}_{1},\bar{w}_{2},\bar{w}_{3}\right)$:
\[
\Delta_{2}\left(c_{1},c_{2};W\left(\infty\right)\right)=\mathbb{E}_{Z}\left[\Delta_{2}^{H}\left(Z,c_{2};W\left(\infty\right)\right)\varphi_{1}\left(\left\langle \bar{w}_{1}\left(c_{1}\right),X\right\rangle \right)\right]=0,
\]
for $P$-almost every $c_{1}$, $c_{2}$. Suppose for the moment that
we are allowed to make an additional (strong) assumption on the limit
$\bar{w}_{1}$: ${\rm supp}\left(\bar{w}_{1}\left(C_{1}\right)\right)=\mathbb{R}^{d}$.
It implies that the universal approximation property, described in
Assumption \ref{assump:three-layers}, holds at $t=\infty$; more
specifically, it implies $\left\{ \varphi_{1}\left(\left\langle \bar{w}_{1}\left(c_{1}\right),\cdot\right\rangle \right):\;c_{1}\in\Omega_{1}\right\} $
has dense span in $L^{2}\left({\cal P}_{X}\right)$. This thus yields
\[
\mathbb{E}_{Z}\left[\Delta_{2}^{H}\left(Z,c_{2};W\left(\infty\right)\right)\middle|X=x\right]=0,
\]
for ${\cal P}$-almost every $x$. Recalling the definition of $\Delta_{2}^{H}$,
one can then easily show that
\[
\mathbb{E}_{Z}\left[\partial_{2}{\cal L}\left(Y,\hat{y}\left(X;W\left(\infty\right)\right)\right)\middle|X=x\right]=0.
\]
Global convergence follows immediately; for example, in Case 2 of
Theorem \ref{thm:global-optimum-3}, this is equivalent to that $\partial_{2}{\cal L}\left(y\left(x\right),\hat{y}\left(x;W\left(\infty\right)\right)\right)=0$
and hence ${\cal L}\left(y\left(x\right),\hat{y}\left(x;W\left(\infty\right)\right)\right)=0$
for ${\cal P}$-almost every $x$. In short, the gradient flow structure
of the dynamics of $w_{2}$ provides a seamless way to obtain global
convergence. Furthermore there is no critical reliance on convexity.

However this plan of attack has a potential flaw in the strong assumption
that ${\rm supp}\left(\bar{w}_{1}\left(C_{1}\right)\right)=\mathbb{R}^{d}$,
i.e. the universal approximation property holds at convergence time.
Indeed there are setups where it is desirable that ${\rm supp}\left(\bar{w}_{1}\left(C_{1}\right)\right)\neq\mathbb{R}^{d}$
(\cite{mei2018mean,chizat2019sparse}); for instance, it is the case
where the neural network is to learn some ``sparse and spiky'' solution,
and hence the weight distribution at convergence time, if successfully
trained, cannot have full support. On the other hand, one can entirely
expect that if ${\rm supp}\left(w_{1}\left(0,C_{1}\right)\right)=\mathbb{R}^{d}$
initially at $t=0$, then ${\rm supp}\left(w_{1}\left(t,C_{1}\right)\right)=\mathbb{R}^{d}$
at \textsl{any} finite $t\geq0$. The crux of our proof is to show
the latter without assuming ${\rm supp}\left(\bar{w}_{1}\left(C_{1}\right)\right)=\mathbb{R}^{d}$.

This task is the more major technical step of the proof. To that end,
we first show that there exists a mapping $\left(t,u\right)\mapsto M\left(t,u\right)$
that maps from $\left(t,w_{1}\left(0,c_{1}\right)\right)=\left(t,u\right)$
to $w_{1}\left(t,c_{1}\right)$ via a careful measurability argument.
This argument rests on a scheme that exploits the symmetry in the
network evolution. Furthermore the map $M$ is shown to be continuous.
The desired conclusion then follows from an algebraic topology argument
that the map $M$ preserves a homotopic structure through time.

\section{Discussion\label{sec:Conclusion}}

The MF literature is fairly recent. A long line of works (\cite{nitanda2017stochastic,mei2018mean,chizat2018,rotskoff2018neural,sirignano2018mean,wei2018margin,javanmard2019analysis,mei2019mean,alex2019landscape,wojtowytsch2020convergence})
have focused mainly on two-layer neural networks, taking an interacting
particle system approach to describe the MF limiting dynamics as Wasserstein
gradient flows. The three works \cite{nguyen2019mean,araujo2019mean,sirignano2019mean}
independently develop different formulations for the MF limit in multilayer
neural networks, under different assumptions. These works take perspectives
that are different from ours. In particular, while the central object
in \cite{nguyen2019mean} is a new abstract representation of each
individual neuron, our neuronal embedding idea instead takes a keen
view on a whole ensemble of neurons. Likewise our idea is also distant
from \cite{araujo2019mean,sirignano2019mean}: the central objects
in \cite{araujo2019mean} are paths over the weights across layers;
those in \cite{sirignano2019mean} are time-dependent functions of
the initialization, which are simplified upon i.i.d. initializations.

The result of our perspective is a neuronal embedding framework that
allows one to describe the MF limit in a clean and rigorous manner.
In particular, it avoids extra assumptions made in \cite{araujo2019mean,sirignano2019mean}:
unlike our work, \cite{araujo2019mean} assumes untrained first and
last layers and requires non-trivial technical tools; \cite{sirignano2019mean}
takes an unnatural sequential limit $n_{1}\to\infty$ before $n_{2}\to\infty$
and proves a non-quantitative result, unlike Theorem \ref{thm:gradient descent coupling}
which only requires sufficiently large $\min\left\{ n_{1},n_{2}\right\} $.
We note that Theorem \ref{thm:gradient descent coupling} can be extended
to general multilayer networks using the neuronal embedding idea.
The advantages of our framework come from the fact that while MF formulations
in \cite{araujo2019mean,sirignano2019mean} are specific to and exploit
i.i.d. initializations, our formulation does not. Remarkably as shown
in \cite{araujo2019mean}, when there are more than three layers and
no biases, i.i.d. initializations lead to a certain simplifying effect
on the MF limit. On the other hand, our framework supports non-i.i.d.
initializations which avoid the simplifying effect, as long as there
exist suitable neuronal embeddings (\cite{nguyen2020rigorous}). Although
our global convergence result in Theorem \ref{thm:global-optimum-3}
is proven in the context of i.i.d. initializations for three-layer
networks, in the general multilayer case, it turns out that the use
of a special type of non-i.i.d. initialization allows one to prove
a global convergence guarantee (\cite{pham2020note}).

In this aspect, our framework follows closely the spirit of the work
\cite{nguyen2019mean}, whose MF formulation is also not specific
to i.i.d. initializations. Yet though similar in the spirit, \cite{nguyen2019mean}
develops a heuristic formalism and does not prove global convergence.

Global convergence in the two-layer case with convex losses has enjoyed
multiple efforts with a lot of new and interesting results (\cite{mei2018mean,chizat2018,javanmard2019analysis,rotskoff2019global,wei2018margin}).
Our work is the first to establish a global convergence guarantee
for SGD-trained three-layer networks in the MF regime. Our proof sends
a new message that the crucial factor is not necessarily convexity,
but rather that the whole learning trajectory maintains the universal
approximation property of the function class represented by the first
layer's neurons, together with the gradient flow structure of the
second layer's weights. As a remark, our approach can also be applied
to prove a similar global convergence guarantee for two-layer networks,
removing the convex loss assumption in previous works (\cite{nguyen2020rigorous}).
The recent work \cite{lu2020mean} on a MF resnet model (a composition
of many two-layer MF networks) and a recent update of \cite{sirignano2019mean}
essentially establish conditions of stationary points to be global
optima. They however require strong assumptions on the support of
the limit point. As explained in Section \ref{subsec:High-level-idea},
we analyze the training dynamics without such assumption and in fact
allow it to be violated.

Our global convergence result is non-quantitative. An important, highly
challenging future direction is to develop a quantitative version
of global convergence; previous works on two-layer networks \cite{javanmard2019analysis,wei2018margin,rotskoff2019global,chizat2019sparse}
have done so under sophisticated modifications of the architecture
and training algorithms.

Finally we remark that our insights here can be applied to prove similar
global convergence guarantees and derive other sufficient conditions
for global convergence of two-layer or multilayer networks (\cite{nguyen2020rigorous,pham2020note}).

\section*{Acknowledgement}

H. T. Pham would like to thank Jan Vondrak for many helpful discussions
and in particular for the shorter proof of Lemma \ref{lem:square hoeffding}.
We would like to thank Andrea Montanari for the succinct description
of the difficulty in extending the mean field formulation to the multilayer
case, in that there are multiple symmetry group actions in a multilayer
network.

\bibliographystyle{iclr2021_conference}
\bibliography{iclr2021}

\newpage{}

\appendix

\section{Notational preliminaries\label{sec:Notational-preliminaries}}

For a real-valued random variable $Z$ defined on a probability space
$(\Omega,{\cal F},P)$, we recall
\[
{\rm ess\text{-}sup}Z=\inf\left\{ z\in\mathbb{R}:\;P\left(Z>z\right)=0\right\} .
\]

We also introduce some convenient definitions which we use throughout
the appendices. For a set of neural network's parameter $\mathbf{W}$,
we define
\[
\interleave\mathbf{W}\interleave_{T}=\max\Big\{\max_{j_{1}\leq n_{1},\;j_{2}\leq n_{2}}\sup_{t\leq T}\left|\mathbf{w}_{2}\left(\left\lfloor t/\epsilon\right\rfloor ,j_{1},j_{2}\right)\right|,\;\max_{j_{2}\leq n_{2}}\sup_{t\leq T}\left|\mathbf{w}_{3}\left(\left\lfloor t/\epsilon\right\rfloor ,j_{2}\right)\right|\Big\}.
\]
Similarly for a set of MF parameters $W$, we define:
\[
\interleave W\interleave_{T}=\max\Big\{{\rm ess\text{-}sup}\sup_{t\leq T}\left|w_{2}\left(t,C_{1},C_{2}\right)\right|,\;{\rm ess\text{-}sup}\sup_{t\leq T}\left|w_{3}\left(t,C_{2}\right)\right|\Big\}.
\]
For two sets of neural network's parameters $\mathbf{W}',\mathbf{W}''$,
we define their distance:
\begin{align*}
\left\Vert \mathbf{W}'-\mathbf{W}''\right\Vert _{T}=\sup\big\{ & \left|\mathbf{w}_{1}'\left(\left\lfloor t/\epsilon\right\rfloor ,j_{1}\right)-\mathbf{w}_{1}''\left(\left\lfloor t/\epsilon\right\rfloor ,j_{1}\right)\right|,\;\left|\mathbf{w}_{2}'\left(\left\lfloor t/\epsilon\right\rfloor ,j_{1},j_{2}\right)-\mathbf{w}_{2}''\left(\left\lfloor t/\epsilon\right\rfloor ,j_{1},j_{2}\right)\right|,\\
 & \left|\mathbf{w}_{3}'\left(\left\lfloor t/\epsilon\right\rfloor ,j_{2}\right)-\mathbf{w}_{3}''\left(\left\lfloor t/\epsilon\right\rfloor ,j_{2}\right)\right|:\;t\in\left[0,T\right],\;j_{1}\in\left[n_{1}\right],\;j_{2}\in\left[n_{2}\right]\big\}.
\end{align*}
Similarly for two sets of MF parameters $W',W''$, we define their
distance:
\begin{align*}
\left\Vert W'-W''\right\Vert _{T}={\rm ess\text{-}sup}\sup_{t\in\left[0,T\right]} & \Big\{\left|w_{1}'\left(t,C_{1}\right)-w_{1}''\left(t,C_{1}\right)\right|,\;\left|w_{2}'\left(t,C_{1},C_{2}\right)-w_{2}''\left(t,C_{1},C_{2}\right)\right|,\\
 & \left|w_{3}'\left(t,C_{2}\right)-w_{3}''\left(t,C_{2}\right)\right|\Big\}.
\end{align*}

\section{Existence and uniqueness of the solution to MF ODEs\label{sec:Existence-and-uniqueness-proof}}

We first collect some a priori estimates.
\begin{lem}
\label{lem:a-priori-MF-norms}Under Assumption \ref{assump:Regularity},
consider a solution $W$ to the MF ODEs with initialization $W\left(0\right)$
such that $\interleave W\interleave_{0}<\infty$. If this solution
exists, it satisfies the following a priori bounds, for any $T\geq0$:
\begin{align*}
{\rm ess\text{-}sup}\sup_{t\leq T}\left|w_{3}\left(t,C_{2}\right)\right| & \leq\interleave W\interleave_{0}+KT\equiv\interleave W\interleave_{0}+K_{0,3}\left(T\right),\\
{\rm ess\text{-}sup}\sup_{t\leq T}\left|w_{2}\left(t,C_{1},C_{2}\right)\right| & \leq\interleave W\interleave_{0}+KTK_{0,3}\left(T\right)\equiv\interleave W\interleave_{0}+K_{0,2}\left(T\right),
\end{align*}
and consequently, $\interleave W\interleave_{T}\leq1+\max\left\{ K_{0,2}\left(T\right),\,K_{0,3}\left(T\right)\right\} .$
\end{lem}

\begin{proof}
The bounds can be obtained easily by bounding the respective initializations
and update quantities separately. In particular,
\begin{align*}
{\rm ess\text{-}sup}\sup_{t\leq T}\left|w_{3}\left(t,C_{2}\right)\right| & \leq{\rm ess\text{-}sup}\left|w_{3}\left(0,C_{2}\right)\right|+T{\rm ess\text{-}sup}\sup_{t\leq T}\left|\frac{\partial}{\partial t}w_{3}\left(t,C_{2}\right)\right|\leq\interleave W\interleave_{0}+KT,\\
{\rm ess\text{-}sup}\sup_{t\leq T}\left|w_{2}\left(t,C_{1},C_{2}\right)\right| & \leq{\rm ess\text{-}sup}\left|w_{2}\left(0,C_{1},C_{2}\right)\right|+T{\rm ess\text{-}sup}\sup_{t\leq T}\left|\frac{\partial}{\partial t}w_{2}\left(t,C_{1},C_{2}\right)\right|\\
 & \leq{\rm ess\text{-}sup}\left|w_{2}\left(0,C_{1},C_{2}\right)\right|+KT{\rm ess\text{-}sup}\sup_{t\leq T}\left|w_{3}\left(t,C_{2}\right)\right|\\
 & \leq\interleave W\interleave_{0}+KTK_{0,3}\left(T\right).
\end{align*}
\end{proof}
Inspired by the a priori bounds in Lemma \ref{lem:a-priori-MF-norms},
given an arbitrary terminal time $T$ and the initialization $W\left(0\right)$,
let us consider:
\begin{itemize}
\item for a tuple $\left(a,b\right)\in\mathbb{R}_{\geq0}^{2}$, a space
${\cal W}_{T}\left(a,b\right)$ of $W'=\left(W'\left(t\right)\right)_{t\leq T}=\left(w_{1}'\left(t,\cdot\right),w_{2}'\left(t,\cdot,\cdot\right),w_{3}'\left(t,\cdot\right)\right)_{t\leq T}$
such that 
\begin{align*}
{\rm ess\text{-}sup}\sup_{t\leq T}\left|w_{3}'\left(t,C_{2}\right)\right| & \leq b,\\
{\rm ess\text{-}sup}\sup_{t\leq T}\left|w_{2}'\left(t,C_{1},C_{2}\right)\right| & \leq a,
\end{align*}
where $w_{1}':\;\mathbb{R}_{\geq0}\times\Omega_{1}\to\mathbb{R}^{d}$,
$w_{2}':\;\mathbb{R}_{\geq0}\times\Omega_{1}\times\Omega_{2}\mapsto\mathbb{R}$,
$w_{3}':\;\mathbb{R}_{\geq0}\times\Omega_{3}\mapsto\mathbb{R}$,
\item for a tuple $\left(a,b\right)\in\mathbb{R}_{\geq0}^{2}$ and $W\left(0\right)$,
a space ${\cal W}_{T}^{+}\left(a,b,W\left(0\right)\right)$ of $W'\in{\cal W}_{T}\left(a,b\right)$
such that $W'\left(0\right)=W\left(0\right)$ additionally (and hence
every $W'$ in this space shares the same initialization $W\left(0\right)$).
\end{itemize}
We equip the spaces with the metric $\left\Vert W'-W''\right\Vert _{T}$.
It is easy to see that both spaces are complete. Note that Lemma \ref{lem:a-priori-MF-norms}
implies, under Assumption \ref{assump:Regularity} and $\interleave W\interleave_{0}<\infty$,
we have any MF solution $W$, if exists, is in ${\cal W}_{T}\left(\interleave W\interleave_{0}+K_{0,2}\left(T\right),\interleave W\interleave_{0}+K_{0,3}\left(T\right)\right)$.
For the proof of Theorem \ref{thm:existence ODE}, we work mainly
with ${\cal W}_{T}^{+}\left(\interleave W\interleave_{0}+K_{0,2}\left(T\right),\interleave W\interleave_{0}+K_{0,3}\left(T\right),W\left(0\right)\right)$,
although several intermediate lemmas are proven in more generality
for other uses.
\begin{lem}
\label{lem:Lipschitz-MF}Under Assumption \ref{assump:Regularity},
for $T\geq0$, any $W',W''\in{\cal W}_{T}\left(a,b\right)$ and almost
every $z\sim{\cal P}$:
\begin{align*}
{\rm ess\text{-}sup}\sup_{t\leq T}\left|\Delta_{2}^{H}\left(z,C_{2};W'\left(t\right)\right)\right| & \leq K_{a,b},\\
{\rm ess\text{-}sup}\sup_{t\leq T}\left|H_{2}\left(x,C_{2};W'\left(t\right)\right)-H_{2}\left(x,C_{2};W''\left(t\right)\right)\right| & \leq K_{a,b}\left\Vert W'-W''\right\Vert _{T},\\
\sup_{t\leq T}\left|H_{3}\left(x;W'\left(t\right)\right)-H_{3}\left(x;W''\left(t\right)\right)\right| & \leq K_{a,b}\left\Vert W'-W''\right\Vert _{T},\\
\sup_{t\leq T}\left|\partial_{2}{\cal L}\left(y,\hat{y}\left(x;W'\left(t\right)\right)\right)-\partial_{2}{\cal L}\left(y,\hat{y}\left(x;W''\left(t\right)\right)\right)\right| & \leq K_{a,b}\left\Vert W'-W''\right\Vert _{T},\\
{\rm ess\text{-}sup}\sup_{t\leq T}\left|\Delta_{2}^{H}\left(z,C_{2};W'\left(t\right)\right)-\Delta_{2}^{H}\left(z,C_{2};W''\left(t\right)\right)\right| & \leq K_{a,b}\left\Vert W'-W''\right\Vert _{T},
\end{align*}
where $K_{a,b}\geq1$ is a generic constant that grows polynomially
with $a$ and $b$.
\end{lem}

\begin{proof}
The first bound is easy to see:
\[
{\rm ess\text{-}sup}\sup_{t\leq T}\left|\Delta_{2}^{H}\left(z,C_{2};W'\left(t\right)\right)\right|\leq{\rm ess\text{-}sup}\sup_{t\leq T}\left|w_{3}'\left(t,C_{2}\right)\right|\leq b.
\]
We prove the second bound, invoking Assumption \ref{assump:Regularity}:
\begin{align*}
 & \left|H_{2}\left(x,C_{2};W'\left(t\right)\right)-H_{2}\left(x,C_{2};W''\left(t\right)\right)\right|\\
 & \leq K\left|w_{2}'\left(t,C_{1},C_{2}\right)\right|\left|\varphi_{1}\left(\left\langle w_{1}'\left(t,C_{1}\right),x\right\rangle \right)-\varphi_{1}\left(\left\langle w_{1}''\left(t,C_{1}\right),x\right\rangle \right)\right|\\
 & \quad+K\left|w_{2}'\left(t,C_{1},C_{2}\right)-w_{2}''\left(t,C_{1},C_{2}\right)\right|\\
 & \leq K\left(\left|w_{2}'\left(t,C_{1},C_{2}\right)\right|+1\right)\left\Vert W'-W''\right\Vert _{T},
\end{align*}
which yields by the fact $W'\in{\cal W}_{T}\left(a,b\right)$:
\[
{\rm ess\text{-}sup}\sup_{t\leq T}\left|H_{2}\left(x,C_{2};W'\left(t\right)\right)-H_{2}\left(x,C_{2};W''\left(t\right)\right)\right|\leq K\left(a+1\right)\left\Vert W'-W''\right\Vert _{T}.
\]
Consequently, we have:
\begin{align*}
\left|H_{3}\left(x;W'\left(t\right)\right)-H_{3}\left(x;W''\left(t\right)\right)\right| & \leq K\left|w_{3}'\left(t,C_{2}\right)\right|\left|\varphi_{2}\left(H_{2}\left(x,C_{2};W'\left(t\right)\right)\right)-\varphi_{2}\left(H_{2}\left(x,C_{2};W''\left(t\right)\right)\right)\right|\\
 & \quad+K\left|w_{3}'\left(t,C_{2}\right)-w_{3}''\left(t,C_{2}\right)\right|\\
 & \leq K\left|w_{3}'\left(t,C_{2}\right)\right|\left|H_{2}\left(x,C_{2};W'\left(t\right)\right)-H_{2}\left(x,C_{2};W''\left(t\right)\right)\right|\\
 & \quad+K\left\Vert W'-W''\right\Vert _{T},\\
\left|\partial_{2}{\cal L}\left(y,\hat{y}\left(x;W'\left(t\right)\right)\right)-\partial_{2}{\cal L}\left(y,\hat{y}\left(x;W''\left(t\right)\right)\right)\right| & \leq K\left|\hat{y}\left(x;W'\left(t\right)\right)-\hat{y}\left(x;W''\left(t\right)\right)\right|\\
 & \leq K\left|H_{3}\left(x;W'\left(t\right)\right)-H_{3}\left(x;W''\left(t\right)\right)\right|,
\end{align*}
which then yield the third and fourth bounds by the fact $W',W''\in{\cal W}_{T}\left(a,b\right)$.
Using these bounds, we obtain the last bound:
\begin{align*}
 & \left|\Delta_{2}^{H}\left(z,C_{2};W'\left(t\right)\right)-\Delta_{2}^{H}\left(z,C_{2};W''\left(t\right)\right)\right|\\
 & \leq K\left|w_{3}'\left(t,C_{2}\right)\right|\Big(\left|\partial_{2}{\cal L}\left(y,\hat{y}\left(x;W'\left(t\right)\right)\right)-\partial_{2}{\cal L}\left(y,\hat{y}\left(x;W''\left(t\right)\right)\right)\right|\\
 & \qquad+\left|H_{3}\left(x;W'\left(t\right)\right)-H_{3}\left(x;W''\left(t\right)\right)\right|+\left|H_{2}\left(x,C_{2};W'\left(t\right)\right)-H_{2}\left(x,C_{2};W''\left(t\right)\right)\right|\Big)\\
 & \quad+K\left|w_{3}'\left(t,C_{2}\right)-w_{3}''\left(t,C_{2}\right)\right|,
\end{align*}
from which the last bound follows.
\end{proof}
To prove Theorem \ref{thm:existence ODE}, for a given $W\left(0\right)$,
we define a mapping $F_{W\left(0\right)}$ that maps from $W'=\left(w_{1}',w_{2}',w_{3}'\right)\in{\cal W}_{T}\left(a,b\right)$
to $F_{W\left(0\right)}\left(W'\right)=\bar{W}'=\left(\bar{w}_{1}',\bar{w}_{2}',\bar{w}_{3}'\right)$,
defined by $\bar{W}'\left(0\right)=W\left(0\right)$ and 
\begin{align*}
\frac{\partial}{\partial t}\bar{w}_{3}'\left(t,c_{2}\right) & =-\xi_{3}\left(t\right)\Delta_{3}\left(c_{2};W'\left(t\right)\right),\\
\frac{\partial}{\partial t}\bar{w}_{2}'\left(t,c_{1},c_{2}\right) & =-\xi_{2}\left(t\right)\Delta_{2}\left(c_{1},c_{2};W'\left(t\right)\right),\\
\frac{\partial}{\partial t}\bar{w}_{1}'\left(t,c_{1}\right) & =-\xi_{1}\left(t\right)\Delta_{1}\left(c_{1};W'\left(t\right)\right).
\end{align*}
Notice that the right-hand sides do not involve $\bar{W}'$. Note
that the MF ODEs' solution, initialized at $W\left(0\right)$, is
a fixed point of this mapping.

We establish the following estimates for this mapping.
\begin{lem}
\label{lem:a-priori-MF}Under Assumption \ref{assump:Regularity},
for $T\geq0$, any initialization $W\left(0\right)$ and any $W',W''\in{\cal W}_{T}\left(a,b\right)$,
\begin{align*}
{\rm ess\text{-}sup}\sup_{s\leq t}\left|\Delta_{3}\left(C_{2};W'\left(s\right)\right)-\Delta_{3}\left(C_{2};W''\left(s\right)\right)\right| & \leq K_{a,b}\left\Vert W'-W''\right\Vert _{t},\\
{\rm ess\text{-}sup}\sup_{s\leq t}\left|\Delta_{2}\left(C_{1},C_{2};W'\left(s\right)\right)-\Delta_{2}\left(C_{1},C_{2};W''\left(s\right)\right)\right| & \leq K_{a,b}\left\Vert W'-W''\right\Vert _{t},\\
{\rm ess\text{-}sup}\sup_{s\leq t}\left|\Delta_{1}\left(C_{1};W'\left(s\right)\right)-\Delta_{1}\left(C_{1};W''\left(s\right)\right)\right| & \leq K_{a,b}\left\Vert W'-W''\right\Vert _{t},
\end{align*}
and consequently, if in addition $W'\left(0\right)=W''\left(0\right)$
(not necessarily equal $W\left(0\right)$), then 
\begin{align*}
{\rm ess\text{-}sup}\sup_{t\leq T}\left|\bar{w}_{3}'\left(t,C_{2}\right)-\bar{w}_{3}''\left(t,C_{2}\right)\right| & \leq K_{a,b}\int_{0}^{T}\left\Vert W'-W''\right\Vert _{s}ds,\\
{\rm ess\text{-}sup}\sup_{t\leq T}\left|\bar{w}_{2}'\left(t,C_{1},C_{2}\right)-\bar{w}_{2}''\left(t,C_{1},C_{2}\right)\right| & \leq K_{a,b}\int_{0}^{T}\left\Vert W'-W''\right\Vert _{s}ds,\\
{\rm ess\text{-}sup}\sup_{t\leq T}\left|\bar{w}_{1}'\left(t,C_{1}\right)-\bar{w}_{1}''\left(t,C_{1}\right)\right| & \leq K_{a,b}\int_{0}^{T}\left\Vert W'-W''\right\Vert _{s}ds,
\end{align*}
in which $\bar{W}'=\left(\bar{w}_{1}',\bar{w}_{2}',\bar{w}_{3}'\right)=F_{W\left(0\right)}\left(W'\right)$,
$\bar{W}''=\left(\bar{w}_{1}'',\bar{w}_{2}'',\bar{w}_{3}''\right)=F_{W\left(0\right)}\left(W''\right)$
and $K_{a,b}\geq1$ is a generic constant that grows polynomially
with $a$ and $b$.
\end{lem}

\begin{proof}
From Assumption \ref{assump:Regularity} and the fact $W',W''\in{\cal W}_{T}\left(a,b\right)$,
we get:
\begin{align*}
\left|\Delta_{3}\left(C_{2};W'\left(s\right)\right)-\Delta_{3}\left(C_{2};W''\left(s\right)\right)\right| & \leq K\mathbb{E}_{Z}\left[\left|\partial_{2}{\cal L}\left(Y,\hat{y}\left(X;W'\left(s\right)\right)\right)-\partial_{2}{\cal L}\left(Y,\hat{y}\left(X;W''\left(s\right)\right)\right)\right|\right]\\
 & \qquad+K\mathbb{E}_{Z}\left[\left|H_{3}\left(X;W'\left(s\right)\right)-H_{3}\left(X;W''\left(s\right)\right)\right|\right]\\
 & \qquad+K\mathbb{E}_{Z}\left[\left|H_{2}\left(X,C_{2};W'\left(s\right)\right)-H_{2}\left(X,C_{2};W''\left(s\right)\right)\right|\right],\\
\left|\Delta_{2}\left(C_{1},C_{2};W'\left(s\right)\right)-\Delta_{2}\left(C_{1},C_{2};W''\left(s\right)\right)\right| & \leq K_{a,b}\left|w_{1}'\left(s,C_{1}\right)-w_{1}''\left(s,C_{1}\right)\right|\\
 & \qquad+K\left|\mathbb{E}_{Z}\left[\Delta_{2}^{H}\left(Z,C_{2};W'\left(s\right)\right)-\Delta_{2}^{H}\left(Z,C_{2};W''\left(s\right)\right)\right]\right|,\\
\left|\Delta_{1}\left(C_{1};W'\left(s\right)\right)-\Delta_{1}\left(C_{1};W''\left(s\right)\right)\right| & \leq K_{a,b}\mathbb{E}_{Z}\left[\left|\Delta_{2}^{H}\left(Z,C_{2};W'\left(s\right)\right)-\Delta_{2}^{H}\left(Z,C_{2};W''\left(s\right)\right)\right|\right]\\
 & \quad+K_{a,b}\left|w_{2}'\left(s,C_{1},C_{2}\right)-w_{2}''\left(s,C_{1},C_{2}\right)\right|\\
 & \quad+K_{a,b}\left|w_{1}'\left(s,C_{1}\right)-w_{1}''\left(s,C_{1}\right)\right|,
\end{align*}
from which the first three estimates then follow, in light of Lemma
\ref{lem:Lipschitz-MF}. The last three estimates then follow from
the fact that $\bar{W}'\left(0\right)=\bar{W}''\left(0\right)$ and
Assumption \ref{assump:Regularity}; for instance,
\begin{align*}
{\rm ess\text{-}sup}\sup_{t\leq T}\left|\bar{w}_{3}'\left(t,C_{2}\right)-\bar{w}_{3}''\left(t,C_{2}\right)\right| & \leq\int_{0}^{T}{\rm ess\text{-}sup}\left|\frac{\partial}{\partial t}\bar{w}_{3}'\left(s,C_{2}\right)-\frac{\partial}{\partial t}\bar{w}_{3}''\left(s,C_{2}\right)\right|ds\\
 & \leq K\int_{0}^{T}{\rm ess\text{-}sup}\left|\Delta_{3}\left(C_{2};W'\left(s\right)\right)-\Delta_{3}\left(C_{2};W''\left(s\right)\right)\right|ds.
\end{align*}
\end{proof}
We are now ready to prove Theorem \ref{thm:existence ODE}. 
\begin{proof}[Proof of Theorem \ref{thm:existence ODE}]
We will use a Picard-type iteration. To lighten notations:
\[
{\cal W}_{T}^{+}\equiv{\cal W}_{T}^{+}\left(\interleave W\interleave_{0}+K_{0,2}\left(T\right),\interleave W\interleave_{0}+K_{0,3}\left(T\right),W\left(0\right)\right),\qquad F\equiv F_{W\left(0\right)}.
\]
Since $\interleave W\interleave_{0}\leq K$ by assumption, we have
$\interleave W\interleave_{0}+K_{0,2}\left(T\right)+K_{0,3}\left(T\right)\leq K_{T}$.
Recall that ${\cal W}_{T}^{+}$ is complete. For an arbitrary $T>0$,
consider $W',W''\in{\cal W}_{T}^{+}$. Lemma \ref{lem:a-priori-MF}
yields:
\[
\left\Vert F\left(W'\right)-F\left(W''\right)\right\Vert _{T}\leq K_{T}\int_{0}^{T}\left\Vert W'-W''\right\Vert _{s}ds.
\]
Note that $F$ maps to ${\cal W}_{T}^{+}$ under Assumption \ref{assump:Regularity}
by the same argument as Lemma \ref{lem:a-priori-MF-norms}. Hence
we are allowed to iterating this inequality and get, for an arbitrary
$T>0$, 
\begin{align*}
\left\Vert F^{\left(k\right)}\left(W'\right)-F^{\left(k\right)}\left(W''\right)\right\Vert _{T} & \leq K_{T}\int_{0}^{T}\left\Vert F^{\left(k-1\right)}\left(W'\right)-F^{\left(k-1\right)}\left(W''\right)\right\Vert _{T_{2}}dT_{2}\\
 & \leq K_{T}^{2}\int_{0}^{T}\int_{0}^{T_{2}}\left\Vert F^{\left(k-2\right)}\left(W'\right)-F^{\left(k-2\right)}\left(W''\right)\right\Vert _{T_{3}}\mathbb{I}\left(T_{2}\leq T\right)dT_{3}dT_{2}\\
 & ...\\
 & \leq K_{T}^{k}\int_{0}^{T}\int_{0}^{T_{2}}...\int_{0}^{T_{k}}\left\Vert W'-W''\right\Vert _{T_{k+1}}\mathbb{I}\left(T_{k}\leq...\leq T_{2}\leq T\right)dT_{k+1}...dT_{2}\\
 & \leq\frac{1}{k!}K_{T}^{k}\left\Vert W'-W''\right\Vert _{T}.
\end{align*}
By substituting $W''=F\left(W'\right)$, we have:
\begin{align*}
\sum_{k=1}^{\infty}\left\Vert F^{\left(k+1\right)}\left(W'\right)-F^{\left(k\right)}\left(W'\right)\right\Vert _{T} & =\sum_{k=1}^{\infty}\left\Vert F^{\left(k\right)}\left(W''\right)-F^{\left(k\right)}\left(W'\right)\right\Vert _{T}\\
 & \leq\sum_{k=1}^{\infty}\frac{1}{k!}K_{T}^{k}\left\Vert W'-W''\right\Vert _{T}\\
 & <\infty.
\end{align*}
Hence as $k\to\infty$, $F^{\left(k\right)}\left(W'\right)$ converges
to a limit in ${\cal W}_{T}^{+}$, which is a fixed point of $F$.
The uniqueness of a fixed point follows from the above estimate, since
if $W'$ and $W''$ are fixed points then 
\[
\left\Vert W'-W''\right\Vert _{T}=\left\Vert F^{\left(k\right)}\left(W'\right)-F^{\left(k\right)}\left(W''\right)\right\Vert _{T}\leq\frac{1}{k!}K_{T}^{k}\left\Vert W'-W''\right\Vert _{T},
\]
while one can take $k$ arbitrarily large. This proves that the solution
exists and is unique on $t\in\left[0,T\right]$. Since $T$ is arbitrary,
we have existence and uniqueness of the solution on the time interval
$[0,\infty)$.
\end{proof}

\section{Connection between the neural net and its MF limit: proofs for Section
\ref{sec:connection}\label{sec:Connection-proof}}

\subsection{Proof of Theorem \ref{thm:gradient descent coupling}}

We construct an auxiliary trajectory, which we call the \textit{particle
ODEs}: 
\begin{align*}
\frac{\partial}{\partial t}\tilde{w}_{3}\left(t,j_{2}\right) & =-\xi_{3}\left(t\right)\mathbb{E}_{Z}\left[\partial_{2}{\cal L}\left(Y,\hat{\mathbf{y}}\left(X;\tilde{W}\left(t\right)\right)\right)\varphi_{3}'\left(\mathbf{H}_{3}\left(X;\tilde{W}\left(t\right)\right)\right)\varphi_{2}\left(\mathbf{H}_{2}\left(X,j_{2};\tilde{W}\left(t\right)\right)\right)\right],\\
\frac{\partial}{\partial t}\tilde{w}_{2}\left(t,j_{1},j_{2}\right) & =-\xi_{2}\left(t\right)\mathbb{E}_{Z}\left[\Delta_{2}^{\mathbf{H}}\left(Z,j_{2};\tilde{W}\left(t\right)\right)\varphi_{1}\left(\left\langle \tilde{w}_{1}\left(t,j_{1}\right),X\right\rangle \right)\right],\\
\frac{\partial}{\partial t}\tilde{w}_{1}\left(t,j_{1}\right) & =-\xi_{1}\left(t\right)\mathbb{E}_{Z}\left[\frac{1}{n_{2}}\sum_{j_{2}=1}^{n_{2}}\Delta_{2}^{\mathbf{H}}\left(Z,j_{2};\tilde{W}\left(t\right)\right)\tilde{w}_{2}\left(t,j_{1},j_{2}\right)\varphi_{1}'\left(\left\langle \tilde{w}_{1}\left(t,j_{1}\right),X\right\rangle \right)X\right],
\end{align*}
in which $j_{1}=1,...,n_{1}$, $j_{2}=1,...,n_{2}$, $\tilde{W}\left(t\right)=\left(\tilde{w}_{1}\left(t,\cdot\right),\tilde{w}_{2}\left(t,\cdot,\cdot\right),\tilde{w}_{3}\left(t,\cdot\right)\right)$,
and $t\in\mathbb{R}_{\geq0}$. We specify the initialization $\tilde{W}\left(0\right)$:
$\tilde{w}_{1}\left(0,j_{1}\right)=w_{1}^{0}\left(C_{1}\left(j_{1}\right)\right)$,
$\tilde{w}_{2}\left(0,j_{1},j_{2}\right)=w_{2}^{0}\left(C_{1}\left(j_{1}\right),C_{2}\left(j_{2}\right)\right)$
and $\tilde{w}_{3}\left(0,j_{3}\right)=w_{3}^{0}\left(C_{2}\left(j_{2}\right)\right)$.
That is, it shares the same initialization with the neural network
one $\mathbf{W}\left(0\right)$, and hence is coupled with the neural
network and the MF ODEs. Roughly speaking, the particle ODEs are continuous-time
trajectories of finitely many neurons, averaged over the data distribution.
We note that $\tilde{W}\left(t\right)$ is random for all $t\in\mathbb{R}_{\geq0}$
due to the randomness of $\left\{ C_{i}\left(j_{i}\right)\right\} _{i=1,2}$.

The existence and uniqueness of the solution to the particle ODEs
follows from the same proof as in Theorem \ref{thm:existence ODE},
which we shall not repeat here. We equip $\tilde{W}\left(t\right)$
with the norm
\[
\interleave\tilde{W}\interleave_{T}=\max\bigg\{\max_{j_{1}\leq n_{1},\;j_{2}\leq n_{2}}\sup_{t\leq T}\left|\tilde{w}_{2}\left(t,j_{1},j_{2}\right)\right|,\;\max_{j_{2}\leq n_{2}}\sup_{t\leq T}\left|\tilde{w}_{3}\left(t,j_{2}\right)\right|\bigg\}.
\]
One can also define the measures $\mathscr{D}_{T}\left(W,\tilde{W}\right)$
and $\mathscr{D}_{T}\left(\tilde{W},\mathbf{W}\right)$ similar to
Eq. (\ref{eq:dist_W}): 
\begin{align*}
\mathscr{D}_{T}\left(W,\tilde{W}\right)=\sup\big\{ & \left|w_{1}\left(t,C_{1}\left(j_{1}\right)\right)-\tilde{w}_{1}\left(t,C_{1}\left(j_{1}\right)\right)\right|,\;\left|w_{2}\left(t,C_{1}\left(j_{1}\right),C_{2}\left(j_{2}\right)\right)-\tilde{w}_{2}\left(t,C_{1}\left(j_{1}\right),C_{2}\left(j_{2}\right)\right)\right|,\\
 & \left|w_{3}\left(t,C_{2}\left(j_{2}\right)\right)-\tilde{w}_{3}\left(t,C_{2}\left(j_{2}\right)\right)\right|:\;t\leq T,\;j_{1}\leq n_{1},\;j_{2}\leq n_{2}\big\},\\
\mathscr{D}_{T}\left(\tilde{W},\mathbf{W}\right)=\sup\big\{ & \left|\mathbf{w}_{1}\left(\left\lfloor t/\epsilon\right\rfloor ,j_{1}\right)-\tilde{w}_{1}\left(t,C_{1}\left(j_{1}\right)\right)\right|,\;\left|\mathbf{w}_{2}\left(\left\lfloor t/\epsilon\right\rfloor ,j_{1},j_{2}\right)-\tilde{w}_{2}\left(t,C_{1}\left(j_{1}\right),C_{2}\left(j_{2}\right)\right)\right|,\\
 & \left|\mathbf{w}_{3}\left(\left\lfloor t/\epsilon\right\rfloor ,j_{2}\right)-\tilde{w}_{3}\left(t,C_{2}\left(j_{2}\right)\right)\right|:\;t\leq T,\;j_{1}\leq n_{1},\;j_{2}\leq n_{2}\big\}.
\end{align*}
We have the following results: 
\begin{thm}
\label{thm:particle coupling}Under the same setting as Theorem \ref{thm:gradient descent coupling},
for any $\delta>0$, with probability at least $1-\delta$, 
\[
\mathscr{D}_{T}\left(W,\tilde{W}\right)\leq\frac{1}{\sqrt{n_{\min}}}\log^{1/2}\left(\frac{3\left(T+1\right)n_{\max}^{2}}{\delta}+e\right)e^{K_{T}},
\]
in which $n_{\min}=\min\left\{ n_{1},n_{2}\right\} $, $n_{\max}=\max\left\{ n_{1},n_{2}\right\} $,
and $K_{T}=K\left(1+T^{K}\right)$.
\end{thm}

\begin{thm}
\label{thm:gradient descent}Under the same setting as Theorem \ref{thm:gradient descent coupling},
for any $\delta>0$ and $\epsilon\leq1$, with probability at least
$1-\delta$,
\[
\mathscr{D}_{T}\left(\tilde{W},\mathbf{W}\right)\leq\sqrt{\epsilon\log\left(\frac{2n_{1}n_{2}}{\delta}+e\right)}e^{K_{T}},
\]
in which $K_{T}=K\left(1+T^{K}\right)$.
\end{thm}

\begin{proof}[Proof of Theorem \ref{thm:gradient descent coupling}]
Using the fact 
\[
\mathscr{D}_{T}\left(W,\mathbf{W}\right)\leq\mathscr{D}_{T}\left(W,\tilde{W}\right)+\mathscr{D}_{T}\left(\tilde{W},\mathbf{W}\right),
\]
the thesis is immediate from Theorems \ref{thm:particle coupling}
and \ref{thm:gradient descent}.
\end{proof}

\subsection{Proof of Theorems \ref{thm:particle coupling} and \ref{thm:gradient descent}}
\begin{proof}[Proof of Theorem \ref{thm:particle coupling}]
In the following, let $K_{t}$ denote an generic positive constant
that may change from line to line and takes the form 
\[
K_{t}=K\left(1+t^{K}\right),
\]
such that $K_{t}\geq1$ and $K_{t}\leq K_{T}$ for all $t\leq T$.
We first note that at initialization, $\mathscr{D}_{0}\left(W,\tilde{W}\right)=0$.
Since $\interleave W\interleave_{0}\leq K$, $\interleave W\interleave_{T}\leq K_{T}$
by Lemma \ref{lem:a-priori-MF-norms}. Furthermore it is easy to see
that $\interleave\tilde{W}\interleave_{0}\leq\interleave W\interleave_{0}\leq K$
almost surely. By the same argument as in Lemma \ref{lem:a-priori-MF-norms},
$\interleave\tilde{W}\interleave_{T}\leq K_{T}$ almost surely.

We shall use all above bounds repeatedly in the proof. We decompose
the proof into several steps.

\paragraph*{Step 1 - Main proof.}

Let us define, for brevity
\begin{align*}
q_{3}\left(t,x\right) & =\mathbf{H}_{3}\left(x;\tilde{W}\left(t\right)\right)-H_{3}\left(x;W\left(t\right)\right),\\
q_{2}\left(t,x,j_{2},c_{2}\right) & =\mathbf{H}_{2}\left(x,j_{2};\tilde{W}\left(t\right)\right)-H_{2}\left(x,c_{2};W\left(t\right)\right),\\
q_{\Delta}\left(t,z,j_{1},j_{2},c_{1},c_{2}\right) & =\Delta_{2}^{\mathbf{H}}\left(Z,j_{2};\tilde{W}\left(t\right)\right)\tilde{w}_{2}\left(t,j_{1},j_{2}\right)-\Delta_{2}^{H}\left(z,c_{2};W\left(t\right)\right)w_{2}\left(t,c_{1},c_{2}\right).
\end{align*}
Consider $t\geq0$. We first bound the difference in the updates between
$W$ and $\tilde{W}$. Let us start with $w_{3}$ and $\tilde{w}_{3}$.
By Assumption \ref{assump:Regularity}, we have:
\[
\left|\frac{\partial}{\partial t}\tilde{w}_{3}\left(t,j_{2}\right)-\frac{\partial}{\partial t}w_{3}\left(t,C_{2}\left(j_{2}\right)\right)\right|\leq K\mathbb{E}_{Z}\left[\left|q_{3}\left(t,X\right)\right|+\left|q_{2}\left(t,X,j_{2},C_{2}\left(j_{2}\right)\right)\right|\right].
\]
Similarly, for $w_{2}$ and $\tilde{w}_{2}$,
\begin{align*}
 & \left|\frac{\partial}{\partial t}\tilde{w}_{2}\left(t,j_{1},j_{2}\right)-\frac{\partial}{\partial t}w_{2}\left(t,C_{1}\left(j_{1}\right),C_{2}\left(j_{2}\right)\right)\right|\\
 & \leq K\mathbb{E}_{Z}\left[\left|\Delta_{2}^{\mathbf{H}}\left(Z,j_{2};\tilde{W}\left(t\right)\right)-\Delta_{2}^{H}\left(Z,C_{2}\left(j_{2}\right);W\left(t\right)\right)\right|\right]\\
 & \quad+K\left|w_{3}\left(t,C_{2}\left(j_{2}\right)\right)\right|\left|\tilde{w}_{1}\left(t,j_{1}\right)-w_{1}\left(t,C_{1}\left(j_{1}\right)\right)\right|\\
 & \leq K_{t}\mathbb{E}_{Z}\left[\left|q_{3}\left(t,X\right)\right|+\left|q_{2}\left(t,X,j_{2},C_{2}\left(j_{2}\right)\right)\right|\right]\\
 & \quad+K_{t}\left(\left|\tilde{w}_{1}\left(t,j_{1}\right)-w_{1}\left(t,C_{1}\left(j_{1}\right)\right)\right|+\left|\tilde{w}_{3}\left(t,j_{2}\right)-w_{3}\left(t,C_{2}\left(j_{2}\right)\right)\right|\right)\\
 & \leq K_{t}\mathbb{E}_{Z}\left[\left|q_{3}\left(t,X\right)\right|+\left|q_{2}\left(t,X,j_{2},C_{2}\left(j_{2}\right)\right)\right|\right]+K_{t}\mathscr{D}_{t}\left(W,\tilde{W}\right),
\end{align*}
and for $w_{1}$ and $\tilde{w}_{1}$, by Lemma \ref{lem:Lipschitz-MF},
\begin{align*}
 & \left|\frac{\partial}{\partial t}\tilde{w}_{1}\left(t,j_{1}\right)-\frac{\partial}{\partial t}w_{1}\left(t,C_{1}\left(j_{1}\right)\right)\right|\\
 & \leq K\mathbb{E}_{Z}\left[\left|\frac{1}{n_{2}}\sum_{j_{2}=1}^{n_{2}}\mathbb{E}_{C_{2}}\left[q_{\Delta}\left(t,Z,j_{1},j_{2},C_{1}\left(j_{1}\right),C_{2}\right)\right]\right|\right]\\
 & \quad+\mathbb{E}_{C_{2}}\left[\left|\Delta_{2}^{H}\left(Z,C_{2};W\left(t\right)\right)\right|\left|w_{2}\left(t,C_{1}\left(j_{1}\right),C_{2}\right)\right|\right]\left|\tilde{w}_{1}\left(t,j_{1}\right)-w_{1}\left(t,C_{1}\left(j_{1}\right)\right)\right|\\
 & \leq K\mathbb{E}_{Z}\left[\left|\frac{1}{n_{2}}\sum_{j_{2}=1}^{n_{2}}\mathbb{E}_{C_{2}}\left[q_{\Delta}\left(t,Z,j_{1},j_{2},C_{1}\left(j_{1}\right),C_{2}\right)\right]\right|\right]\\
 & \quad+K_{t}\mathscr{D}_{t}\left(W,\tilde{W}\right).
\end{align*}
To further the bounding, we now make the following two claims:
\begin{itemize}
\item \textbf{Claim 1:} For any $\xi>0$,
\begin{align*}
\max_{j_{2}\leq n_{2}}\left|\frac{\partial}{\partial t}w_{3}\left(t+\xi,C_{2}\left(j_{2}\right)\right)-\frac{\partial}{\partial t}w_{3}\left(t,C_{2}\left(j_{2}\right)\right)\right| & \leq K_{t+\xi}\xi,\\
\max_{j_{1}\leq n_{1},\;j_{2}\leq n_{2}}\left|\frac{\partial}{\partial t}w_{2}\left(t+\xi,C_{1}\left(j_{1}\right),C_{2}\left(j_{2}\right)\right)-\frac{\partial}{\partial t}w_{2}\left(t,C_{1}\left(j_{1}\right),C_{2}\left(j_{2}\right)\right)\right| & \leq K_{t+\xi}\xi,\\
\max_{j_{1}\leq n_{1}}\left|\frac{\partial}{\partial t}w_{1}\left(t+\xi,C_{1}\left(j_{1}\right)\right)-\frac{\partial}{\partial t}w_{1}\left(t,C_{1}\left(j_{1}\right)\right)\right| & \leq K_{t+\xi}\xi,
\end{align*}
and similarly,
\begin{align*}
\max_{j_{2}\leq n_{2}}\left|\frac{\partial}{\partial t}\tilde{w}_{3}\left(t+\xi,j_{2}\right)-\frac{\partial}{\partial t}\tilde{w}_{3}\left(t,j_{2}\right)\right| & \leq K_{t+\xi}\xi,\\
\max_{j_{1}\leq n_{1},\;j_{2}\leq n_{2}}\left|\frac{\partial}{\partial t}\tilde{w}_{2}\left(t+\xi,j_{1},j_{2}\right)-\frac{\partial}{\partial t}\tilde{w}_{2}\left(t,j_{1},j_{2}\right)\right| & \leq K_{t+\xi}\xi,\\
\max_{j_{1}\leq n_{1}}\left|\frac{\partial}{\partial t}\tilde{w}_{1}\left(t+\xi,j_{1}\right)-\frac{\partial}{\partial t}\tilde{w}_{1}\left(t,j_{1}\right)\right| & \leq K_{t+\xi}\xi.
\end{align*}
\item \textbf{Claim 2:} For any $\gamma_{1},\gamma_{2},\gamma_{3}>0$ and
$t\geq0$,
\begin{align*}
 & \max\Bigg\{\max_{j_{2}\leq n_{2}}\mathbb{E}_{Z}\left[\left|q_{2}\left(t,X,j_{2},C_{2}\left(j_{2}\right)\right)\right|\right],\quad\mathbb{E}_{Z}\left[\left|q_{3}\left(t,X\right)\right|\right],\\
 & \qquad\max_{j_{1}\leq n_{1}}\mathbb{E}_{Z}\left[\left|\frac{1}{n_{2}}\sum_{j_{2}=1}^{n_{2}}\mathbb{E}_{C_{2}}\left[q_{\Delta}\left(t,Z,j_{1},j_{2},C_{1}\left(j_{1}\right),C_{2}\right)\right]\right|\right]\Bigg\}\\
 & \geq K_{t}\left(\mathscr{D}_{t}\left(W,\tilde{W}\right)+\gamma_{1}+\gamma_{2}+\gamma_{3}\right),
\end{align*}
with probability at most
\[
\frac{n_{1}}{\gamma_{1}}\exp\left(-\frac{n_{2}\gamma_{1}^{2}}{K_{t}}\right)+\frac{n_{2}}{\gamma_{2}}\exp\left(-\frac{n_{1}\gamma_{2}^{2}}{K_{t}}\right)+\frac{1}{\gamma_{3}}\exp\left(-\frac{n_{2}\gamma_{3}^{2}}{K_{t}}\right).
\]
\end{itemize}
Combining these claims with the previous bounds, taking a union bound
over $t\in\left\{ 0,\xi,2\xi,...,\left\lfloor T/\xi\right\rfloor \xi\right\} $
for some $\xi\in\left(0,1\right)$, we obtain that
\begin{align*}
 & \max\bigg\{\max_{j_{2}\leq n_{2}}\left|\frac{\partial}{\partial t}\tilde{w}_{3}\left(t,j_{2}\right)-\frac{\partial}{\partial t}w_{3}\left(t,C_{2}\left(j_{2}\right)\right)\right|,\\
 & \qquad\max_{j_{1}\leq n_{1},\;j_{2}\leq n_{2}}\left|\frac{\partial}{\partial t}\tilde{w}_{2}\left(t,j_{1},j_{2}\right)-\frac{\partial}{\partial t}w_{2}\left(t,C_{1}\left(j_{1}\right),C_{2}\left(j_{2}\right)\right)\right|,\\
 & \qquad\max_{j_{1}\leq n_{1}}\left|\frac{\partial}{\partial t}\tilde{w}_{1}\left(t,j_{1}\right)-\frac{\partial}{\partial t}w_{1}\left(t,C_{1}\left(j_{1}\right)\right)\right|\bigg\}\\
 & \leq K_{T}\left(\mathscr{D}_{t}\left(W,\tilde{W}\right)+\gamma_{1}+\gamma_{2}+\gamma_{3}+\xi\right),\qquad\forall t\in\left[0,T\right],
\end{align*}
with probability at least
\[
1-\frac{T+1}{\xi}\left[\frac{n_{1}}{\gamma_{1}}\exp\left(-\frac{n_{2}\gamma_{1}^{2}}{K_{T}}\right)+\frac{n_{2}}{\gamma_{2}}\exp\left(-\frac{n_{1}\gamma_{2}^{2}}{K_{T}}\right)+\frac{1}{\gamma_{3}}\exp\left(-\frac{n_{2}\gamma_{3}^{2}}{K_{T}}\right)\right].
\]
The above event in turn implies
\[
\mathscr{D}_{t}\left(W,\tilde{W}\right)\leq K_{T}\int_{0}^{t}\left(\mathscr{D}_{s}\left(W,\tilde{W}\right)+\gamma_{1}+\gamma_{2}+\gamma_{3}+\xi\right)ds,
\]
and hence by Gronwall's lemma and the fact $\mathscr{D}_{0}\left(W,\tilde{W}\right)=0$,
we get
\[
\mathscr{D}_{T}\left(W,\tilde{W}\right)\leq\left(\gamma_{1}+\gamma_{2}+\gamma_{3}+\xi\right)e^{K_{T}}.
\]
The theorem then follows from the choice
\[
\xi=\frac{1}{\sqrt{n_{\max}}},\quad\gamma_{2}=\frac{K_{T}}{\sqrt{n_{1}}}\log^{1/2}\left(\frac{3\left(T+1\right)n_{\max}^{2}}{\delta}+e\right),\quad\gamma_{1}=\gamma_{3}=\frac{K_{T}}{\sqrt{n_{2}}}\log^{1/2}\left(\frac{3\left(T+1\right)n_{\max}^{2}}{\delta}+e\right).
\]
We are left with proving the claims.

\paragraph*{Step 2 - Proof of Claim 1.}

We have from Assumption \ref{assump:Regularity},
\begin{align*}
{\rm ess\text{-}sup}\left|w_{3}\left(t+\xi,C_{2}\right)-w_{3}\left(t,C_{2}\right)\right| & \leq K\int_{t}^{t+\xi}{\rm ess\text{-}sup}\left|\frac{\partial}{\partial t}w_{3}\left(s,C_{2}\right)\right|ds\\
 & \leq K\xi,\\
{\rm ess\text{-}sup}\left|w_{2}\left(t+\xi,C_{1},C_{2}\right)-w_{2}\left(t,C_{1},C_{2}\right)\right| & \leq K\int_{t}^{t+\xi}{\rm ess\text{-}sup}\left|\frac{\partial}{\partial t}w_{2}\left(s,C_{1},C_{2}\right)\right|ds\\
 & \leq K\int_{t}^{t+\xi}{\rm ess\text{-}sup}\left|w_{3}\left(s,C_{2}\right)\right|ds\\
 & \leq K_{t+\xi}\xi,\\
{\rm ess\text{-}sup}\left|w_{1}\left(t+\xi,C_{1}\right)-w_{1}\left(t,C_{1}\right)\right| & \leq K\int_{t}^{t+\xi}{\rm ess\text{-}sup}\left|\frac{\partial}{\partial t}w_{1}\left(s,C_{1}\right)\right|ds\\
 & \leq K\int_{t}^{t+\xi}{\rm ess\text{-}sup}\left|w_{3}\left(s,C_{2}\right)w_{2}\left(s,C_{1},C_{2}\right)\right|ds\\
 & \leq K_{t+\xi}\xi.
\end{align*}
By Lemma \ref{lem:Lipschitz-MF}, we then obtain that
\begin{align*}
{\rm ess\text{-}sup}\mathbb{E}_{Z}\left[\left|H_{2}\left(X,C_{2};W\left(t+\xi\right)\right)-H_{2}\left(X,C_{2};W\left(t\right)\right)\right|\right] & \leq K_{t+\xi}\xi,\\
\mathbb{E}_{Z}\left[\left|H_{3}\left(X;W\left(t+\xi\right)\right)-H_{3}\left(X;W\left(t\right)\right)\right|\right] & \leq K_{t+\xi}\xi,\\
{\rm ess\text{-}sup}\mathbb{E}_{Z}\left[\left|\Delta_{2}^{H}\left(Z,C_{2};W\left(t+\xi\right)\right)-\Delta_{2}^{H}\left(Z,C_{2};W\left(t\right)\right)\right|\right] & \leq K_{t+\xi}\xi.
\end{align*}
Using these estimates, we thus have, by Assumption \ref{assump:Regularity},
\begin{align*}
 & \max_{j_{2}\leq n_{2}}\left|\frac{\partial}{\partial t}w_{3}\left(t+\xi,C_{2}\left(j_{2}\right)\right)-\frac{\partial}{\partial t}w_{3}\left(t,C_{2}\left(j_{2}\right)\right)\right|\\
 & \quad\leq K_{t+\xi}\xi+K\mathbb{E}_{Z}\left[\left|H_{3}\left(X;W\left(t+\xi\right)\right)-H_{3}\left(X;W\left(t\right)\right)\right|\right]\\
 & \quad\quad+K{\rm ess\text{-}sup}\mathbb{E}_{Z}\left[\left|H_{2}\left(X,C_{2};W\left(t+\xi\right)\right)-H_{2}\left(X,C_{2};W\left(t\right)\right)\right|\right]\\
 & \quad\leq K_{t+\xi}\xi,\\
 & \max_{j_{1}\leq n_{1},\;j_{2}\leq n_{2}}\left|\frac{\partial}{\partial t}w_{2}\left(t+\xi,C_{1}\left(j_{1}\right),C_{2}\left(j_{2}\right)\right)-\frac{\partial}{\partial t}w_{2}\left(t,C_{1}\left(j_{1}\right),C_{2}\left(j_{2}\right)\right)\right|\\
 & \quad\leq K_{t+\xi}\xi+K{\rm ess\text{-}sup}\mathbb{E}_{Z}\left[\left|\Delta_{2}^{H}\left(Z,C_{2};W\left(t+\xi\right)\right)-\Delta_{2}^{H}\left(Z,C_{2};W\left(t\right)\right)\right|\right]\\
 & \quad\quad+K{\rm ess\text{-}sup}\left|w_{3}\left(t,C_{2}\right)\right|\left|w_{1}\left(t+\xi,C_{1}\right)-w_{1}\left(t,C_{1}\right)\right|\\
 & \quad\leq K_{t+\xi}\xi,\\
 & \max_{j_{1}\leq n_{1}}\left|\frac{\partial}{\partial t}w_{1}\left(t+\xi,C_{1}\left(j_{1}\right)\right)-\frac{\partial}{\partial t}w_{1}\left(t,C_{1}\left(j_{1}\right)\right)\right|\\
 & \quad\leq K_{t+\xi}\xi+K{\rm ess\text{-}sup}\mathbb{E}_{Z}\left[\mathbb{E}_{C_{2}}\left[\left|\Delta_{2}^{H}\left(Z,C_{2};W\left(t+\xi\right)\right)-\Delta_{2}^{H}\left(Z,C_{2};W\left(t\right)\right)\right|\left|w_{2}\left(t,C_{1},C_{2}\right)\right|\right]\right]\\
 & \quad\quad+K{\rm ess\text{-}sup}\mathbb{E}_{C_{2}}\left[\left|w_{3}\left(t,C_{2}\right)\right|\left|w_{2}\left(t+\xi,C_{1},C_{2}\right)-w_{2}\left(t,C_{1},C_{2}\right)\right|\right]\\
 & \quad\quad+K{\rm ess\text{-}sup}\mathbb{E}_{C_{2}}\left[\left|w_{3}\left(t,C_{2}\right)w_{2}\left(t,C_{1},C_{2}\right)\right|\right]\left|w_{1}\left(t+\xi,C_{1}\right)-w_{1}\left(t,C_{1}\right)\right|\\
 & \quad\leq K_{t+\xi}\xi.
\end{align*}
The proof of the rest of the claim is similar.

\paragraph*{Step 3 - Proof of Claim 2.}

We recall the definitions of $q_{\Delta}$, $q_{2}$ and $q_{3}$.
Let us decompose them as follows. We start with $q_{2}$:
\begin{align*}
 & \left|q_{2}\left(t,x,j_{2},C_{2}\left(j_{2}\right)\right)\right|\\
 & =\left|\frac{1}{n_{1}}\sum_{j_{1}=1}^{n_{1}}\tilde{w}_{2}\left(t,j_{1},j_{2}\right)\varphi_{1}\left(\left\langle \tilde{w}_{1}\left(t,j_{1}\right),x\right\rangle \right)-\mathbb{E}_{C_{1}}\left[w_{2}\left(t,C_{1},C_{2}\left(j_{2}\right)\right)\varphi_{1}\left(\left\langle w_{1}\left(t,C_{1}\right),x\right\rangle \right)\right]\right|\\
 & \leq\max_{j_{1}\leq n_{1}}\left|\tilde{w}_{2}\left(t,j_{1},j_{2}\right)\varphi_{1}\left(\left\langle \tilde{w}_{1}\left(t,j_{1}\right),x\right\rangle \right)-w_{2}\left(t,C_{1}\left(j_{1}\right),C_{2}\left(j_{2}\right)\right)\varphi_{1}\left(\left\langle w_{1}\left(t,C_{1}\left(j_{1}\right)\right),x\right\rangle \right)\right|\\
 & \quad+\left|\frac{1}{n_{1}}\sum_{j_{1}=1}^{n_{1}}w_{2}\left(t,C_{1}\left(j_{1}\right),C_{2}\left(j_{2}\right)\right)\varphi_{1}\left(\left\langle w_{1}\left(t,C_{1}\left(j_{1}\right)\right),x\right\rangle \right)-\mathbb{E}_{C_{1}}\left[w_{2}\left(t,C_{1},C_{2}\left(j_{2}\right)\right)\varphi_{1}\left(\left\langle w_{1}\left(t,C_{1}\right),x\right\rangle \right)\right]\right|\\
 & \equiv Q_{2,1}\left(x,j_{2}\right)+Q_{2,2}\left(x,j_{2}\right).
\end{align*}
Similarly, we have for $q_{3}$:
\begin{align*}
 & \left|q_{3}\left(t,x\right)\right|\\
 & =\left|\frac{1}{n_{2}}\sum_{j_{2}=1}^{n_{2}}\tilde{w}_{3}\left(t,j_{2}\right)\varphi_{2}\left(\mathbf{H}_{2}\left(x,j_{2};\tilde{W}\left(t\right)\right)\right)-\mathbb{E}_{C_{2}}\left[w_{3}\left(t,C_{2}\right)\varphi_{2}\left(H_{2}\left(x,C_{2};W\left(t\right)\right)\right)\right]\right|\\
 & \leq\max_{j_{2}\leq n_{2}}\left|\tilde{w}_{3}\left(t,j_{2}\right)\varphi_{2}\left(\mathbf{H}_{2}\left(x,j_{2};\tilde{W}\left(t\right)\right)\right)-w_{3}\left(t,C_{2}\left(j_{2}\right)\right)\varphi_{2}\left(H_{2}\left(x,C_{2}\left(j_{2}\right);W\left(t\right)\right)\right)\right|\\
 & \quad+\left|\frac{1}{n_{2}}\sum_{j_{2}=1}^{n_{2}}w_{3}\left(t,C_{2}\left(j_{2}\right)\right)\varphi_{2}\left(H_{2}\left(x,C_{2}\left(j_{2}\right);W\left(t\right)\right)\right)-\mathbb{E}_{C_{2}}\left[w_{3}\left(t,C_{2}\right)\varphi_{2}\left(H_{2}\left(x,C_{2};W\left(t\right)\right)\right)\right]\right|\\
 & \equiv Q_{3,1}\left(x\right)+Q_{3,2}\left(x\right).
\end{align*}
Finally we have for $q_{\Delta}$:
\begin{align*}
 & \left|\frac{1}{n_{2}}\sum_{j_{2}=1}^{n_{2}}\mathbb{E}_{C_{2}}\left[q_{\Delta}\left(t,z,j_{1},j_{2},C_{1}\left(j_{1}\right),C_{2}\right)\right]\right|\\
 & \leq\max_{j_{2}\leq n_{2}}\left|\Delta_{2}^{\mathbf{H}}\left(z,j_{2};\tilde{W}\left(t\right)\right)\tilde{w}_{2}\left(t,j_{1},j_{2}\right)-\Delta_{2}^{H}\left(z,C_{2}\left(j_{2}\right);W\left(t\right)\right)w_{2}\left(t,C_{1}\left(j_{1}\right),C_{2}\left(j_{2}\right)\right)\right|\\
 & \quad+\left|\frac{1}{n_{2}}\sum_{j_{2}=1}^{n_{2}}\Delta_{2}^{H}\left(z,C_{2}\left(j_{2}\right);W\left(t\right)\right)w_{2}\left(t,C_{1}\left(j_{1}\right),C_{2}\left(j_{2}\right)\right)-\mathbb{E}_{C_{2}}\left[\Delta_{2}^{H}\left(z,C_{2};W\left(t\right)\right)w_{2}\left(t,C_{1}\left(j_{1}\right),C_{2}\right)\right]\right|\\
 & \equiv Q_{1,1}\left(z,j_{1}\right)+Q_{1,2}\left(z,j_{1}\right).
\end{align*}
Now let us analyze each of the terms.
\begin{itemize}
\item We start with $Q_{2,1}$. We have from Assumption \ref{assump:Regularity},
\begin{align*}
 & \max_{j_{2}\leq n_{2}}\mathbb{E}_{Z}\left[Q_{2,1}\left(X,j_{2}\right)\right]\\
 & \leq K\max_{j_{1}\leq n_{1},\;j_{2}\leq n_{2}}\left|\tilde{w}_{2}\left(t,j_{1},j_{2}\right)-w_{2}\left(t,C_{1}\left(j_{1}\right),C_{2}\left(j_{2}\right)\right)\right|\\
 & \quad+K\max_{j_{1}\leq n_{1},\;j_{2}\leq n_{2}}\left|w_{2}\left(t,C_{1}\left(j_{1}\right),C_{2}\left(j_{2}\right)\right)\right|\left|\tilde{w}_{1}\left(t,j_{1}\right)-w_{1}\left(t,C_{1}\left(j_{1}\right)\right)\right|\\
 & \leq K_{t}\mathscr{D}_{t}\left(W,\tilde{W}\right).
\end{align*}
\item To bound $Q_{2,2}$, let us write:
\[
Z_{2}\left(x,c_{1},c_{2}\right)=w_{2}\left(t,c_{1},c_{2}\right)\varphi_{1}\left(\left\langle w_{1}\left(t,c_{1}\right),x\right\rangle \right).
\]
Recall that $C_{1}\left(j_{1}\right)$ and $C_{2}\left(j_{2}\right)$
are independent. We thus have: 
\[
\mathbb{E}\left[Z_{2}\left(X,C_{1}\left(j_{1}\right),C_{2}\left(j_{2}\right)\right)\middle|X,C_{2}\left(j_{2}\right)\right]=\mathbb{E}_{C_{1}}\left[Z_{2}\left(X,C_{1},C_{2}\left(j_{2}\right)\right)\right].
\]
Furthermore $\left\{ Z_{2}\left(C_{1}\left(j_{1}\right),C_{2}\left(j_{2}\right)\right)\right\} _{j_{1}\in\left[n_{1}\right]}$
are independent, conditional on $C_{2}\left(j_{2}\right)$. We also
have, almost surely
\[
\left|Z_{2}\left(X,C_{1}\left(j_{1}\right),C_{2}\left(j_{2}\right)\right)\right|\leq K_{t},
\]
by Assumption \ref{assump:Regularity}. Then by Lemma \ref{lem:square hoeffding},
\[
\mathbb{P}\left(\mathbb{E}_{Z}\left[Q_{2,2}\left(X,j_{2}\right)\right]\geq K_{t}\gamma_{2}\right)\leq\left(1/\gamma_{2}\right)\exp\left(-n_{1}\gamma_{2}^{2}/K_{t}\right).
\]
\item To bound $Q_{3,1}$, we have from Assumption \ref{assump:Regularity},
\begin{align*}
\mathbb{E}_{Z}\left[Q_{3,1}\left(X\right)\right] & \leq\max_{j_{2}\leq n_{2}}\left(K\left|\tilde{w}_{3}\left(t,j_{2}\right)-w_{3}\left(t,C_{2}\left(j_{2}\right)\right)\right|+K_{t}\mathbb{E}_{Z}\left[\left|q_{2}\left(t,X,j_{2},C_{2}\left(j_{2}\right)\right)\right|\right]\right)\\
 & \leq K\mathscr{D}_{t}\left(W,\tilde{W}\right)+K_{t}\max_{j_{2}\leq n_{2}}\mathbb{E}_{Z}\left[\left|q_{2}\left(t,X,j_{2},C_{2}\left(j_{2}\right)\right)\right|\right].
\end{align*}
\item To bound $Q_{3,2}$, noticing that almost surely
\[
\left|w_{3}\left(t,C_{2}\left(j_{2}\right)\right)\varphi_{2}\left(H_{2}\left(x,C_{2}\left(j_{2}\right);W\left(t\right)\right)\right)\right|\leq K_{t}
\]
by Assumption \ref{assump:Regularity}, we obtain
\[
\mathbb{P}\left(\mathbb{E}_{Z}\left[Q_{3,2}\left(X\right)\right]\geq K_{t}\gamma_{3}\right)\leq\left(1/\gamma_{3}\right)\exp\left(-n_{2}\gamma_{3}^{2}/K_{t}\right),
\]
similar to the treatment of $Q_{2,2}$.
\item To bound $Q_{1,1}$, using Assumption \ref{assump:Regularity},
\begin{align*}
\mathbb{E}_{Z}\left[Q_{1,1}\left(Z,j_{1}\right)\right] & \leq K\max_{j_{2}\leq n_{2}}\left|w_{2}\left(t,C_{1}\left(j_{1}\right),C_{2}\left(j_{2}\right)\right)\right|\mathbb{E}_{Z}\left[\left|\Delta_{2}^{\mathbf{H}}\left(Z,j_{2};\tilde{W}\left(t\right)\right)-\Delta_{2}^{H}\left(Z,C_{2}\left(j_{2}\right);W\left(t\right)\right)\right|\right]\\
 & \quad+K\max_{j_{2}\leq n_{2}}\left|\tilde{w}_{2}\left(t,j_{1},j_{2}\right)-w_{2}\left(t,C_{1}\left(j_{1}\right),C_{2}\left(j_{2}\right)\right)\right|\mathbb{E}_{Z}\left[\left|\Delta_{2}^{\mathbf{H}}\left(Z,j_{2};\tilde{W}\left(t\right)\right)\right|\right]\\
 & \leq K\max_{j_{2}\leq n_{2}}\left|w_{2}\left(t,C_{1}\left(j_{1}\right),C_{2}\left(j_{2}\right)\right)\right|\Big(\left|\tilde{w}_{3}\left(t,j_{2}\right)-w_{3}\left(t,C_{2}\left(j_{2}\right)\right)\right|\\
 & \quad\qquad\qquad+\left|w_{3}\left(t,C_{2}\left(j_{2}\right)\right)\right|\mathbb{E}_{Z}\left[\left|q_{3}\left(t,X\right)\right|+\left|q_{2}\left(t,X,j_{2},C_{2}\left(j_{2}\right)\right)\right|\right]\Big)\\
 & \quad+K\max_{j_{2}\leq n_{2}}\left|\tilde{w}_{2}\left(t,j_{1},j_{2}\right)-w_{2}\left(t,C_{1}\left(j_{1}\right),C_{2}\left(j_{2}\right)\right)\right|\left|\tilde{w}_{3}\left(t,j_{2}\right)\right|\\
 & \leq K_{t}\left(\mathscr{D}_{t}\left(W,\tilde{W}\right)+\mathbb{E}_{Z}\left[\left|q_{3}\left(t,X\right)\right|+\max_{j_{2}\leq n_{2}}\left|q_{2}\left(t,X,j_{2},C_{2}\left(j_{2}\right)\right)\right|\right]\right).
\end{align*}
\item To bound $Q_{1,2}$, we note that almost surely
\begin{align*}
\left|\Delta_{2}^{H}\left(Z,C_{2}\left(j_{2}\right);W\left(t\right)\right)w_{2}\left(t,C_{1}\left(j_{1}\right),C_{2}\left(j_{2}\right)\right)\right| & \leq K\left|w_{3}\left(t,C_{2}\left(j_{2}\right)\right)\right|\left|w_{2}\left(t,C_{1}\left(j_{1}\right),C_{2}\left(j_{2}\right)\right)\right|\\
 & \leq K_{t}.
\end{align*}
Then similar to the bounding of $Q_{2,2}$, we get:
\[
\mathbb{P}\left(\mathbb{E}_{Z}\left[Q_{1,2}\left(Z,j_{1}\right)\right]\geq K_{t}\gamma_{1}\right)\leq\left(1/\gamma_{1}\right)\exp\left(-n_{2}\gamma_{1}^{2}/K_{t}\right).
\]
\end{itemize}
Finally, combining all of these bounds together, applying suitably
the union bound over $j_{1}\in\left[n_{1}\right]$ and $j_{2}\in\left[n_{2}\right]$,
we obtain the claim.
\end{proof}
\begin{proof}[Proof of Theorem \ref{thm:gradient descent}]
We consider $t\leq T$, for a given terminal time $T\in\epsilon\mathbb{N}_{\geq0}$.
We again reuse the notation $K_{t}$ from the proof of Theorem \ref{thm:gradient descent coupling}.
Note that $K_{t}\leq K_{T}$ for all $t\leq T$. We also note that
at initialization, $\mathscr{D}_{0}\left(\mathbf{W},\tilde{W}\right)=0$.
We also recall from the proof of Theorem \ref{thm:gradient descent coupling}
that $\interleave\tilde{W}\interleave_{T}\leq K_{T}$ almost surely.

For brevity, let us define several quantities that relate to the difference
in the gradient updates between $\mathbf{W}$ and $\tilde{W}$:
\begin{align*}
q_{3}\left(k,z,\tilde{z},j_{2}\right) & =\partial_{2}{\cal L}\left(y,\hat{\mathbf{y}}\left(x;\mathbf{W}\left(k\right)\right)\right)\varphi_{3}'\left(\mathbf{H}_{3}\left(x;\mathbf{W}\left(k\right)\right)\right)\varphi_{2}\left(\mathbf{H}_{2}\left(x,j_{2};\mathbf{W}\left(k\right)\right)\right)\\
 & \qquad-\partial_{2}{\cal L}\left(\tilde{y},\hat{\mathbf{y}}\left(\tilde{x};\tilde{W}\left(k\epsilon\right)\right)\right)\varphi_{3}'\left(\mathbf{H}_{3}\left(\tilde{x};\tilde{W}\left(k\epsilon\right)\right)\right)\varphi_{2}\left(\mathbf{H}_{2}\left(\tilde{x},j_{2};\tilde{W}\left(k\epsilon\right)\right)\right),\\
r_{3}\left(k,z,j_{2}\right) & =\xi_{3}\left(k\epsilon\right)\partial_{2}{\cal L}\left(y,\hat{\mathbf{y}}\left(x;\tilde{W}\left(k\epsilon\right)\right)\right)\varphi_{3}'\left(\mathbf{H}_{3}\left(x;\tilde{W}\left(k\epsilon\right)\right)\right)\varphi_{2}\left(\mathbf{H}_{2}\left(x,j_{2};\tilde{W}\left(k\epsilon\right)\right)\right)\\
 & \qquad-\xi_{3}\left(k\epsilon\right)\mathbb{E}_{Z}\left[\partial_{2}{\cal L}\left(Y,\hat{\mathbf{y}}\left(X;\tilde{W}\left(k\epsilon\right)\right)\right)\varphi_{3}'\left(\mathbf{H}_{3}\left(X;\tilde{W}\left(k\epsilon\right)\right)\right)\varphi_{2}\left(\mathbf{H}_{2}\left(X,j_{2};\tilde{W}\left(k\epsilon\right)\right)\right)\right],\\
q_{2}\left(k,z,\tilde{z},j_{1},j_{2}\right) & =\Delta_{2}^{\mathbf{H}}\left(z,j_{2};\mathbf{W}\left(k\right)\right)\varphi_{1}\left(\left\langle {\bf w}_{1}\left(k,j_{1}\right),x\right\rangle \right)\\
 & \qquad-\Delta_{2}^{\mathbf{H}}\left(\tilde{z},j_{2};\tilde{W}\left(k\epsilon\right)\right)\varphi_{1}\left(\left\langle \tilde{w}_{1}\left(k\epsilon,j_{1}\right),\tilde{x}\right\rangle \right),\\
r_{2}\left(k,z,j_{1},j_{2}\right) & =\xi_{2}\left(k\epsilon\right)\Delta_{2}^{\mathbf{H}}\left(z,j_{2};\tilde{W}\left(k\epsilon\right)\right)\varphi_{1}\left(\left\langle \tilde{w}_{1}\left(k\epsilon,j_{1}\right),x\right\rangle \right)\\
 & \qquad-\xi_{2}\left(k\epsilon\right)\mathbb{E}_{Z}\left[\Delta_{2}^{\mathbf{H}}\left(Z,j_{2};\tilde{W}\left(k\epsilon\right)\right)\varphi_{1}\left(\left\langle \tilde{w}_{1}\left(k\epsilon,j_{1}\right),X\right\rangle \right)\right],\\
q_{1}\left(k,z,\tilde{z},j_{1}\right) & =\frac{1}{n_{2}}\sum_{j_{2}=1}^{n_{2}}\Delta_{2}^{\mathbf{H}}\left(z,j_{2};\mathbf{W}\left(k\right)\right)\mathbf{w}_{2}\left(k,j_{1},j_{2}\right)\varphi_{1}'\left(\left\langle \mathbf{w}_{1}\left(k,j_{1}\right),x\right\rangle \right)x\\
 & \qquad-\frac{1}{n_{2}}\sum_{j_{2}=1}^{n_{2}}\Delta_{2}^{\mathbf{H}}\left(\tilde{z},j_{2};\tilde{W}\left(k\epsilon\right)\right)\tilde{w}_{2}\left(k\epsilon,j_{1},j_{2}\right)\varphi_{1}'\left(\left\langle \tilde{w}_{1}\left(k\epsilon,j_{1}\right),\tilde{x}\right\rangle \right)\tilde{x},\\
r_{1}\left(k,z,j_{1}\right) & =\xi_{1}\left(k\epsilon\right)\frac{1}{n_{2}}\sum_{j_{2}=1}^{n_{2}}\Delta_{2}^{\mathbf{H}}\left(z,j_{2};\tilde{W}\left(k\epsilon\right)\right)\tilde{w}_{2}\left(k\epsilon,j_{1},j_{2}\right)\varphi_{1}'\left(\left\langle \tilde{w}_{1}\left(k\epsilon,j_{1}\right),x\right\rangle \right)x\\
 & \qquad-\xi_{1}\left(k\epsilon\right)\mathbb{E}_{Z}\left[\frac{1}{n_{2}}\sum_{j_{2}=1}^{n_{2}}\Delta_{2}^{\mathbf{H}}\left(Z,j_{2};\tilde{W}\left(k\epsilon\right)\right)\tilde{w}_{2}\left(k\epsilon,j_{1},j_{2}\right)\varphi_{1}'\left(\left\langle \tilde{w}_{1}\left(k\epsilon,j_{1}\right),X\right\rangle \right)X\right].
\end{align*}
Let us also define:
\begin{align*}
q_{3}^{H}\left(k,x\right) & =\mathbf{H}_{3}\left(x;\mathbf{W}\left(k\right)\right)-\mathbf{H}_{3}\left(x;\tilde{W}\left(k\epsilon\right)\right),\\
q_{2}^{H}\left(k,x,j_{2}\right) & =\mathbf{H}_{2}\left(x,j_{2};\mathbf{W}\left(k\right)\right)-\mathbf{H}_{2}\left(x,j_{2};\tilde{W}\left(k\epsilon\right)\right).
\end{align*}
We proceed in several steps.

\paragraph*{Step 1: Decomposition.}

As shown in the proof of Theorem \ref{thm:gradient descent coupling}:
\begin{align*}
\max_{j_{2}\leq n_{2}}\left|\frac{\partial}{\partial t}\tilde{w}_{3}\left(t+\xi,j_{2}\right)-\frac{\partial}{\partial t}\tilde{w}_{3}\left(t,j_{2}\right)\right| & \leq K_{t+\xi}\xi,\\
\max_{j_{1}\leq n_{1},\;j_{2}\leq n_{2}}\left|\frac{\partial}{\partial t}\tilde{w}_{2}\left(t+\xi,j_{1},j_{2}\right)-\frac{\partial}{\partial t}\tilde{w}_{2}\left(t,j_{1},j_{2}\right)\right| & \leq K_{t+\xi}\xi,\\
\max_{j_{1}\leq n_{1}}\left|\frac{\partial}{\partial t}\tilde{w}_{1}\left(t+\xi,j_{1}\right)-\frac{\partial}{\partial t}\tilde{w}_{1}\left(t,j_{1}\right)\right| & \leq K_{t+\xi}\xi.
\end{align*}
for any $t\geq0$ and $\xi\geq0$. These time-interpolation estimates,
along with Assumption \ref{assump:Regularity}, allow to derive the
following. We first have:
\begin{align*}
 & \max_{j_{2}\leq n_{2}}\left|{\bf w}_{3}\left(\left\lfloor t/\epsilon\right\rfloor ,j_{2}\right)-\tilde{w}_{3}\left(t,j_{2}\right)\right|\\
 & \leq K\max_{j_{2}\leq n_{2}}\left|\epsilon\sum_{k=0}^{\left\lfloor t/\epsilon\right\rfloor -1}\xi_{3}\left(k\epsilon\right)\mathbb{E}_{Z}\left[q_{3}\left(k,z\left(k\right),Z,j_{2}\right)\right]\right|+tK_{t}\epsilon\\
 & \leq K\max_{j_{2}\leq n_{2}}\left[Q_{3,1}\left(\left\lfloor t/\epsilon\right\rfloor ,j_{2}\right)+Q_{3,2}\left(\left\lfloor t/\epsilon\right\rfloor ,j_{2}\right)\right]+tK_{t}\epsilon,
\end{align*}
where we define
\begin{align*}
Q_{3,1}\left(\left\lfloor t/\epsilon\right\rfloor ,j_{2}\right) & =\epsilon\sum_{k=0}^{\left\lfloor t/\epsilon\right\rfloor -1}\left|q_{3}\left(k,z\left(k\right),z\left(k\right),j_{2}\right)\right|,\\
Q_{3,2}\left(\left\lfloor t/\epsilon\right\rfloor ,j_{2}\right) & =\left|\epsilon\sum_{k=0}^{\left\lfloor t/\epsilon\right\rfloor -1}r_{3}\left(k,z\left(k\right),j_{2}\right)\right|.
\end{align*}
(Here $\sum_{k=0}^{\left\lfloor t/\epsilon\right\rfloor -1}=0$ if
$\left\lfloor t/\epsilon\right\rfloor =0$.) We have similarly:
\begin{align*}
\max_{j_{1}\leq n_{1}}\left|{\bf w}_{1}\left(\left\lfloor t/\epsilon\right\rfloor ,j_{1}\right)-\tilde{w}_{1}\left(t,j_{1}\right)\right| & \leq K\max_{j_{1}\leq n_{1}}\left[Q_{1,1}\left(\left\lfloor t/\epsilon\right\rfloor ,j_{1}\right)+Q_{1,2}\left(\left\lfloor t/\epsilon\right\rfloor ,j_{1}\right)\right]+tK_{t}\epsilon,\\
\max_{j_{1}\leq n_{1},\;j_{2}\leq n_{2}}\left|{\bf w}_{2}\left(\left\lfloor t/\epsilon\right\rfloor ,j_{1},j_{2}\right)-\tilde{w}_{2}\left(t,j_{1},j_{2}\right)\right| & \leq K\max_{j_{1}\leq n_{1},\;j_{2}\leq n_{2}}\left[Q_{2,1}\left(\left\lfloor t/\epsilon\right\rfloor ,j_{1},j_{2}\right)+Q_{2,2}\left(\left\lfloor t/\epsilon\right\rfloor ,j_{1},j_{2}\right)\right]+tK_{t}\epsilon,
\end{align*}
in which
\begin{align*}
Q_{1,1}\left(\left\lfloor t/\epsilon\right\rfloor ,j_{1}\right) & =\epsilon\sum_{k=0}^{\left\lfloor t/\epsilon\right\rfloor -1}\left|q_{1}\left(k,z\left(k\right),z\left(k\right),j_{1}\right)\right|,\\
Q_{1,2}\left(\left\lfloor t/\epsilon\right\rfloor ,j_{1}\right) & =\left|\epsilon\sum_{k=0}^{\left\lfloor t/\epsilon\right\rfloor -1}r_{1}\left(k,z\left(k\right),j_{1}\right)\right|,\\
Q_{2,1}\left(\left\lfloor t/\epsilon\right\rfloor ,j_{1},j_{2}\right) & =\epsilon\sum_{k=0}^{\left\lfloor t/\epsilon\right\rfloor -1}\left|q_{2}\left(k,z\left(k\right),z\left(k\right),j_{1},j_{2}\right)\right|,\\
Q_{2,2}\left(\left\lfloor t/\epsilon\right\rfloor ,j_{1},j_{2}\right) & =\left|\epsilon\sum_{k=0}^{\left\lfloor t/\epsilon\right\rfloor -1}r_{2}\left(k,z\left(k\right),j_{1},j_{2}\right)\right|.
\end{align*}
The task is now to bound $Q_{1,1}$, $Q_{1,2}$, $Q_{2,1}$, $Q_{2,2}$,
$Q_{3,1}$ and $Q_{3,2}$.

\paragraph*{Step 2: Bounding the terms.}

Before we proceed, let us give some bounds for $q_{3}^{H}$ and $q_{2}^{H}$,
which hold for any $x\in\mathbb{R}^{d}$:
\begin{align*}
\left|q_{2}^{H}\left(k,x,j_{2}\right)\right| & \leq K\max_{j_{1}\leq n_{1}}\left(\left|\tilde{w}_{2}\left(k\epsilon,j_{1},j_{2}\right)\right|\left|{\bf w}_{1}\left(k,j_{1}\right)-\tilde{w}_{1}\left(k\epsilon,j_{1}\right)\right|+\left|{\bf w}_{2}\left(k,j_{1},j_{2}\right)-\tilde{w}_{2}\left(k\epsilon,j_{1},j_{2}\right)\right|\right)\\
 & \leq K_{k\epsilon}\mathscr{D}_{k\epsilon}\left(\tilde{W},\mathbf{W}\right),\\
\left|q_{3}^{H}\left(k,x\right)\right| & \leq K\max_{j_{2}\leq n_{2}}\left(\left|{\bf w}_{3}\left(k,j_{2}\right)-\tilde{w}_{3}\left(k\epsilon,j_{2}\right)\right|+\left|\tilde{w}_{3}\left(k\epsilon,j_{2}\right)\right|\left|q_{2}^{H}\left(k,x,j_{2}\right)\right|\right)\\
 & \leq K_{k\epsilon}\mathscr{D}_{k\epsilon}\left(\tilde{W},\mathbf{W}\right).
\end{align*}
With these, we have the following:
\begin{itemize}
\item Let us bound $Q_{3,1}$. By Assumption \ref{assump:Regularity},
\begin{align*}
\left|q_{3}\left(k,z\left(k\right),z\left(k\right),j_{2}\right)\right| & \leq K\left(q_{2}^{H}\left(k,x\left(k\right),j_{2}\right)+\left|q_{3}^{H}\left(k,x\left(k\right)\right)\right|\right).
\end{align*}
We then get:
\[
\max_{j_{2}\leq n_{2}}Q_{3,1}\left(\left\lfloor t/\epsilon\right\rfloor ,j_{2}\right)\leq K_{t}\epsilon\sum_{k=0}^{\left\lfloor t/\epsilon\right\rfloor -1}\mathscr{D}_{k\epsilon}\left(\tilde{W},\mathbf{W}\right).
\]
\item Similarly to $Q_{3,1}$, we consider $Q_{2,1}$:
\begin{align*}
\left|q_{2}\left(k,z\left(k\right),z\left(k\right),j_{1},j_{2}\right)\right| & \leq K\left|\Delta_{2}^{\mathbf{H}}\left(z\left(k\right),j_{2};\mathbf{W}\left(k\right)\right)-\Delta_{2}^{\mathbf{H}}\left(z\left(k\right),j_{2};\tilde{W}\left(k\epsilon\right)\right)\right|\\
 & \quad+K\left|\Delta_{2}^{\mathbf{H}}\left(z\left(k\right),j_{2};\tilde{W}\left(k\epsilon\right)\right)\right|\left|{\bf w}_{1}\left(k,j_{1}\right)-\tilde{w}_{1}\left(k\epsilon,j_{1}\right)\right|\\
 & \leq K_{k\epsilon}\left(\left|q_{2}^{H}\left(k,x\left(k\right),j_{2}\right)\right|+\left|q_{3}^{H}\left(k,x\left(k\right)\right)\right|\right)\\
 & \quad+K\left|{\bf w}_{3}\left(k,j_{2}\right)-\tilde{w}_{3}\left(k\epsilon,j_{2}\right)\right|+K_{k\epsilon}\left|{\bf w}_{1}\left(k,j_{1}\right)-\tilde{w}_{1}\left(k\epsilon,j_{1}\right)\right|,
\end{align*}
which yields
\[
\max_{j_{1}\leq n_{1},\;j_{2}\leq n_{2}}Q_{2,1}\left(\left\lfloor t/\epsilon\right\rfloor ,j_{1},j_{2}\right)\leq K_{t}\epsilon\sum_{k=0}^{\left\lfloor t/\epsilon\right\rfloor -1}\mathscr{D}_{k\epsilon}\left(\tilde{W},\mathbf{W}\right).
\]
\item Again we get a similar bound for $Q_{1,1}$:
\begin{align*}
\left|q_{1}\left(k,z\left(k\right),z\left(k\right),j_{1}\right)\right| & \leq\frac{K}{n_{2}}\sum_{j_{2}=1}^{n_{2}}\left|\tilde{w}_{2}\left(k\epsilon,j_{1},j_{2}\right)\right|\left|\Delta_{2}^{\mathbf{H}}\left(z\left(k\right),j_{2};\mathbf{W}\left(k\right)\right)-\Delta_{2}^{\mathbf{H}}\left(z\left(k\right),j_{2};\tilde{W}\left(k\epsilon\right)\right)\right|\\
 & \quad+\frac{K}{n_{2}}\sum_{j_{2}=1}^{n_{2}}\left|\Delta_{2}^{\mathbf{H}}\left(z\left(k\right),j_{2};\tilde{W}\left(k\epsilon\right)\right)\right|\left|\mathbf{w}_{2}\left(k,j_{1},j_{2}\right)-\tilde{w}_{2}\left(k\epsilon,j_{1},j_{2}\right)\right|\\
 & \quad+\frac{K}{n_{2}}\sum_{j_{2}=1}^{n_{2}}\left|\tilde{w}_{2}\left(k\epsilon,j_{1},j_{2}\right)\right|\left|\Delta_{2}^{\mathbf{H}}\left(z\left(k\right),j_{2};\tilde{W}\left(k\epsilon\right)\right)\right|\left|\mathbf{w}_{1}\left(k,j_{1}\right)-\tilde{w}_{1}\left(k\epsilon,j_{1}\right)\right|\\
 & \leq K_{k\epsilon}\left(\max_{j_{2}\leq n_{2}}\left|q_{2}^{H}\left(k,x\left(k\right),j_{2}\right)\right|+\left|q_{3}^{H}\left(k,x\left(k\right)\right)\right|\right)+K_{k\epsilon}\mathscr{D}_{k\epsilon}\left(\tilde{W},\mathbf{W}\right),
\end{align*}
which yields
\[
\max_{j_{1}\leq n_{1}}Q_{1,1}\left(\left\lfloor t/\epsilon\right\rfloor ,j_{1}\right)\leq K_{t}\epsilon\sum_{k=0}^{\left\lfloor t/\epsilon\right\rfloor -1}\mathscr{D}_{k\epsilon}\left(\tilde{W},\mathbf{W}\right).
\]
\item Let us bound $Q_{3,2}$. Let us define:
\[
\underline{r}_{3}\left(k,j_{2}\right)=\sum_{\ell=0}^{k-1}r_{3}\left(k,z\left(k\right),j_{2}\right),\qquad\underline{r}_{3}\left(0,j_{2}\right)=0.
\]
Let ${\cal F}_{k}$ be the sigma-algebra generated by $\left\{ z\left(\ell\right):\;\ell\in\left\{ 0,...,k-1\right\} \right\} $.
Note that $\left\{ \underline{r}_{3}\left(k,j_{2}\right)\right\} _{k\in\mathbb{N}}$
is a martingale adapted to $\left\{ {\cal F}_{k}\right\} _{k\in\mathbb{N}}$.
Furthermore, for $k\leq T/\epsilon$, the martingale difference is
bounded: $\left|r_{3}\left(k,z\left(k\right),j_{2}\right)\right|\leq K$
by Assumption \ref{assump:Regularity}. Therefore, by Theorem \ref{thm:azuma-hilbert}
and the union bound, we have:
\[
\mathbb{P}\left(\max_{j_{2}\leq n_{2}}\max_{\ell\in\left\{ 0,1,...,T/\epsilon\right\} }Q_{3,2}\left(\ell,j_{2}\right)\geq\xi\right)\leq2n_{2}\exp\left(-\frac{\xi^{2}}{K\left(T+1\right)\epsilon}\right).
\]
\item The bounding of $Q_{2,2}$ is similar: $\left|r_{2}\left(k,z\left(k\right),j_{1},j_{2}\right)\right|\leq K_{k\epsilon}$
almost surely by Assumption \ref{assump:Regularity}, and thus
\[
\mathbb{P}\left(\max_{j_{1}\leq n_{1},\;j_{2}\leq n_{2}}\max_{\ell\in\left\{ 0,1,...,T/\epsilon\right\} }Q_{2,2}\left(\ell,j_{1},j_{2}\right)\geq\xi\right)\leq2n_{1}n_{2}\exp\left(-\frac{\xi^{2}}{K_{T}\left(T+1\right)\epsilon}\right).
\]
\item Again the bounding of $Q_{1,2}$ is also similar: $\left|r_{1}\left(k,z\left(k\right),j_{1}\right)\right|\leq K_{k\epsilon}$
almost surely by Assumption \ref{assump:Regularity}, and thus
\[
\mathbb{P}\left(\max_{j_{1}\leq n_{1}}\max_{\ell\in\left\{ 0,1,...,T/\epsilon\right\} }Q_{1,2}\left(\ell,j_{1}\right)\geq\xi\right)\leq2n_{1}\exp\left(-\frac{\xi^{2}}{K_{T}\left(T+1\right)\epsilon}\right).
\]
\end{itemize}

\paragraph*{Step 3: Putting everything together.}

All the above results give us
\[
\mathscr{D}_{\left\lfloor t/\epsilon\right\rfloor \epsilon}\left(\tilde{W},\mathbf{W}\right)\leq K_{T}\epsilon\sum_{k=0}^{\left\lfloor t/\epsilon\right\rfloor -1}\mathscr{D}_{k\epsilon}\left(\tilde{W},\mathbf{W}\right)+\xi+TK_{T}\epsilon\qquad\forall t\leq T,
\]
which hold with probability at least
\[
1-2n_{1}n_{2}\exp\left(-\frac{\xi^{2}}{K_{T}\left(T+1\right)\epsilon}\right).
\]
The above event implies, by Gronwall's lemma,
\[
\mathscr{D}_{T}\left(\tilde{W},\mathbf{W}\right)\leq\left(\xi+\epsilon\right)e^{K_{T}}.
\]
Choosing $\xi=K_{T}\sqrt{\left(T+1\right)\epsilon\log\left(2n_{1}n_{2}/\delta\right)}$
completes the proof.
\end{proof}

\subsection{Proofs of Corollaries \ref{cor:gradient descent quality} and \ref{cor:MF_insensitivity}}
\begin{proof}[Proof of Corollary \ref{cor:gradient descent quality}]
By the assumption on $\psi$ and Assumption \ref{assump:Regularity},
we have:
\begin{align*}
 & \left|\mathbb{E}_{Z}\left[\psi\left(Y,\hat{\mathbf{y}}\left(X;\mathbf{W}\left(\left\lfloor t/\epsilon\right\rfloor \right)\right)\right)-\psi\left(Y,\hat{y}\left(X;W\left(t\right)\right)\right)\right]\right|\\
 & \leq K\mathbb{E}_{Z}\left[\left|\mathbf{H}_{3}\left(X;\mathbf{W}\left(\left\lfloor t/\epsilon\right\rfloor \right)\right)-H_{3}\left(X;W\left(t\right)\right)\right|\right]\\
 & \leq K\mathbb{E}_{Z}\left[\left|H_{3}\left(X;W\left(t\right)\right)-H_{3}\left(X;W\left(\left\lfloor t/\epsilon\right\rfloor \epsilon\right)\right)\right|\right]+K\mathbb{E}_{Z}\left[\left|\mathbf{H}_{3}\left(X;\mathbf{W}\left(\left\lfloor t/\epsilon\right\rfloor \right)\right)-H_{3}\left(X;W\left(\left\lfloor t/\epsilon\right\rfloor \epsilon\right)\right)\right|\right].
\end{align*}
An inspection of the proof of Theorem \ref{thm:gradient descent coupling}
(in particular, the proofs of Theorems \ref{thm:particle coupling}
and \ref{thm:gradient descent}) reveals that firstly,
\[
\sup_{t\leq T}\mathbb{E}_{Z}\left[\left|H_{3}\left(X;W\left(t\right)\right)-H_{3}\left(X;W\left(\left\lfloor t/\epsilon\right\rfloor \epsilon\right)\right)\right|\right]\leq K_{T}\epsilon,
\]
and secondly, 
\begin{align*}
 & \sup_{t\leq T}\mathbb{E}_{Z}\left[\left|H_{3}\left(X;W\left(\left\lfloor t/\epsilon\right\rfloor \epsilon\right)\right)-\mathbf{H}_{3}\left(X;\mathbf{W}\left(\left\lfloor t/\epsilon\right\rfloor \right)\right)\right|\right]\\
 & \leq K_{T}\mathscr{D}_{T}\left(W,\mathbf{W}\right)+\frac{1}{\sqrt{n_{\min}}}\log^{1/2}\left(\frac{3Tn_{\max}^{2}}{\delta}+e\right)e^{K_{T}}
\end{align*}
with probability at least $1-\delta$. Together with Theorem \ref{thm:gradient descent coupling},
we obtain the claim.
\end{proof}
\begin{proof}[Proof of Corollary \ref{cor:MF_insensitivity}]
Observe that for each index $\left\{ N_{1},N_{2}\right\} $ of $\mathsf{Init}$,
one obtains a neural network initialization ${\bf W}(0)$ with law
$\rho$ by setting 
\[
{\bf w}_{1}(0,j_{1})=w_{1}(0,C_{1}(j_{1})),\;{\bf w}_{2}(0,j_{1},j_{2})=w_{2}\left(0,C_{1}\left(j_{1}\right),C_{2}\left(j_{2}\right)\right),
\]
\[
{\bf w}_{3}(0,j_{2})=w_{3}\left(0,C_{2}\left(j_{2}\right)\right),\;j_{1}\in\left[N_{1}\right],\;j_{2}\in\left[N_{2}\right].
\]
We consider the evolution ${\bf W}\left(k\right)$ starting from ${\bf W}\left(0\right)$,
which is independent of $W$. Note that $\mathbf{W}\left(k\right)$
is a deterministic function of its initialization $\mathbf{W}\left(0\right)$
and the data $\left\{ z\left(j\right)\right\} _{j\leq k}$. Similarly,
we consider the counterpart for $\hat{W}$: the evolution $\hat{\mathbf{W}}\left(k\right)$
as a function of the initialization $\hat{\mathbf{W}}\left(0\right)$
and the data $\left\{ \hat{z}\left(j\right)\right\} _{j\leq k}$.
Due to sharing the same distribution for both the initialization and
the data, these evolutions have the same law; to be specific, $\underline{{\bf W}}\left(n_{1},n_{2},T\right)$
and $\hat{\underline{{\bf W}}}\left(n_{1},n_{2},T\right)$ has the
same distribution for any $n_{1}$, $n_{2}$ and $T$, where we define
\begin{align*}
\underline{{\bf W}}\left(n_{1},n_{2},T\right) & =\big\{{\bf w}_{1}\left(k,j_{1}\right),\;{\bf w}_{2}\left(k,j_{1},j_{2}\right),\;{\bf w}_{3}\left(k,j_{2}\right):\\
 & \qquad j_{1}\in\left[n_{1}\right],\;j_{2}\in\left[n_{2}\right],\;k\leq\left\lfloor T/\epsilon\right\rfloor \big\},
\end{align*}
and a similar definition for $\hat{\underline{{\bf W}}}\left(n_{1},n_{2},T\right)$.
In other words,
\begin{align*}
\mathscr{W}\left(\underline{{\bf W}},\hat{\underline{\mathbf{W}}}\right) & \equiv\inf_{{\rm coupling\,of\,\left(\underline{{\bf W}},\hat{\underline{\mathbf{W}}}\right)}}\mathbb{E}\bigg[\max_{k\leq\left\lfloor T/\epsilon\right\rfloor ,\;j_{1}\leq n_{1},\;j_{2}\leq n_{2}}\Big\{\left|{\bf w}_{1}\left(k,j_{1}\right)-\hat{{\bf w}}_{1}\left(k,j_{1}\right)\right|,\\
 & \qquad\left|{\bf w}_{2}\left(k,j_{1},j_{2}\right)-\hat{{\bf w}}_{2}\left(k,j_{1},j_{2}\right)\right|,\;\left|{\bf w}_{3}\left(k,j_{2}\right)-\hat{{\bf w}}_{3}\left(k,j_{2}\right)\right|\Big\}\bigg]=0.
\end{align*}
Theorem \ref{thm:gradient descent coupling} implies that for any
tuple $\left\{ n_{1},n_{2}\right\} $ such that $n_{1}\leq N_{1}$
and $n_{2}\leq N_{2}$, with probability at least $1-2\delta$,
\begin{align*}
\mathscr{D}_{T}^{\left(n_{1},n_{2}\right)}\left(W,\mathbf{W}\right)\equiv\max\bigg( & \sup_{t\le T,\;j_{1}\leq n_{1},\;j_{2}\leq n_{2}}\left|{\bf w}_{2}\left(t,j_{1},j_{2}\right)-w_{2}\left(t,C_{1}\left(j_{1}\right),C_{2}\left(j_{2}\right)\right)\right|,\\
 & \sup_{t\le T,\;j_{2}\leq n_{2}}\left|{\bf w}_{3}\left(t,j_{2}\right)-w_{3}\left(t,C_{2}\left(j_{2}\right)\right)\right|,\\
 & \sup_{t\le T,\;j_{1}\le n_{1}}\left|{\bf w}_{1}\left(t,j_{1}\right)-w_{1}\left(t,C_{1}\left(j_{1}\right)\right)\right|\bigg)\leq\tilde{O}_{\delta,T}\left(\epsilon,N_{1},N_{2}\right),
\end{align*}
where $\tilde{O}_{\delta,T}\left(\epsilon,N_{1},N_{2}\right)\to0$
as $\epsilon\to0$ and $N_{\min}^{-1}\log N_{\max}\to0$ with $N_{\min}=\min\left\{ N_{1},N_{2}\right\} $
and $N_{\max}=\max\left\{ N_{1},N_{2}\right\} $. We also have a similar
result for $\hat{\mathbf{W}}$ and $\hat{W}$. As such, with probability
at least $1-4\delta$,
\begin{align*}
\mathscr{W}\left({\cal W},\hat{{\cal W}}\right) & \equiv\inf_{{\rm coupling\,of\,}\left({\cal W},\hat{{\cal W}}\right)}\mathbb{E}\bigg[\sup_{t\leq T,\;j_{1}\leq n_{1},\;j_{2}\leq n_{2}}\Big\{\left|w_{1}\left(t,C_{1}\left(j_{1}\right)\right)-\hat{w}_{1}\left(t,\hat{C}_{1}\left(j_{1}\right)\right)\right|,\\
 & \qquad\left|w_{2}\left(t,C_{1}\left(j_{1}\right),C_{2}\left(j_{2}\right)\right)-\hat{w}_{2}\left(t,\hat{C}_{1}\left(j_{1}\right),\hat{C}_{2}\left(j_{2}\right)\right)\right|,\;\left|w_{3}\left(t,C_{2}\left(j_{2}\right)\right)-\hat{w}_{3}\left(t,\hat{C}_{2}\left(j_{2}\right)\right)\right|\Big\}\bigg]\\
 & \leq\mathscr{D}_{T}^{\left(n_{1},n_{2}\right)}\left(W,\mathbf{W}\right)+\mathscr{D}_{T}^{\left(n_{1},n_{2}\right)}\left(\hat{W},\hat{\mathbf{W}}\right)+\mathscr{W}\left(\underline{{\bf W}},\hat{\underline{\mathbf{W}}}\right)\\
 & \leq2\tilde{O}_{\delta,T}\left(\epsilon,N_{1},N_{2}\right).
\end{align*}
By fixing the tuple $\left\{ n_{1},n_{2}\right\} $ while letting
$\epsilon\to0$, $N_{\min}^{-1}\log N_{\max}\to0$ and $\delta\to0$,
we obtain the claim.
\end{proof}

\section{Global convergence: proofs for Section \ref{sec:global_conv}\label{sec:Global-convergence-proof}}

\subsection{Proof of Proposition \ref{prop:iid_law_det_representable}}
\begin{proof}[Proof of Proposition \ref{prop:iid_law_det_representable}]
Consider a probability space $\left(\Lambda,{\cal G},P_{0}\right)$
with random processes $\mathbb{R}^{d}$-valued $p_{1}\left(\theta_{1}\right)$,
$\mathbb{R}$-valued $p_{2}\left(\theta_{1},\theta_{2}\right)$ and
$\mathbb{R}$-valued $p_{3}\left(\theta_{2}\right)$, which are indexed
by $\left(\theta_{1},\theta_{2}\right)\in\left[0,1\right]\times\left[0,1\right]$,
such that the following holds. Let $m_{1}$ and $m_{2}$ be two arbitrary
finite positive integers and, with these integers, let $\left\{ \theta_{i}^{\left(k_{i}\right)}\in\left[0,1\right]:\;k_{i}\in\left[m_{i}\right],\;i=1,2\right\} $
be an arbitrary collection. For each $i=1,2$, let $S_{i}$ be the
set of unique elements in $\left\{ \theta_{i}^{\left(k_{i}\right)}:\;k_{i}\in\left[m_{i}\right]\right\} $.
Similarly, let $R_{2}$ be the set of unique pairs in $\left\{ \left(\theta_{1}^{\left(k_{1}\right)},\theta_{2}^{\left(k_{2}\right)}\right):\;k_{1}\in\left[m_{1}\right],\;k_{2}\in\left[m_{2}\right]\right\} $.
We have that $\left\{ p_{1}\left(\theta_{1}\right):\;\theta_{1}\in S_{1}\right\} $,
$\left\{ p_{3}\left(\theta_{2}\right):\;\theta_{2}\in S_{2}\right\} $and
$\left\{ p_{2}\left(\theta_{1},\theta_{2}\right):\;\left(\theta_{1},\theta_{2}\right)\in R_{2}\right\} $
are all mutually independent. In addition, we also have ${\rm Law}\left(p_{1}\left(\theta_{1}\right)\right)=\rho^{1}$,
${\rm Law}\left(p_{3}\left(\theta_{2}\right)\right)=\rho^{3}$ and
${\rm Law}\left(p_{2}\left(\theta_{1}',\theta_{2}'\right)\right)=\rho^{2}$
for any $\theta_{1}\in S_{1}$, $\theta_{2}\in S_{2}$ and $\left(\theta_{1}',\theta_{2}'\right)\in R_{2}$.
Such a space exists by Kolmogorov's extension theorem.

We now construct the desired neuronal embedding. For $i=1,2$, consider
$\Omega_{i}=\Lambda\times\left[0,1\right]$ and ${\cal F}_{i}={\cal G}\times{\cal B}\left(\left[0,1\right]\right)$,
equipped with the product measure $P_{0}\times{\rm Unif}\left(\left[0,1\right]\right)$
in which ${\rm Unif}\left(\left[0,1\right]\right)$ is the uniform
measure over $\left[0,1\right]$ equipped with the Borel sigma-algebra
${\cal B}\left(\left[0,1\right]\right)$. We construct $\Omega=\Omega_{1}\times\Omega_{2}$
and ${\cal F}={\cal F}_{1}\times{\cal F}_{2}$, equipped with the
product measure $P=\left(P_{0}\times{\rm Unif}\left(\left[0,1\right]\right)\right)^{2}$.
Define the deterministic functions $w_{1}^{0}:\;\Omega_{1}\to\mathbb{R}^{d}$,
$w_{2}^{0}:\;\Omega_{1}\times\Omega_{2}\to\mathbb{R}$ and $w_{3}^{0}:\;\Omega_{2}\to\mathbb{R}$:
\begin{align*}
w_{1}^{0}\left(\left(\lambda_{1},\theta_{1}\right)\right) & =p_{1}\left(\theta_{1}\right)\left(\lambda_{1}\right),\\
w_{2}^{0}\left(\left(\lambda_{1},\theta_{1}\right),\left(\lambda_{2},\theta_{2}\right)\right) & =p_{2}\left(\theta_{1},\theta_{2}\right)\left(\lambda_{2}\right),\\
w_{3}^{0}\left(\left(\lambda_{2},\theta_{2}\right)\right) & =p_{3}\left(\theta_{2}\right)\left(\lambda_{2}\right).
\end{align*}
It is easy to check that this construction yields the desired neuronal
embedding.
\end{proof}

\subsection{Proof of Theorem \ref{thm:global-optimum-3}}

We first present a measurability argument, which is crucial to showing
that a certain universal approximation property holds throughout the
course of training.
\begin{lem}[Measurability argument]
\label{lem:Measurability}Consider a family $\mathsf{Init}$ of initialization
laws, which are $\left(\rho^{1},\rho^{2},\rho^{3}\right)$-i.i.d.,
such that $\rho^{2}$-almost surely $\left|\mathbf{w}_{2}\right|\leq K$
and $\rho^{3}$-almost surely $\left|\mathbf{w}_{3}\right|\leq K$.
There exists a neuronal embedding $\left(\Omega,{\cal F},P,\left\{ w_{i}^{0}\right\} _{i=1,2,3}\right)$
of $\mathsf{Init}$ such that there exist Borel functions $w_{1}^{*}$
and $\Delta_{2}^{H*}$ for which $P$-almost surely, for all $t\geq0$,
\begin{align*}
w_{1}\left(t,C_{1}\right) & =w_{1}^{*}\left(t,w_{1}^{0}\left(C_{1}\right)\right),\\
\Delta_{2}^{H}\left(z,C_{2};W\left(t\right)\right) & =\Delta_{2}^{H*}\left(t,z,w_{3}^{0}\left(C_{2}\right)\right),
\end{align*}
where $W\left(t\right)$ is the MF dynamics formed under the coupling
procedure with this neuronal embedding as described in Section \ref{subsec:Neuronal-Embedding}.
Furthermore,
\begin{align*}
\frac{\partial}{\partial t}w_{1}^{*}\left(t,u_{1}\right) & =-\xi_{1}\left(t\right)\int\mathbb{E}_{Z}\left[\Delta_{2}^{H*}\left(t,Z,u_{3}\right)u_{2}\varphi_{1}'\left(\left\langle w_{1}^{*}\left(t,u_{1}\right),X\right\rangle \right)X\right]\rho^{2}\left(du_{2}\right)\rho^{3}\left(du_{3}\right)\\
 & \quad+\xi_{1}\left(t\right)\int_{0}^{t}\xi_{2}\left(s\right)\mathbb{E}_{Z,Z'}\bigg[\int\Delta_{2}^{H*}\left(t,Z,u_{3}\right)\Delta_{2}^{H*}\left(s,Z',u_{3}\right)\rho^{3}\left(du_{3}\right)\\
 & \qquad\qquad\qquad\times\varphi_{1}\left(\left\langle w_{1}^{*}\left(s,u_{1}\right),X'\right\rangle \right)\varphi_{1}'\left(\left\langle w_{1}^{*}\left(t,u_{1}\right),X\right\rangle \right)X\bigg]ds,
\end{align*}
with initialization $w_{1}^{*}\left(0,u_{1}\right)=u_{1}$ for all
$u_{1}\in{\rm supp}\left(\rho^{1}\right)$ and $t\geq0$, where $Z'$
is an independent copy of $Z$.
\end{lem}

\begin{proof}
We denote by $K_{t}$ a constant that may depend on $t$ and is finite
with finite $t$. By Proposition \ref{prop:iid_law_det_representable},
there exists a neuronal embedding that accommodates $\mathsf{Init}$.
We recall its construction and reuse the notations from the proof
of Proposition \ref{prop:iid_law_det_representable}; in particular:
\begin{align*}
w_{1}^{0}\left(\left(\lambda_{1},\theta_{1}\right)\right) & =p_{1}\left(\theta_{1}\right)\left(\lambda_{1}\right),\\
w_{2}^{0}\left(\left(\lambda_{1},\theta_{1}\right),\left(\lambda_{2},\theta_{2}\right)\right) & =p_{2}\left(\theta_{1},\theta_{2}\right)\left(\lambda_{2}\right),\\
w_{3}^{0}\left(\left(\lambda_{2},\theta_{2}\right)\right) & =p_{3}\left(\theta_{2}\right)\left(\lambda_{2}\right).
\end{align*}
Let ${\cal S}_{1}$, ${\cal S}_{3}$ and ${\cal S}_{2}$ denote the
sigma-algebras generated by $w_{1}^{0}\left(C_{1}\right)$, $w_{3}^{0}\left(C_{2}\right)$
and $\left(w_{1}^{0}\left(C_{1}\right),w_{2}^{0}\left(C_{1},C_{2}\right),w_{3}^{0}\left(C_{2}\right)\right)$
respectively. Let ${\cal S}_{13}$ denote the sigma-algebra generated
by ${\cal S}_{1}$ and ${\cal S}_{3}$. We also let ${\cal S}_{1}^{Z}$
to denote the sigma-algebra generated by ${\cal S}_{1}$ and the sigma-algebra
of the data $Z$. We define similarly for ${\cal S}_{2}^{Z}$ and
${\cal S}_{3}^{Z}$.

\paragraph*{Step 1: Reduced dynamics.}

Given the MF dynamics $W\left(t\right)$, let us define
\begin{align*}
\bar{\Delta}_{3}\left(t,c_{2}\right) & =\mathbb{E}\left[\Delta_{3}\left(C_{2};W\left(t\right)\right)\middle|{\cal S}_{3}\right]\left(c_{2}\right),\\
\bar{\Delta}_{2}\left(t,c_{1},c_{2}\right) & =\mathbb{E}\left[\Delta_{2}\left(C_{1},C_{2};W\left(t\right)\right)\middle|{\cal S}_{2}\right]\left(c_{1},c_{2}\right),\\
\bar{\Delta}_{1}\left(t,c_{1}\right) & =\mathbb{E}\left[\Delta_{1}\left(C_{1};W\left(t\right)\right)\middle|{\cal S}_{1}\right]\left(c_{1}\right).
\end{align*}
We recall from the proof of Theorem \ref{thm:gradient descent coupling}
that for any $t,s\geq0$,
\begin{align*}
{\rm ess\text{-}sup}\left|w_{3}\left(t,C_{2}\right)-w_{3}\left(s,C_{2}\right)\right| & \leq K\left|t-s\right|,\\
{\rm ess\text{-}sup}\left|w_{2}\left(t,C_{1},C_{2}\right)-w_{2}\left(s,C_{1},C_{2}\right)\right| & \leq K_{t\lor s}\left|t-s\right|,\\
{\rm ess\text{-}sup}\left|w_{1}\left(t,C_{1}\right)-w_{1}\left(s,C_{1}\right)\right| & \leq K_{t\lor s}\left|t-s\right|.
\end{align*}
Then by Lemma \ref{lem:a-priori-MF},
\begin{align*}
\mathbb{E}\left[\left|\bar{\Delta}_{3}\left(t,C_{2}\right)-\bar{\Delta}_{3}\left(s,C_{2}\right)\right|^{2}\right] & \le K_{t\lor s}\left|t-s\right|^{2},\\
\mathbb{E}\left[\left|\bar{\Delta}_{2}\left(t,C_{1},C_{2}\right)-\bar{\Delta}_{2}\left(s,C_{1},C_{2}\right)\right|^{2}\right] & \leq K_{t\lor s}\left|t-s\right|^{2},\\
\mathbb{E}\left[\left|\bar{\Delta}_{1}\left(t,C_{1}\right)-\bar{\Delta}_{1}\left(s,C_{1}\right)\right|^{2}\right] & \leq K_{t\lor s}\left|t-s\right|^{2}.
\end{align*}
Therefore, by Kolmogorov continuity theorem, there exist continuous
modifications of the (time-indexed) processes $\bar{\Delta}_{1}$,
$\bar{\Delta}_{2}$ and $\bar{\Delta}_{3}$. We thus replace them
with their continuous modifications, written by the same notations.

Given these continuous modifications, we consider the following \textit{reduced
dynamics}: 
\begin{align*}
\frac{\partial}{\partial t}\bar{w}_{3}\left(t,c_{2}\right) & =-\xi_{3}\left(t\right)\bar{\Delta}_{3}\left(t,c_{2}\right),\\
\frac{\partial}{\partial t}\bar{w}_{2}\left(t,c_{1},c_{2}\right) & =-\xi_{2}\left(t\right)\bar{\Delta}_{2}\left(t,c_{1},c_{2}\right),\\
\frac{\partial}{\partial t}\bar{w}_{1}\left(t,c_{1}\right) & =-\xi_{1}\left(t\right)\bar{\Delta}_{1}\left(t,c_{1}\right),
\end{align*}
in which:
\begin{itemize}
\item $\bar{w}_{1}:\;\mathbb{R}_{\geq0}\times\Omega_{1}\to\mathbb{R}^{d}$,
$\bar{w}_{2}:\;\mathbb{R}_{\geq0}\times\Omega_{1}\times\Omega_{2}\mapsto\mathbb{R}$,
$\bar{w}_{3}:\;\mathbb{R}_{\geq0}\times\Omega_{3}\mapsto\mathbb{R}$.
\item $\bar{W}\left(t\right)=\left\{ \bar{w}_{1}\left(t,\cdot\right),\bar{w}_{2}\left(t,\cdot,\cdot\right),\bar{w}_{3}\left(t,\cdot\right)\right\} $
is the collection of reduced parameters at time $t$,
\item the initialization is $\bar{w}_{1}\left(0,\cdot\right)=w_{1}^{0}\left(\cdot\right)$,
$\bar{w}_{2}\left(0,\cdot,\cdot\right)=w_{2}^{0}\left(\cdot,\cdot\right)$
and $\bar{w}_{3}\left(0,\cdot\right)=w_{3}^{0}\left(\cdot\right)$,
i.e. $\bar{W}\left(0\right)=W\left(0\right)$.
\end{itemize}

\paragraph*{Step 2: Measurability of the reduced dynamics.}

It is easy to see that $\bar{w}_{3}\left(t,C_{2}\right)$ is ${\cal S}_{3}$-measurable
by its construction and the fact $\bar{w}_{3}\left(0,C_{2}\right)=w_{3}^{0}\left(C_{2}\right)$
is ${\cal S}_{3}$-measurable. Similarly, $\bar{w}_{2}\left(t,C_{1},C_{2}\right)$
is ${\cal S}_{2}$-measurable and $\bar{w}_{1}\left(t,C_{1}\right)$
is ${\cal S}_{1}$-measurable.

Notice that there exist Borel functions $\bar{w}_{1}^{*}$, $\bar{w}_{2}^{*}$
and $\bar{w}_{3}^{*}$ for which $P$-almost surely, 
\begin{align*}
\bar{w}_{1}\left(t,C_{1}\right) & =\bar{w}_{1}^{*}\left(t,w_{1}^{0}\left(C_{1}\right)\right),\\
\bar{w}_{2}\left(t,C_{1},C_{2}\right) & =\bar{w}_{2}^{*}\left(t,w_{1}^{0}\left(C_{1}\right),w_{2}^{0}\left(C_{1},C_{2}\right),w_{3}^{0}\left(C_{2}\right)\right),\\
\bar{w}_{3}\left(t,C_{2}\right) & =\bar{w}_{3}^{*}\left(t,w_{3}^{0}\left(C_{2}\right)\right).
\end{align*}
Indeed, since $\bar{w}_{2}\left(t,C_{1},C_{2}\right)$ is ${\cal S}_{2}$-measurable,
there exists a function $\bar{w}_{2}^{*}\left(t,\cdot\right)$ for
each rational $t$ such that the desired identity holds for $P$-almost
every $\left(C_{1},C_{2}\right)$ and for all rational $t\geq0$.
Since $\bar{w}_{2}$ is continuous in time, there is a unique continuous
(in time) function $\bar{w}_{2}^{*}\left(t,\cdot\right)$ such that
the identity holds for all $t\geq0$ and for $P$-almost every $\left(C_{1},C_{2}\right)$.
The same argument yields the construction of $\bar{w}_{1}^{*}$ and
$\bar{w}_{3}^{*}$.

\paragraph*{Step 3: Measurability of constituent quantities.}

We show that $H_{2}\left(X,C_{2};\bar{W}\left(t\right)\right)$ is
${\cal S}_{3}^{Z}$-measurable. Recall that
\[
H_{2}\left(X,C_{2};\bar{W}\left(t\right)\right)=\mathbb{E}_{C_{1}}\left[\bar{w}_{2}\left(t,C_{1},C_{2}\right)\varphi_{1}\left(\left\langle \bar{w}_{1}\left(t,C_{1}\right),X\right\rangle \right)\right].
\]
By the existence of $\bar{w}_{1}^{*}$ and $\bar{w}_{2}^{*}$, for
each $t\geq0$, there exists a Borel function $f_{t}$ such that almost
surely
\[
\bar{w}_{2}\left(t,C_{1},C_{2}\right)\varphi_{1}\left(\left\langle \bar{w}_{1}\left(t,C_{1}\right),X\right\rangle \right)=f_{t}\left(X,w_{1}^{0}\left(C_{1}\right),w_{2}^{0}\left(C_{1},C_{2}\right),w_{3}^{0}\left(C_{2}\right)\right).
\]
We recall that $\rho^{1}$ and $\rho^{2}$ are the laws of $w_{1}^{0}\left(C_{1}\right)$
and $w_{2}^{0}\left(C_{1},C_{2}\right)$. We analyze the following:
\begin{align*}
 & \mathbb{E}\left[\left|H_{2}\left(X,C_{2};\bar{W}\left(t\right)\right)-\int f_{t}\left(X,u_{1},u_{2},w_{3}^{0}\left(C_{2}\right)\right)\rho^{1}\left(du_{1}\right)\rho^{2}\left(du_{2}\right)\right|^{2}\right]\\
 & =\mathbb{E}\left[\left|H_{2}\left(X,C_{2};\bar{W}\left(t\right)\right)\right|^{2}\right]+\mathbb{E}\left[\left|\int f_{t}\left(X,u_{1},u_{2},w_{3}^{0}\left(C_{2}\right)\right)\rho^{1}\left(du_{1}\right)\rho^{2}\left(du_{2}\right)\right|^{2}\right]\\
 & \quad-2\mathbb{E}\left[H_{2}\left(X,C_{2};\bar{W}\left(t\right)\right)\int f_{t}\left(X,u_{1},u_{2},w_{3}^{0}\left(C_{2}\right)\right)\rho^{1}\left(du_{1}\right)\rho^{2}\left(du_{2}\right)\right].
\end{align*}
Let us evaluate the first term:
\begin{align*}
 & \mathbb{E}\left[\left|H_{2}\left(X,C_{2};\bar{W}\left(t\right)\right)\right|^{2}\right]\\
 & =\mathbb{E}\left[\left|\mathbb{E}_{C_{1}}\left[f_{t}\left(X,w_{1}^{0}\left(C_{1}\right),w_{2}^{0}\left(C_{1},C_{2}\right),w_{3}^{0}\left(C_{2}\right)\right)\right]\right|^{2}\right]\\
 & \stackrel{\left(a\right)}{=}\mathbb{E}\left[f_{t}\left(X,w_{1}^{0}\left(C_{1}\right),w_{2}^{0}\left(C_{1},C_{2}\right),w_{3}^{0}\left(C_{2}\right)\right)f_{t}\left(X,w_{1}^{0}\left(C_{1}'\right),w_{2}^{0}\left(C_{1}',C_{2}\right),w_{3}^{0}\left(C_{2}\right)\right)\right]\\
 & \stackrel{\left(b\right)}{=}\mathbb{E}\Big[f_{t}\left(X,p_{1}\left(\theta_{1}\right)\left(\lambda_{1}\right),p_{2}\left(\theta_{1},\theta_{2}\right)\left(\lambda_{2}\right),p_{3}\left(\theta_{2}\right)\left(\lambda_{2}\right)\right)\\
 & \qquad\times f_{t}\left(X,p_{1}\left(\theta_{1}'\right)\left(\lambda_{1}'\right),p_{2}\left(\theta_{1}',\theta_{2}\right)\left(\lambda_{2}\right),p_{3}\left(\theta_{2}\right)\left(\lambda_{2}\right)\right)\Big]\\
 & \stackrel{\left(c\right)}{=}\mathbb{E}\Big[f_{t}\left(X,p_{1}\left(\theta_{1}\right)\left(\lambda_{1}\right),p_{2}\left(\theta_{1},\theta_{2}\right)\left(\lambda_{2}\right),p_{3}\left(\theta_{2}\right)\left(\lambda_{2}\right)\right)\\
 & \qquad\times f_{t}\left(X,p_{1}\left(\theta_{1}'\right)\left(\lambda_{1}'\right),p_{2}\left(\theta_{1}',\theta_{2}\right)\left(\lambda_{2}\right),p_{3}\left(\theta_{2}\right)\left(\lambda_{2}\right)\right)\mathbb{I}\left(\theta_{1}\neq\theta_{1}'\right)\Big]\\
 & \stackrel{\left(d\right)}{=}\mathbb{E}_{Z}\left[\int f_{t}\left(X,u_{1},u_{2},u_{3}\right)f_{t}\left(X,u_{1}',u_{2}',u_{3}\right)\rho^{1}\left(du_{1}\right)\rho^{1}\left(du_{1}'\right)\rho^{2}\left(du_{2}\right)\rho^{2}\left(du_{2}'\right)\rho^{3}\left(du_{3}\right)\right],
\end{align*}
where in step $\left(a\right)$, we define $C_{1}'$ to be an independent
copy of $C_{1}$; in step $\left(b\right)$, we recall $C_{1}=\left(\lambda_{1},\theta_{1}\right)$;
in step $\left(c\right)$, we recall $\theta_{1},\theta_{1}'\sim{\rm Unif}\left(\left[0,1\right]\right)$
and since $C_{1}$ is independent of $C_{1}'$, we have $\theta_{1}\neq\theta_{1}'$
almost surely; step $\left(d\right)$ is owing to the independence
property of the construction of the functions $p_{1}$, $p_{2}$ and
$p_{3}$. We calculate the second term:
\begin{align*}
 & \mathbb{E}\left[\left|\int f_{t}\left(X,u_{1},u_{2},w_{3}^{0}\left(C_{2}\right)\right)\rho^{1}\left(du_{1}\right)\rho^{2}\left(du_{2}\right)\right|^{2}\right]\\
 & =\int\mathbb{E}\left[f_{t}\left(X,u_{1},u_{2},w_{3}^{0}\left(C_{2}\right)\right)f_{t}\left(X,u_{1}',u_{2}',w_{3}^{0}\left(C_{2}\right)\right)\right]\rho^{1}\left(du_{1}\right)\rho^{2}\left(du_{2}\right)\rho^{1}\left(du_{1}'\right)\rho^{2}\left(du_{2}'\right)\\
 & =\int\mathbb{E}_{Z}\left[f_{t}\left(X,u_{1},u_{2},u_{3}\right)f_{t}\left(X,u_{1}',u_{2}',u_{3}\right)\right]\rho^{1}\left(du_{1}\right)\rho^{2}\left(du_{2}\right)\rho^{1}\left(du_{1}'\right)\rho^{2}\left(du_{2}'\right)\rho^{3}\left(du_{3}\right),
\end{align*}
as well as the last term:
\begin{align*}
 & \mathbb{E}\left[H_{2}\left(X,C_{2};\bar{W}\left(t\right)\right)\int f_{t}\left(X,u_{1},u_{2},w_{3}^{0}\left(C_{2}\right)\right)\rho^{1}\left(du_{1}\right)\rho^{2}\left(du_{2}\right)\right]\\
 & =\mathbb{E}\left[f_{t}\left(X,w_{1}^{0}\left(C_{1}\right),w_{2}^{0}\left(C_{1},C_{2}\right),w_{3}^{0}\left(C_{2}\right)\right)\int f_{t}\left(X,u_{1},u_{2},w_{3}^{0}\left(C_{2}\right)\right)\rho^{1}\left(du_{1}\right)\rho^{2}\left(du_{2}\right)\right]\\
 & =\int\mathbb{E}_{Z}\left[f_{t}\left(X,u_{1},u_{2},u_{3}\right)f_{t}\left(X,u_{1}',u_{2}',u_{3}\right)\right]\rho^{1}\left(du_{1}\right)\rho^{2}\left(du_{2}\right)\rho^{1}\left(du_{1}'\right)\rho^{2}\left(du_{2}'\right)\rho^{3}\left(du_{3}\right).
\end{align*}
It is then easy to see that
\[
\mathbb{E}\left[\left|H_{2}\left(X,C_{2};\bar{W}\left(t\right)\right)-\int f_{t}\left(X,u_{1},u_{2},w_{3}^{0}\left(C_{2}\right)\right)\rho^{1}\left(du_{1}\right)\rho^{2}\left(du_{2}\right)\right|^{2}\right]=0.
\]
That is, we have almost surely
\[
H_{2}\left(X,C_{2};\bar{W}\left(t\right)\right)=\int f_{t}\left(X,u_{1},u_{2},w_{3}^{0}\left(C_{2}\right)\right)\rho^{1}\left(du_{1}\right)\rho^{2}\left(du_{2}\right).
\]
Note that the right-hand side is ${\cal S}_{3}^{Z}$-measurable, and
hence so is $H_{2}\left(X,C_{2};\bar{W}\left(t\right)\right)$.

Next we consider $\Delta_{2}^{H}\left(Z,C_{2};\bar{W}\left(t\right)\right)$.
Recall that
\[
\Delta_{2}^{H}\left(z,c_{2};\bar{W}\left(t\right)\right)=\partial_{2}{\cal L}\left(y,\hat{y}\left(x;\bar{W}\left(t\right)\right)\right)\varphi_{3}'\left(H_{3}\left(x;\bar{W}\left(t\right)\right)\right)\bar{w}_{3}\left(t,c_{2}\right)\varphi_{2}'\left(H_{2}\left(x,c_{2};\bar{W}\left(t\right)\right)\right).
\]
Then together with the existence of $\bar{w}_{3}^{*}$, we have $\Delta_{2}^{H}\left(Z,C_{2};\bar{W}\left(t\right)\right)$
is ${\cal S}_{3}^{Z}$-measurable.

Now we consider $\mathbb{E}_{C_{2}}\left[\Delta_{2}^{H}\left(Z,C_{2};\bar{W}\left(t\right)\right)\bar{w}_{2}\left(t,C_{1},C_{2}\right)\right]$.
With the existence of $\bar{w}_{2}^{*}$, there exists a Borel function
$g_{t}$ such that
\[
\Delta_{2}^{H}\left(Z,C_{2};\bar{W}\left(t\right)\right)\bar{w}_{2}\left(t,C_{1},C_{2}\right)=g_{t}\left(Z,w_{1}^{0}\left(C_{1}\right),w_{2}^{0}\left(C_{1},C_{2}\right),w_{3}^{0}\left(C_{2}\right)\right).
\]
Then with the same argument as the treatment of $H_{2}\left(X,C_{2};\bar{W}\left(t\right)\right)$,
one can show that
\[
\mathbb{E}_{C_{2}}\left[\Delta_{2}^{H}\left(Z,C_{2};\bar{W}\left(t\right)\right)\bar{w}_{2}\left(t,C_{1},C_{2}\right)\right]=\int g_{t}\left(Z,w_{1}^{0}\left(C_{1}\right),u_{2},u_{3}\right)\rho^{2}\left(du_{2}\right)\rho^{3}\left(du_{3}\right),
\]
which is ${\cal S}_{1}^{Z}$-measurable.

Using these facts together with the existence of $\bar{w}_{1}^{*}$,
$\bar{w}_{2}^{*}$ and $\bar{w}_{3}^{*}$, we have $\Delta_{3}\left(C_{2};\bar{W}\left(t\right)\right)$
is ${\cal S}_{3}$-measurable, $\Delta_{2}\left(C_{1},C_{2};\bar{W}\left(t\right)\right)$
is ${\cal S}_{13}$-measurable and $\Delta_{1}\left(C_{1};\bar{W}\left(t\right)\right)$
is ${\cal S}_{1}$-measurable.

\paragraph*{Step 4: Closeness between the MF dynamics and the reduced dynamics.}

We shall use $\left\Vert W-\bar{W}\right\Vert _{t}$ with the same
meaning as the distance between two sets of MF parameters. Recall
by Lemma \ref{lem:a-priori-MF-norms} that $\interleave W\interleave_{T}\leq K_{T}$
since $\interleave W\interleave_{0}\leq K$. By the same argument,
$\left\Vert \bar{W}\right\Vert _{T}\leq K_{T}$. Then by Lemma \ref{lem:a-priori-MF},
we have for any $t\leq T$,
\begin{align*}
{\rm ess\text{-}sup}\sup_{s\leq t}\left|\Delta_{3}\left(C_{2};W\left(s\right)\right)-\Delta_{3}\left(C_{2};\bar{W}\left(s\right)\right)\right| & \leq K_{T}\left\Vert W-\bar{W}\right\Vert _{t},\\
{\rm ess\text{-}sup}\sup_{s\leq t}\left|\Delta_{2}\left(C_{1},C_{2};W\left(s\right)\right)-\Delta_{2}\left(C_{1},C_{2};\bar{W}\left(s\right)\right)\right| & \leq K_{T}\left\Vert W-\bar{W}\right\Vert _{t},\\
{\rm ess\text{-}sup}\sup_{s\leq t}\left|\Delta_{1}\left(C_{1};W\left(s\right)\right)-\Delta_{1}\left(C_{1};\bar{W}\left(s\right)\right)\right| & \leq K_{T}\left\Vert W-\bar{W}\right\Vert _{t}.
\end{align*}
We have:
\begin{align*}
 & {\rm ess\text{-}sup}\left|\bar{\Delta}_{2}\left(t,C_{1},C_{2}\right)-\Delta_{2}\left(C_{1},C_{2};\bar{W}\left(t\right)\right)\right|\\
 & ={\rm ess\text{-}sup}\left|\mathbb{E}\left[\Delta_{2}\left(C_{1},C_{2};W\left(t\right)\right)\middle|{\cal S}_{2}\right]-\Delta_{2}\left(C_{1},C_{2};\bar{W}\left(t\right)\right)\right|\\
 & \leq{\rm ess\text{-}sup}\left|\mathbb{E}\left[\Delta_{2}\left(C_{1},C_{2};\bar{W}\left(t\right)\right)\middle|{\cal S}_{2}\right]-\Delta_{2}\left(C_{1},C_{2};\bar{W}\left(t\right)\right)\right|\\
 & \quad+{\rm ess\text{-}sup}\left|\mathbb{E}\left[\Delta_{2}\left(C_{1},C_{2};\bar{W}\left(t\right)\right)-\Delta_{2}\left(C_{1},C_{2};W\left(t\right)\right)\middle|{\cal S}_{2}\right]\right|\\
 & \stackrel{\left(a\right)}{=}{\rm ess\text{-}sup}\left|\mathbb{E}\left[\Delta_{2}\left(C_{1},C_{2};\bar{W}\left(t\right)\right)-\Delta_{2}\left(C_{1},C_{2};W\left(t\right)\right)\middle|{\cal S}_{2}\right]\right|\\
 & \leq{\rm ess\text{-}sup}\left|\Delta_{2}\left(C_{1},C_{2};\bar{W}\left(t\right)\right)-\Delta_{2}\left(C_{1},C_{2};W\left(t\right)\right)\right|,
\end{align*}
where step $\left(a\right)$ is because $\Delta_{2}\left(C_{1},C_{2};\bar{W}\left(t\right)\right)$
is ${\cal S}_{13}$-measurable from Step 3 and ${\cal S}_{13}\subseteq{\cal S}_{2}$.
As such,
\[
{\rm ess\text{-}sup}\left|\bar{\Delta}_{2}\left(t,C_{1},C_{2}\right)-\Delta_{2}\left(C_{1},C_{2};W\left(t\right)\right)\right|\leq2K_{T}\left\Vert W-\bar{W}\right\Vert _{t}
\]
almost surely for all rational $t\leq T$. By continuity in $t$ of
both sides, the same holds for all $t\leq T$. Hence by Assumption
\ref{assump:Regularity},
\begin{align*}
\left|\frac{\partial}{\partial t}\bar{w}_{2}\left(t,C_{1},C_{2}\right)-\frac{\partial}{\partial t}w_{2}\left(t,C_{1},C_{2}\right)\right| & \leq K\left|\bar{\Delta}_{2}\left(t,C_{1},C_{2}\right)-\Delta_{2}\left(C_{1},C_{2};W\left(t\right)\right)\right|\\
 & \leq2K_{T}\left\Vert W-\bar{W}\right\Vert _{t},
\end{align*}
for all $t\leq T$ almost surely, which leads to
\[
\left|\bar{w}_{2}\left(t,C_{1},C_{2}\right)-w_{2}\left(t,C_{1},C_{2}\right)\right|\leq2K_{T}\int_{0}^{t}\left\Vert W-\bar{W}\right\Vert _{s}ds
\]
almost surely. One can obtain similar results for $\bar{w}_{1}$ versus
$w_{1}$ and $\bar{w}_{3}$ versus $w_{3}$. Therefore,
\[
\left\Vert W-\bar{W}\right\Vert _{t}\leq K_{T}\int_{0}^{t}\left\Vert W-\bar{W}\right\Vert _{s}ds.
\]
Since $W\left(0\right)=\bar{W}\left(0\right)$, by Gronwall's inequality,
$\left\Vert W-\bar{W}\right\Vert _{t}=0$ for all $t\leq T$. In other
words, since $T$ is arbitrary,
\[
\bar{w}_{1}\left(t,C_{1}\right)=w_{1}\left(t,C_{1}\right),\quad\bar{w}_{2}\left(t,C_{1},C_{2}\right)=w_{2}\left(t,C_{1},C_{2}\right),\quad\bar{w}_{3}\left(t,C_{2}\right)=w_{3}\left(t,C_{2}\right),
\]
for all $t\geq0$ almost surely.

\paragraph*{Step 5: Concluding.}

The first claim of the lemma is proven by the conclusion of Step 4
and by choosing $w_{1}^{*}=\bar{w}_{1}^{*}$, $w_{2}^{*}=\bar{w}_{2}^{*}$
and $w_{3}^{*}=\bar{w}_{3}^{*}$, as well as the measurability facts
from Step 3. To prove the second claim, since $\Delta_{2}^{H}\left(Z,C_{2};\bar{W}\left(t\right)\right)$
is ${\cal S}_{3}^{Z}$-measurable and $\left\Vert W-\bar{W}\right\Vert _{t}=0$
for all $t\geq0$, there exists a Borel function $\Delta_{2}^{H*}$
such that
\[
\Delta_{2}^{H}\left(Z,C_{2};\bar{W}\left(t\right)\right)=\Delta_{2}^{H}\left(Z,C_{2};W\left(t\right)\right)=\Delta_{2}^{H*}\left(t,Z,w_{3}^{0}\left(C_{2}\right)\right)
\]
for all $t\geq0$ almost surely, by the same argument in Step 2. These
facts, together with the dynamics of $w_{1}$ and $w_{2}$, imply
that almost surely, for all $t\geq0$,
\begin{align*}
 & \frac{\partial}{\partial t}w_{2}\left(t,C_{1},C_{2}\right)\\
 & \quad=-\xi_{2}\left(t\right)\mathbb{E}_{Z}\left[\Delta_{2}^{H*}\left(t,Z,w_{3}^{0}\left(C_{2}\right)\right)\varphi_{1}\left(\left\langle w_{1}^{*}\left(t,w_{1}^{0}\left(C_{1}\right)\right),X\right\rangle \right)\right],\\
 & \frac{\partial}{\partial t}w_{1}^{*}\left(t,w_{1}^{0}\left(C_{1}\right)\right)\\
 & \quad=-\xi_{1}\left(t\right)\mathbb{E}_{Z}\left[\mathbb{E}_{C_{2}}\left[\Delta_{2}^{H*}\left(t,Z,w_{3}^{0}\left(C_{2}\right)\right)w_{2}\left(t,C_{1},C_{2}\right)\right]\varphi_{1}'\left(\left\langle w_{1}^{*}\left(t,w_{1}^{0}\left(C_{1}\right)\right),X\right\rangle \right)X\right],
\end{align*}
with initialization $w_{1}^{*}\left(0,w_{1}^{0}\left(C_{1}\right)\right)=w_{1}^{0}\left(C_{1}\right)$.
Substituting the first equation into the second one, we get:
\begin{align*}
\frac{\partial}{\partial t}w_{1}^{*}\left(t,w_{1}^{0}\left(C_{1}\right)\right) & =-\xi_{1}\left(t\right)\mathbb{E}_{Z}\left[\mathbb{E}_{C_{2}}\left[\Delta_{2}^{H*}\left(t,Z,w_{3}^{0}\left(C_{2}\right)\right)w_{2}^{0}\left(C_{1},C_{2}\right)\right]\varphi_{1}'\left(\left\langle w_{1}^{*}\left(t,w_{1}^{0}\left(C_{1}\right)\right),X\right\rangle \right)X\right]\\
 & \quad+\xi_{1}\left(t\right)\int_{0}^{t}\xi_{2}\left(s\right)\mathbb{E}_{Z,Z'}\bigg[\mathbb{E}_{C_{2}}\left[\Delta_{2}^{H*}\left(t,Z,w_{3}^{0}\left(C_{2}\right)\right)\Delta_{2}^{H*}\left(s,Z',w_{3}^{0}\left(C_{2}\right)\right)\right]\\
 & \qquad\qquad\qquad\times\varphi_{1}\left(\left\langle w_{1}^{*}\left(s,w_{1}^{0}\left(C_{1}\right)\right),X'\right\rangle \right)\varphi_{1}'\left(\left\langle w_{1}^{*}\left(t,w_{1}^{0}\left(C_{1}\right)\right),X\right\rangle \right)X\bigg]ds.
\end{align*}
Note that by an argument similar to Step 3,
\[
\mathbb{E}_{C_{2}}\left[\Delta_{2}^{H*}\left(t,Z,w_{3}^{0}\left(C_{2}\right)\right)w_{2}^{0}\left(C_{1},C_{2}\right)\right]=\int\Delta_{2}^{H*}\left(t,Z,u_{3}\right)u_{2}\rho^{2}\left(du_{2}\right)\rho^{3}\left(du_{3}\right),
\]
which holds for all $t\geq0$ almost surely by the same argument in
Step 2. We thus obtain:
\begin{align*}
\frac{\partial}{\partial t}w_{1}^{*}\left(t,u_{1}\right) & =-\xi_{1}\left(t\right)\int\mathbb{E}_{Z}\left[\Delta_{2}^{H*}\left(t,Z,u_{3}\right)u_{2}\varphi_{1}'\left(\left\langle w_{1}^{*}\left(t,u_{1}\right),X\right\rangle \right)X\right]\rho^{2}\left(du_{2}\right)\rho^{3}\left(du_{3}\right)\\
 & \quad+\xi_{1}\left(t\right)\int_{0}^{t}\xi_{2}\left(s\right)\mathbb{E}_{Z,Z'}\bigg[\int\Delta_{2}^{H*}\left(t,Z,u_{3}\right)\Delta_{2}^{H*}\left(s,Z',u_{3}\right)\rho^{3}\left(du_{3}\right)\\
 & \qquad\qquad\qquad\times\varphi_{1}\left(\left\langle w_{1}^{*}\left(s,u_{1}\right),X'\right\rangle \right)\varphi_{1}'\left(\left\langle w_{1}^{*}\left(t,u_{1}\right),X\right\rangle \right)X\bigg]ds,
\end{align*}
with initialization $w_{1}^{*}\left(t,u_{1}\right)=u_{1}$ for all
$u_{1}\in{\rm supp}\left(\rho^{1}\right)$ and $t\geq0$.
\end{proof}
An important ingredient of the proof is that the distribution of $w_{1}\left(t,C_{1}\right)$
has full support at all time $t\geq0$, even though we only need to
assume this property at initialization $t=0$. This key property is
proven by a topology argument, supported by the measurability result
of Lemma \ref{lem:Measurability}. We remark that a similar property
for two-layer networks is established in \cite{chizat2018} using
a different topology argument.
\begin{lem}
\label{lem:full-support-3}Consider the same setting as Theorem \ref{thm:global-optimum-3}.
For all finite time $t\geq0$, the support of ${\rm Law}\left(w_{1}\left(t,C_{1}\right)\right)$
is $\mathbb{R}^{d}$.
\end{lem}

\begin{proof}
By Lemma \ref{lem:Measurability}, one can choose a neural embedding
such that there exists Borel functions $w_{1}^{*}$ and $\Delta_{2}^{H*}$
for which almost surely, for all $t\geq0$,
\begin{align*}
w_{1}\left(t,C_{1}\right) & =w_{1}^{*}\left(t,w_{1}^{0}\left(C_{1}\right)\right),\\
\Delta_{2}^{H}\left(z,C_{2};W\left(t\right)\right) & =\Delta_{2}^{H*}\left(t,z,w_{3}^{0}\left(C_{2}\right)\right),
\end{align*}
where $W\left(t\right)$ is the MF dynamics formed under the coupling
procedure with this neuronal embedding as described in Section \ref{subsec:Neuronal-Embedding}.
Furthermore,
\begin{align*}
\frac{\partial}{\partial t}w_{1}^{*}\left(t,u_{1}\right) & =-\int\mathbb{E}_{Z}\left[\Delta_{2}^{H*}\left(t,Z,u_{3}\right)u_{2}\varphi_{1}'\left(\left\langle w_{1}^{*}\left(t,u_{1}\right),X\right\rangle \right)X\right]\rho^{2}\left(du_{2}\right)\rho^{3}\left(du_{3}\right)\\
 & \quad+\int_{0}^{t}\mathbb{E}_{Z,Z'}\bigg[\int\Delta_{2}^{H*}\left(t,Z,u_{3}\right)\Delta_{2}^{H*}\left(s,Z',u_{3}\right)\rho^{3}\left(du_{3}\right)\\
 & \qquad\qquad\times\varphi_{1}\left(\left\langle w_{1}^{*}\left(s,u_{1}\right),X'\right\rangle \right)\varphi_{1}'\left(\left\langle w_{1}^{*}\left(t,u_{1}\right),X\right\rangle \right)X\bigg]ds,
\end{align*}
with initialization $w_{1}^{*}\left(0,u_{1}\right)=u_{1}$ for all
$u_{1}\in{\rm supp}\left(\rho^{1}\right)$ and $t\geq0$, where $Z'$
is an independent copy of $Z$. We recall from Lemma \ref{lem:Lipschitz-MF}
that
\[
{\rm ess\text{-}sup}\Delta_{2}^{H*}\left(t,Z,w_{3}^{0}\left(C_{2}\right)\right)={\rm ess\text{-}sup}\Delta_{2}^{H}\left(Z,C_{2};W\left(t\right)\right)\leq K_{t},
\]
where $K_{t}$ denotes a generic constant that depends on $t$ and
is finite with finite $t$. Therefore, by Assumption \ref{assump:Regularity},
for $t\leq T$ and $u_{1},u_{1}'\in{\rm supp}\left(\rho^{1}\right)$,
\begin{align*}
\left|\frac{\partial}{\partial t}w_{1}^{*}\left(t,u_{1}\right)-\frac{\partial}{\partial t}w_{1}^{*}\left(t,u_{1}'\right)\right| & \leq K_{t}\left|w_{1}^{*}\left(t,u_{1}\right)-w_{1}^{*}\left(t,u_{1}'\right)\right|+K_{t}\int_{0}^{t}\left|w_{1}^{*}\left(s,u_{1}\right)-w_{1}^{*}\left(s,u_{1}'\right)\right|ds\\
 & \leq K_{T}\sup_{s\leq t}\left|w_{1}^{*}\left(s,u_{1}\right)-w_{1}^{*}\left(s,u_{1}'\right)\right|,\\
\left|\frac{\partial}{\partial t}w_{1}^{*}\left(t,u_{1}\right)\right| & \leq K_{T}.
\end{align*}
Applying Gronwall's lemma to the first bound:
\begin{align*}
\sup_{t\leq T}\left|w_{1}^{*}\left(t,u_{1}\right)-w_{1}^{*}\left(t,u_{1}'\right)\right| & \leq e^{K_{T}}\left|w_{1}^{*}\left(0,u_{1}\right)-w_{1}^{*}\left(0,u_{1}'\right)\right|\\
 & =e^{K_{T}}\left|u_{1}-u_{1}'\right|.
\end{align*}
Furthermore the second bound implies
\[
\sup_{t,t'\leq T}\left|w_{1}^{*}\left(t,u_{1}\right)-w_{1}^{*}\left(t',u_{1}\right)\right|\leq K_{T}\left|t-t'\right|.
\]
Therefore $\left(t,u_{1}\right)\mapsto w_{1}^{*}\left(t,u_{1}\right)$
is a continuous mapping on $\left[0,T\right]\times\mathbb{R}^{d}$
for an arbitrary $T\geq0$.

Given this continuity, we show the thesis by a topology argument.
Consider the sphere $\mathbb{S}^{d}$ which is a compactification
of $\mathbb{R}^{d}$. We can extend $w_{1}^{*}$ to a function $M:\;\left[0,T\right]\times\mathbb{S}^{d}\to\mathbb{S}^{d}$
fixing the point at infinity, which remains a continuous map since
$\left|M\left(t,u_{1}\right)-u_{1}\right|=\left|M\left(t,u_{1}\right)-M\left(0,u_{1}\right)\right|\leq K_{T}t$.
Let $M_{t}:\;\mathbb{R}^{d}\to\mathbb{R}^{d}$ be defined by $M_{t}\left(u_{1}\right)=M\left(t,u_{1}\right)$.
We claim that $M_{t}$ is surjective for all finite $t$. Indeed,
if $M_{t}$ fails to be surjective for some $t$, then for some $p\in\mathbb{S}^{d}$,
$M_{t}:\;\mathbb{S}^{d}\to\mathbb{S}^{d}\backslash\left\{ p\right\} \to\mathbb{S}^{d}$
is homotopic to the constant map, but $M$ then gives a homotopy from
the identity map $M_{0}$ on the sphere to a constant map, which is
a contradiction as the sphere $\mathbb{S}^{d}$ is not contractible.
Hence $w_{1}^{*}\left(t,\cdot\right)$ is surjective for all finite
$t$. Recall that $w_{1}\left(t,C_{1}\right)=w_{1}^{*}\left(t,w_{1}^{0}\left(C_{1}\right)\right)$
almost surely and $w_{1}^{0}\left(C_{1}\right)$ has full support.
Now let us assume that $w_{1}\left(t,C_{1}\right)$ does not have
full support at some time $t$, which implies there is an open ball
$B$ in $\mathbb{R}^{d}$ for which $\mathbb{P}\left(w_{1}\left(t,C_{1}\right)\in B\right)=0$.
Then $\mathbb{P}\left(w_{1}^{*}\left(t,w_{1}^{0}\left(C_{1}\right)\right)\in B\right)=0$.
Since $w_{1}^{*}\left(t,\cdot\right)$ has full support, there is
an open set $U$ such that $w_{1}^{*}\left(t,u_{1}\right)\in B$ for
all $u_{1}\in U$. Then $\mathbb{P}\left(w_{1}^{0}\left(C_{1}\right)\in U\right)=0$,
contradicting the assumption that $w_{1}^{0}\left(C_{1}\right)$ has
full support. Therefore $w_{1}\left(t,C_{1}\right)$ must have full
support at all $t\geq0$.
\end{proof}
With this, we are ready to prove Theorem \ref{thm:global-optimum-3}.
\begin{proof}[Proof of Theorem \ref{thm:global-optimum-3}]
Recall, by Theorem \ref{thm:existence ODE}, the solution to the
MF ODEs exists uniquely, and by Lemma \ref{lem:full-support-3}, the
support of ${\rm Law}\left(w_{1}\left(t,C_{1}\right)\right)$ is $\mathbb{R}^{d}$
at all $t$. By the convergence assumption, we have that for any $\epsilon>0$,
there exists $T\left(\epsilon\right)$ such that for all $t\geq T\left(\epsilon\right)$
and $P$-almost every $c_{1}$:
\[
\mathbb{E}_{C_{2}}\left[\left|\mathbb{E}_{Z}\left[\Delta_{2}^{H}\left(Z,C_{2};W\left(t\right)\right)\varphi_{1}\left(\left\langle w_{1}\left(t,c_{1}\right),X\right\rangle \right)\right]\right|\right]\leq\epsilon.
\]
Since ${\rm Law}\left(w_{1}\left(t,C_{1}\right)\right)$ has full
support, we obtain that for $u$ in a dense subset of $\mathbb{R}^{d}$,
\[
\mathbb{E}_{C_{2}}\left[\left|\mathbb{E}_{Z}\left[\Delta_{2}^{H}\left(Z,C_{2};W\left(t\right)\right)\varphi_{1}\left(\left\langle u,X\right\rangle \right)\right]\right|\right]\leq\epsilon.
\]
By continuity of $u\mapsto\varphi_{1}(\left\langle u,x\right\rangle )$,
we extend the above to all $u\in\mathbb{R}^{d}$. Since $\varphi_{1}$
is bounded,
\begin{align*}
 & \mathbb{E}_{C_{2}}\left[\left|\mathbb{E}_{Z}\left[\left(\Delta_{2}^{H}\left(Z,C_{2};W\left(t\right)\right)-\Delta_{2}^{H}\left(Z,C_{2};\bar{w}_{1},\bar{w}_{2},\bar{w}_{3}\right)\right)\varphi_{1}\left(\left\langle u,X\right\rangle \right)\right]\right|\right]\\
 & \le K\mathbb{E}\left[\left|\Delta_{2}^{H}\left(Z,C_{2};W\left(t\right)\right)-\Delta_{2}^{H}\left(Z,C_{2};\bar{w}_{1},\bar{w}_{2},\bar{w}_{3}\right)\right|\right]\\
 & \le K\mathbb{E}\Big[\left(1+\left|\bar{w}_{3}(C_{2})\right|\right)\Big(\left|w_{3}(t,C_{2})-\bar{w}_{3}(C_{2})\right|+\left|\bar{w}_{3}(C_{2})\right|\left|w_{2}(t,C_{1},C_{2})-\bar{w}_{2}(C_{1},C_{2})\right|\\
 & \qquad+\left|\bar{w}_{3}(C_{2})\right|\left|\bar{w}_{2}(C_{1},C_{2})\right|\left|w_{1}(t,C_{1})-\bar{w}_{1}(C_{1})\right|\Big)\Big],
\end{align*}
where the last step is by Assumption \ref{assump:Regularity}. Recall
that the right-hand side converges to $0$ as $t\to\infty$. We thus
obtain that for all $u\in\mathbb{R}^{d}$,
\begin{align*}
 & \mathbb{E}_{C_{2}}\left[\left|\left\langle \mathbb{E}_{Z}\left[\Delta_{2}^{H}\left(Z,C_{2};\bar{w}_{1},\bar{w}_{2},\bar{w}_{3}\right)|X=x\right],\varphi_{1}\left(\left\langle u,x\right\rangle \right)\right\rangle _{L^{2}\left({\cal P}_{X}\right)}\right|\right]\\
 & =\mathbb{E}_{C_{2}}\left[\left|\mathbb{E}_{Z}\left[\Delta_{2}^{H}\left(Z,C_{2};\bar{w}_{1},\bar{w}_{2},\bar{w}_{3}\right)\varphi_{1}\left(\left\langle u,X\right\rangle \right)\right]\right|\right]\\
 & =0,
\end{align*}
which yields that for all $u\in\mathbb{R}^{d}$ and $P$-almost every
$c_{2}$,
\[
\left|\left\langle \mathbb{E}_{Z}\left[\Delta_{2}^{H}\left(Z,c_{2};\bar{w}_{1},\bar{w}_{2},\bar{w}_{3}\right)|X=x\right],\varphi_{1}\left(\left\langle u,x\right\rangle \right)\right\rangle _{L^{2}\left({\cal P}_{X}\right)}\right|=0.
\]
Here we note that by Assumption \ref{assump:Regularity},
\[
\left|\mathbb{E}_{Z}\left[\Delta_{2}^{H}\left(Z,c_{2};\bar{w}_{1},\bar{w}_{2},\bar{w}_{3}\right)|X=x\right]\right|\leq K\left|\bar{w}_{3}\left(c_{2}\right)\right|,
\]
and so $\mathbb{E}_{Z}\left[\Delta_{2}^{H}\left(Z,c_{2};\bar{w}_{1},\bar{w}_{2},\bar{w}_{3}\right)|X=x\right]$
is in $L^{2}\left({\cal P}_{X}\right)$ for $P$-almost every $c_{2}$.
Since $\left\{ \varphi_{1}\left(\left\langle u,\cdot\right\rangle \right):\;u\in\mathbb{R}^{d}\right\} $
has dense span in $L^{2}\left({\cal P}_{X}\right)$, we have $\mathbb{E}_{Z}\left[\Delta_{2}^{H}\left(Z,c_{2};\bar{w}_{1},\bar{w}_{2},\bar{w}_{3}\right)|X=x\right]=0$
for ${\cal P}_{X}$-almost every $x$ and $P$-almost every $c_{2}$,
and hence
\[
\mathbb{E}_{Z}\left[\partial_{2}{\cal L}\left(Y,\hat{y}\left(X;\bar{w}_{1},\bar{w}_{2},\bar{w}_{3}\right)\right)\middle|X=x\right]\varphi_{3}'\left(H_{3}\left(x;\bar{w}_{1},\bar{w}_{2},\bar{w}_{3}\right)\right)\bar{w}_{3}\left(c_{2}\right)\varphi_{2}'\left(H_{2}\left(x,c_{2};\bar{w}_{1},\bar{w}_{2}\right)\right)=0.
\]
We note that our assumptions guarantee that $\mathbb{P}\left(\bar{w}_{3}\left(C_{2}\right)\ne0\right)$
is positive. Indeed:
\begin{itemize}
\item In the case $w_{3}^{0}\left(C_{2}\right)\neq0$ with positive probability
and $\xi_{3}\left(\cdot\right)=0$, the conclusion is obvious.
\item In the case $\mathscr{L}\left(w_{1}^{0},w_{2}^{0},w_{3}^{0}\right)<\mathbb{E}_{Z}\left[{\cal L}\left(Y,\varphi_{3}\left(0\right)\right)\right]$,
we recall the following standard property of gradient flows:
\[
\mathscr{L}\left(w_{1}\left(t,\cdot\right),w_{2}\left(t,\cdot,\cdot\right),w_{3}\left(t,\cdot\right)\right)\leq\mathscr{L}\left(w_{1}\left(t',\cdot\right),w_{2}\left(t',\cdot,\cdot\right),w_{3}\left(t',\cdot\right)\right),
\]
for $t\geq t'$. In particular, setting $t'=0$ and taking $t\to\infty$,
it is easy to see that
\[
\mathscr{L}\left(\bar{w}_{1},\bar{w}_{2},\bar{w}_{3}\right)\leq\mathscr{L}\left(w_{1}^{0},w_{2}^{0},w_{3}^{0}\right)<\mathbb{E}_{Z}\left[{\cal L}\left(Y,\varphi_{3}\left(0\right)\right)\right].
\]
If $\mathbb{P}\left(\bar{w}_{3}\left(C_{2}\right)=0\right)=1$ then
$\mathscr{L}\left(\bar{w}_{1},\bar{w}_{2},\bar{w}_{3}\right)=\mathbb{E}_{Z}\left[{\cal L}\left(Y,\varphi_{3}\left(0\right)\right)\right]$,
a contradiction.
\end{itemize}
Then since $\varphi_{2}'$ and $\varphi_{3}'$ are strictly non-zero,
we have $\mathbb{E}_{Z}\left[\partial_{2}{\cal L}\left(Y,\hat{y}\left(X;\bar{w}_{1},\bar{w}_{2},\bar{w}_{3}\right)\right)\middle|X=x\right]=0$
for ${\cal P}_{X}$-almost every $x$.

In Case 1, since ${\cal L}$ convex in the second variable, for any
measurable function $\tilde{y}(x)$, 
\[
{\cal L}\left(y,\tilde{y}\left(x\right)\right)-{\cal L}\left(y,\hat{y}\left(x;\bar{w}_{1},\bar{w}_{2},\bar{w}_{3}\right)\right)\ge\partial_{2}{\cal L}\left(y,\hat{y}\left(x;\bar{w}_{1},\bar{w}_{2},\bar{w}_{3}\right)\right)\left(\tilde{y}\left(x\right)-\hat{y}\left(x;\bar{w}_{1},\bar{w}_{2},\bar{w}_{3}\right)\right).
\]
Taking expectation, we get $\mathbb{E}_{Z}\left[{\cal L}\left(Y,\tilde{y}\left(X\right)\right)\right]\geq\mathscr{L}\left(\bar{w}_{1},\bar{w}_{2},\bar{w}_{3}\right)$,
i.e. $\left(\bar{w}_{1},\bar{w}_{2},\bar{w}_{3}\right)$ is a global
minimizer of $\mathscr{L}$.

In Case 2, since $y$ is a function of $x$, we obtain $\partial_{2}{\cal L}\left(y,\hat{y}\left(x;\bar{w}_{1},\bar{w}_{2},\bar{w}_{3}\right)\right)=0$
and hence ${\cal L}\left(y,\hat{y}\left(x;\bar{w}_{1},\bar{w}_{2},\bar{w}_{3}\right)\right)=0$
for ${\cal P}_{X}$-almost every $x$.

Finally we have from Assumptions \ref{assump:Regularity}, \ref{assump:three-layers}:
\begin{align*}
\left|\mathscr{L}\left(W\left(t\right)\right)-\mathscr{L}\left(\bar{w}_{1},\bar{w}_{2},\bar{w}_{3}\right)\right| & =\left|\mathbb{E}_{Z}\left[{\cal L}\left(Y,\hat{y}\left(X;W\left(t\right)\right)\right)-{\cal L}\left(Y,\hat{y}\left(X;\bar{w}_{1},\bar{w}_{2},\bar{w}_{3}\right)\right)\right]\right|\\
 & \leq K\mathbb{E}_{Z}\left[\left|\hat{y}\left(X;W\left(t\right)\right)-\hat{y}\left(X;\bar{w}_{1},\bar{w}_{2},\bar{w}_{3}\right)\right|\right]\\
 & \leq K\mathbb{E}\Big[\left|w_{3}\left(t,C_{2}\right)-\bar{w}_{3}\left(C_{2}\right)\right|+\left|\bar{w}_{3}\left(C_{2}\right)\right|\left|w_{2}\left(t,C_{1},C_{2}\right)-\bar{w}_{2}\left(C_{1},C_{2}\right)\right|\\
 & \qquad+\left|\bar{w}_{3}\left(C_{2}\right)\right|\left|\bar{w}_{2}\left(C_{1},C_{2}\right)\right|\left|w_{1}\left(t,C_{1}\right)-\bar{w}_{1}\left(C_{1}\right)\right|\Big]
\end{align*}
which tends to $0$ as $t\to\infty$. This completes the proof.
\end{proof}

\section{Converse for global convergence: Remark \ref{rem:Converse}\label{sec:Converse}}

We prove a converse statement for global convergence in relation with
the essential supremum condition (\ref{eq:Assump_esssup}).
\begin{prop}
Consider a neuronal embedding $\left(\Omega,{\cal F},P,\left\{ w_{i}^{0}\right\} _{i=1,2,3}\right)$
of $\left(\rho^{1},\rho^{2},\rho^{3}\right)$-i.i.d. initialization.
Consider the MF limit corresponding to the network (\ref{eq:three-layer-nn}),
such that they are coupled together by the coupling procedure in Section
\ref{subsec:Neuronal-Embedding}, under Assumptions \ref{assump:Regularity},
\ref{assump:Regularity-init}, $\xi_{1}\left(\cdot\right)=\xi_{2}\left(\cdot\right)=1$.
Assume that ${\cal L}(y,\hat{y})\to\infty$ as $|\hat{y}|\to\infty$
for each $y$. Further assume that there exists $\bar{w}_{3}$ such
that as $t\to\infty$,
\[
\mathbb{E}_{C_{2}}\left[\left|w_{3}(t,C_{2})-\bar{w}_{3}(C_{2})\right|\right]\to0.
\]
Then the following hold:
\begin{itemize}
\item Case 1 (convex loss): If ${\cal L}$ is convex in the second variable
and
\[
\lim_{t\to\infty}\mathscr{L}\left(W\left(t\right)\right)=\inf_{V}\mathscr{L}\left(V\right),
\]
then it must be that
\[
\sup_{c_{1}\in\Omega_{1}}\mathbb{E}_{C_{2}}\left[\left|\frac{\partial}{\partial t}w_{2}\left(t,c_{1},C_{2}\right)\right|\right]\to0\quad\text{as }t\to\infty.
\]
\item Case 2 (generic non-negative loss): Suppose that $\partial_{2}{\cal L}\left(y,\hat{y}\right)=0$
implies ${\cal L}\left(y,\hat{y}\right)=0$, and $y=y(x)$ is a function
of $x$. If $\mathscr{L}\left(W\left(t\right)\right)\to0$ as $t\to\infty$,
then the same conclusion also holds.
\end{itemize}
\end{prop}

\begin{proof}
We recall
\begin{align*}
\frac{\partial}{\partial t}w_{2}\left(t,c_{1},c_{2}\right) & =-\mathbb{E}_{Z}\Big[\partial_{2}{\cal L}\left(Y,\hat{y}\left(X;W\left(t\right)\right)\right)w_{3}\left(t,c_{2}\right)\\
 & \qquad\quad\times\varphi_{3}'\left(H_{3}\left(X;W\left(t\right)\right)\right)\varphi_{2}'\left(H_{2}\left(X,c_{2};W\left(t\right)\right)\right)\varphi_{1}\left(\left\langle w_{1}\left(t,c_{1}\right),X\right\rangle \right)\Big],
\end{align*}
for $c_{1}\in\Omega_{1}$, $c_{2}\in\Omega_{2}$. By Assumption \ref{assump:Regularity},
\[
\left|\frac{\partial}{\partial t}w_{2}\left(t,c_{1},c_{2}\right)\right|\leq K\mathbb{E}_{Z}\left[\left|\partial_{2}{\cal L}\left(Y,\hat{y}\left(X;W\left(t\right)\right)\right)\right|\right]\left|w_{3}\left(t,c_{2}\right)\right|.
\]
Note that the right-hand side is independent of $c_{1}$. Since $\mathbb{E}_{C_{2}}\left[\left|w_{3}(t,C_{2})-\bar{w}_{3}(C_{2})\right|\right]\to0$
as $t\to\infty$, we have for some finite $t_{0}\leq K$,
\[
\mathbb{E}_{C_{2}}\left[\left|\bar{w}_{3}(C_{2})\right|\right]\leq\mathbb{E}_{C_{2}}\left[\left|w_{3}(t_{0},C_{2})\right|\right]+K\leq K,
\]
where the last step is by Lemma \ref{lem:a-priori-MF-norms} and Assumption
\ref{assump:Regularity-init}. As such, for all $t$ sufficiently
large, we have:
\begin{align*}
\sup_{c_{1}\in\Omega_{1}}\mathbb{E}_{C_{2}}\left[\left|\frac{\partial}{\partial t}w_{2}\left(t,c_{1},C_{2}\right)\right|\right] & \leq K\mathbb{E}_{Z}\left[\left|\partial_{2}{\cal L}\left(Y,\hat{y}\left(X;W\left(t\right)\right)\right)\right|\right]\mathbb{E}_{C_{2}}\left[\left|w_{3}\left(t,C_{2}\right)\right|\right]\\
 & \leq K\mathbb{E}_{Z}\left[\left|\partial_{2}{\cal L}\left(Y,\hat{y}\left(X;W\left(t\right)\right)\right)\right|\right]\left(K+\mathbb{E}_{C_{2}}\left[\left|\bar{w}_{3}\left(C_{2}\right)\right|\right]\right)\\
 & \leq K\mathbb{E}_{Z}\left[\left|\partial_{2}{\cal L}\left(Y,\hat{y}\left(X;W\left(t\right)\right)\right)\right|\right].
\end{align*}
The proof concludes once we show that $\mathbb{E}_{Z}\left[\left|\partial_{2}{\cal L}\left(Y,\hat{y}\left(X;W\left(t\right)\right)\right)\right|\right]\to0$
as $t\to\infty$.

For a fixed $z=\left(x,y\right)$, let us write ${\cal L}\left(t,z\right)={\cal L}(y,\hat{y}(x;W(t)))$
and $\partial_{2}{\cal L}(t,z)=\partial_{2}{\cal L}(y,\hat{y}(x;W(t)))$
for brevity. Consider Case 1. We claim that if there is an increasing
sequence of time $t_{i}$ so that $\lim_{i\to\infty}\left[{\cal L}(t_{i},z)-\inf_{\hat{y}}{\cal L}(y,\hat{y})\right]=0$,
then $\lim_{i\to\infty}\left|\partial_{2}{\cal L}(t_{i},z)\right|=0$.
Indeed, it suffices to show that for any subsequence $t_{i_{j}}$
of $t_{i}$, there exists a further subsequence $t_{i_{j_{k}}}$ such
that $\lim_{k\to\infty}\left|\partial_{2}{\cal L}(t_{i_{j_{k}}},z)\right|=0$.
In any subsequence $t_{i_{j}}$ of $t_{i}$, using that ${\cal L}(t_{i_{j}},z)$
is convergent and the fact ${\cal L}(y,\hat{y})\to\infty$ as $|\hat{y}|\to\infty$,
we have $\hat{y}(x;W(t_{i_{j}}))$ is bounded. Hence, we obtain a
subsequence $t_{i_{j_{k}}}$ for which $\hat{y}(x;W(t_{i_{j_{k}}}))$
converges to some limit $\hat{y}^{*}$. By continuity, we have ${\cal L}(y,\hat{y}^{*})=\lim_{k\to\infty}{\cal L}(t_{i_{j_{k}}},z)=\inf_{\hat{y}}{\cal L}(y,\hat{y})$.
Thus, since ${\cal L}$ is convex in the second variable, we have
$\partial_{2}{\cal L}(y,\hat{y}^{*})=0$. Thus, $\lim_{k\to\infty}\left|\partial_{2}{\cal L}(t_{i_{j_{k}}},z)\right|=\left|\partial_{2}{\cal L}(y,\hat{y}^{*})\right|=0$,
as claimed. Similarly, we obtain in Case 2 that if there is an increasing
sequence of time $t_{i}$ so that $\lim_{i\to\infty}\left[{\cal L}(t_{i},z)\right]=0$,
then $\lim_{i\to\infty}\left|\partial_{2}{\cal L}(t_{i},z)\right|=0$.

To show that $\mathbb{E}_{Z}\left[\left|\partial_{2}{\cal L}\left(t,Z\right)\right|\right]\to0$
as $t\to\infty$, it suffices to show that for any increasing sequence
of times $t_{i}$ tending to infinity, there exists a subsequence
$t_{i_{j}}$ of $t_{i}$ such that $\mathbb{E}_{Z}\left[\left|\partial_{2}{\cal L}\left(t_{i_{j}},Z\right)\right|\right]\to0$.
In Case 1, we have $\lim_{i\to\infty}\mathscr{L}\left(W\left(t_{i}\right)\right)=\inf_{V}\mathscr{L}\left(V\right)$,
so $\lim_{i\to\infty}\mathbb{E}_{Z}\left[{\cal L}\left(t_{i},Z\right)-\inf_{\hat{Y}}{\cal L}(Y,\hat{Y})\right]=0$.
Since ${\cal L}\left(t_{i},Z\right)-\inf_{\hat{Y}}{\cal L}(Y,\hat{Y})$
is nonnegative, this implies that ${\cal L}\left(t_{i},Z\right)-\inf_{\hat{Y}}{\cal L}(Y,\hat{Y})$
converges to $0$ in probability. Thus, there is a further subsequence
$t_{i_{j}}$ for which ${\cal L}\left(t_{i_{j}},Z\right)-\inf_{\hat{Y}}{\cal L}(Y,\hat{Y})$
converges to $0$ ${\cal P}$-almost surely. By the previous claim,
$\left|\partial_{2}{\cal L}\left(t_{i_{j}},Z\right)\right|$ converges
to $0$ ${\cal P}$-almost surely. Since $\left|\partial_{2}{\cal L}\left(t_{i_{j}},Z\right)\right|$
is bounded ${\cal P}$-almost surely, we obtain that $\mathbb{E}_{Z}\left[\left|\partial_{2}{\cal L}\left(t_{i_{j}},Z\right)\right|\right]\to0$
from the bounded convergence theorem. The result in Case 2 can be
established similarly.
\end{proof}

\section{Useful tools}

We first present a useful concentration result. In fact, the tail
bound can be improved using the argument in \cite{feldman2018generalization},
but the following simpler version is sufficient for our purposes.
\begin{lem}
\label{lem:square hoeffding}Consider an integer $n\geq1$ and let
$x$, $c_{1}$, ..., $c_{n}$ be mutually independent random variables.
Let $\mathbb{E}_{x}$ and $\mathbb{E}_{c}$ denote the expectations
w.r.t. $x$ only and $\left\{ c_{i}\right\} _{i\in\left[n\right]}$
only, respectively. Consider a collection of mappings $\left\{ f_{i}\right\} _{i\in\left[n\right]}$,
which map to a separable Hilbert space $\mathbb{F}$. Let $f_{i}\left(x\right)=\mathbb{E}_{c}\left[f_{i}\left(c_{i},x\right)\right]$.
Assume that for some $R>0$, $\left|f_{i}\left(c_{i},x\right)-f_{i}\left(x\right)\right|\leq R$
almost surely, then for any $\delta>0$,
\[
\mathbb{P}\left(\mathbb{E}_{x}\left[\left|\frac{1}{n}\sum_{i=1}^{n}f_{i}\left(c_{i},x\right)-f_{i}\left(x\right)\right|\right]\geq\delta\right)\leq\frac{8R}{\sqrt{n}\delta}\exp\left(-\frac{n\delta^{2}}{8R^{2}}\right)\leq\frac{8R}{\delta}\exp\left(-\frac{n\delta^{2}}{8R^{2}}\right).
\]
\end{lem}

\begin{proof}
For brevity, let us define
\[
Z_{n}\left(x\right)=\sum_{i=1}^{n}\left(f_{i}\left(c_{i},x\right)-f_{i}\left(x\right)\right).
\]
By Theorem \ref{thm:iid-hilbert},
\[
\mathbb{P}\left(\left|Z_{n}\left(x\right)\right|\geq n\delta\middle|x\right)\leq2\exp\left(-n\delta^{2}/\left(4R^{2}\right)\right),
\]
and therefore,
\[
\mathbb{P}\left(\left|Z_{n}\left(x\right)\right|\geq n\delta\right)\leq2\exp\left(-n\delta^{2}/\left(4R^{2}\right)\right),
\]
since the right-hand side is uniform in $x$. Next note that, w.r.t.
the randomness of $x$ only,
\begin{align*}
\mathbb{E}_{x}\left[\left|Z_{n}\left(x\right)\right|\right] & =\mathbb{E}_{x}\left[\left|Z_{n}\left(x\right)\right|\mathbb{I}\left(\left|Z_{n}\left(x\right)\right|\geq n\delta/2\right)\right]+\mathbb{E}_{x}\left[\left|Z_{n}\left(x\right)\right|\mathbb{I}\left(\left|Z_{n}\left(x\right)\right|<n\delta/2\right)\right]\\
 & \leq\mathbb{E}_{x}\left[\left|Z_{n}\left(x\right)\right|\mathbb{I}\left(\left|Z_{n}\left(x\right)\right|\geq n\delta/2\right)\right]+n\delta/2.
\end{align*}
As such, by Markov's inequality and Cauchy-Schwarz's inequality,
\begin{align*}
\mathbb{P}\left(\mathbb{E}_{x}\left[\left|Z_{n}\left(x\right)\right|\right]\geq n\delta\right) & \leq\mathbb{P}\left(\mathbb{E}_{x}\left[\left|Z_{n}\left(x\right)\right|\mathbb{I}\left(\left|Z_{n}\left(x\right)\right|\geq n\delta/2\right)\right]\geq n\delta/2\right)\\
 & \leq\frac{2}{n\delta}\mathbb{E}\left[\left|Z_{n}\left(x\right)\right|\mathbb{I}\left(\left|Z_{n}\left(x\right)\right|\geq n\delta/2\right)\right]\\
 & \leq\frac{2}{n\delta}\mathbb{E}\left[\left|Z_{n}\left(x\right)\right|^{2}\right]^{1/2}\mathbb{P}\left(\left|Z_{n}\left(x\right)\right|\geq n\delta/2\right)^{1/2}\\
 & \leq\frac{4}{n\delta}\mathbb{E}\left[\left|Z_{n}\left(x\right)\right|^{2}\right]^{1/2}\exp\left(-\frac{n\delta^{2}}{8R^{2}}\right).
\end{align*}
Notice that since $c_{1},...,c_{n}$ are independent and $f_{i}\left(x\right)=\mathbb{E}_{c}\left[f_{i}\left(c_{i},x\right)\right]$,
\[
\mathbb{E}\left[\left|Z_{n}\left(x\right)\right|^{2}\right]=\sum_{i=1}^{n}\mathbb{E}\left[\left|f_{i}\left(c_{i},x\right)-f_{i}\left(x\right)\right|^{2}\right]\leq2nR^{2}.
\]
We thus get:
\begin{align*}
\mathbb{P}\left(\mathbb{E}_{x}\left[\left|Z_{n}\left(x\right)\right|\right]\geq n\delta\right) & \leq\frac{8R}{\sqrt{n}\delta}\exp\left(-\frac{n\delta^{2}}{8R^{2}}\right).
\end{align*}
This proves the claim.
\end{proof}
We state a martingale concentration result, which is a special case
of \cite[Theorem 3.5]{pinelis1994optimum} which applies to a more
general Banach space.
\begin{thm}[Concentration of martingales in Hilbert spaces.]
\label{thm:azuma-hilbert}Consider a martingale $Z_{n}\in\mathbb{Z}$
a separable Hilbert space such that $\left|Z_{n}-Z_{n-1}\right|\leq R$
and $Z_{0}=0$. Then for any $t>0$,
\[
\mathbb{P}\left(\max_{k\leq n}\left|Z_{k}\right|\geq t\right)\leq2\inf_{\lambda>0}\exp\left(-\lambda t+{\rm ess\text{-}sup}\sum_{k=1}^{n}\mathbb{E}\left[e^{\lambda|Z_{k}-Z_{k-1}|}-1-\lambda|Z_{k}-Z_{k-1}|\mid{\cal F}_{k-1}\right]\right).
\]
In particular, for any $\delta>0$,
\[
\mathbb{P}\left(\max_{k\leq n}\left|Z_{k}\right|\geq n\delta\right)\leq2\exp\left(-\frac{n\delta^{2}}{2R^{2}}\right).
\]
\end{thm}

The following concentration result for i.i.d. random variables in
Hilbert spaces is a corollary.
\begin{thm}[Concentration of i.i.d. sum in Hilbert spaces.]
\label{thm:iid-hilbert}Consider $n$ i.i.d. random variables $X_{1},...,X_{n}$
in a separable Hilbert space. Suppose that there exists a constant
$R>0$ such that $\left|X_{i}-\mathbb{E}\left[X_{i}\right]\right|\leq R$
almost surely. Then for any $\delta>0$,
\[
\mathbb{P}\left(\frac{1}{n}\left|\sum_{i=1}^{n}X_{i}-\mathbb{E}\left[X_{i}\right]\right|\geq\delta\right)\leq2\exp\left(-\frac{n\delta^{2}}{2R^{2}}\right).
\]
\end{thm}

\end{document}